\documentclass[twoside,11pt]{article}

\usepackage{blindtext}
\usepackage{setspace}
\usepackage{enumitem}
\usepackage{quoting}
\quotingsetup{vskip=5pt,leftmargin=15pt,rightmargin=15pt}
\usepackage{tocloft}
\usepackage[preprint]{jmlr2e}
\usepackage[top=1in, bottom=1in, left=1in, right=1in]{geometry}

\usepackage[utf8]{inputenc} 
\usepackage[T1]{fontenc}    
\usepackage{hyperref}       
\usepackage{url}            
\usepackage{booktabs}       
\usepackage{amsfonts}       
\usepackage{nicefrac}       
\usepackage{microtype}      

\usepackage{amsfonts}
\usepackage{amssymb}
\usepackage{lmodern}

\usepackage{empheq}

\usepackage{relsize}
\usepackage{bbm}

\usepackage{etoolbox}
\usepackage{setspace}
\AtBeginEnvironment{quote}{\par\singlespacing\small}

\usepackage{graphicx}
\usepackage{outlines}

\usepackage{tikz}
\usetikzlibrary{cd}
\usepackage[outline]{contour}
\usepackage{onimage}

\usepackage{xcolor}

\newcommand{\qedsymbol}{$\blacksquare$}

\newcommand*\widefbox[1]{\fbox{\hspace{2em}#1\hspace{2em}}}

\newcommand{\fndict}{\mathcal{D}}

\newcommand{\vf}{\mathfrak{X}}
\newcommand{\tf}{\mathfrak{T}}

\DeclareMathOperator{\Lie}{Lie}
\DeclareMathOperator{\graph}{gr}
\DeclareMathOperator{\image}{im}
\DeclareMathOperator{\sect}{\Sigma}
\DeclareMathOperator{\sym}{\mathfrak{sym}}
\DeclareMathOperator{\Sym}{Sym}

\DeclareMathOperator*{\minimize}{\min\!imize\enskip}

\DeclareMathOperator{\Tr}{Tr}
\DeclareMathOperator{\rank}{rank}
\DeclareMathOperator{\codim}{codim}
\DeclareMathOperator{\diag}{diag}
\DeclareMathOperator{\vspan}{span}
\DeclareMathOperator{\Range}{Range}
\DeclareMathOperator{\Null}{Null}

\DeclareMathOperator{\diverg}{div}
\DeclareMathOperator{\td}{d}
\DeclareMathOperator{\ddt}{\frac{\td}{\td t}}
\DeclareMathOperator{\D}{d}
\DeclareMathOperator{\Id}{Id}
\DeclareMathOperator{\intmul}{\mathlarger{\lrcorner}}
\DeclareMathOperator{\vpr}{vpr} 
\DeclareMathOperator{\Flow}{Fl}
\DeclareMathOperator{\volform}{dV}

\newcommand{\mat}[1]{{\pmb{#1}}}
\newcommand{\vect}[1]{{\boldsymbol{#1}}}

\newcommand{\mcal}[1]{{\mathcal{#1}}}
\newcommand{\mfrak}[1]{{\mathfrak{#1}}}

\newcommand{\R}{\mathbb{R}}

\DeclareMathOperator{\GL}{GL}

\DeclareMathOperator{\SE}{SE}
\DeclareMathOperator{\se}{\mathfrak{se}}
\DeclareMathOperator{\Ogrp}{O}
\DeclareMathOperator{\oalg}{\mathfrak{o}}
\DeclareMathOperator{\SO}{SO}

\makeatletter
\renewcommand*\env@matrix[1][*\c@MaxMatrixCols c]{%
  \hskip -\arraycolsep
  \let\@ifnextchar\new@ifnextchar
  \array{#1}}
\makeatother

\ShortHeadings{Unified Framework for Symmetry in Machine Learning}{Otto, Zolman, Kutz, and Brunton}

\begin{document}

\title{\LARGE{\textbf{A Unified Framework to Enforce, Discover, and Promote Symmetry in Machine Learning}}}

\author{
    \name \hspace{-4pt}Samuel E. Otto\thanks{\,Present affiliation: Sibley School of Mechanical and Aerospace Engineering, Cornell University, Ithaca, NY, USA} \email s.otto@cornell.edu \\
    \addr AI Institute in Dynamic Systems\\    
    University of Washington\\
    Seattle, WA 98195-4322, USA
    \AND
    \name Nicholas Zolman \email nzolman@uw.edu \\
    \addr AI Institute in Dynamic Systems\\    
    University of Washington\\
    Seattle, WA 98195-4322, USA
    \AND
    \name J. Nathan Kutz \email kutz@uw.edu \\
    \addr AI Institute in Dynamic Systems\\
    University of Washington \\
    Seattle, WA 98195-4322, USA
    \AND
    \name Steven L. Brunton \email sbrunton@uw.edu \\
    \addr AI Institute in Dynamic Systems\\
    University of Washington \\
    Seattle, WA 98195-4322, USA
}

\editor{My editor}

\maketitle

\begin{abstract}
    Symmetry is present throughout nature and continues to play an increasingly central role in machine learning.
    In this paper, we provide a unifying theoretical and methodological framework for incorporating Lie group symmetry into machine learning models in three ways:
    \textbf{1. enforcing} known symmetry when training a model; \textbf{2. discovering} unknown symmetries of a given model or data set; and \textbf{3. promoting} symmetry during training by learning a model that breaks symmetries within a user-specified group of candidates when there is sufficient evidence in the data.
    We show that these tasks can be cast within a common mathematical framework whose central object is the Lie derivative. 
    We extend and unify several existing results by showing that enforcing and discovering symmetry are linear-algebraic tasks that are dual with respect to the bilinear structure of the Lie derivative.
    We also propose a novel way to promote symmetry by introducing a class of convex regularization functions based on the Lie derivative and nuclear norm relaxation to penalize symmetry breaking during training of machine learning models. 
    We explain how these ideas can be applied to a wide range of machine learning models including basis function regression, dynamical systems discovery, neural networks, and neural operators acting on fields.\\

\end{abstract}


\newpage
\renewcommand{\baselinestretch}{.9}\normalsize
\tableofcontents
\renewcommand{\baselinestretch}{1.0}\normalsize

\newpage

\section{Introduction}

Symmetry is present throughout nature, 
and according to David \cite{Gross1996role} the discovery of fundamental symmetries has played an increasingly central role in physics since the beginning of the 20th century.
He asserts that
\vspace{-.1in}\begin{quote}
   ``Einstein’s great advance in 1905 was to put symmetry first, to regard the symmetry principle as the primary feature of nature that constrains the allowable dynamical laws.''
\end{quote}
\vspace{-.025in}
According to Einstein's special theory of relativity, physical laws including those of electromagnetism and quantum mechanics are Poincar\'{e}-invariant, meaning that after predictable transformations (actions of the Poincar\'{e} group), these laws can be applied in any non-accelerating reference frame, anywhere in the universe, at all times.
Specifically these transformations form a ten-parameter group including four translations of space-time, three rotations of space, and three shifts or \emph{boosts} in velocity.
For small boosts of velocity, these transformations become the Galilean symmetries of classical mechanics.
Similarly, the theorems of Euclidean geometry are unchanged after arbitrary translations, rotations, and reflections, comprising the Euclidean group.
In fluid mechanics, the conformal (angle-preserving) symmetry of Laplace's equation is used to reduce the study of idealized flows in complicated geometries to canonical flows in simple domains.
In dynamical systems, the celebrated theorem of \citet{noether1918invariante} establishes a correspondence between symmetries and conservation laws, an idea which has become a central pillar of mechanics~\citep{Abraham2008foundations}.
These examples illustrate the diversity of symmetry groups and their physical applications.
More importantly, they illustrate how \emph{symmetric models and theories in physics automatically extrapolate in explainable ways to environments beyond the available data.}

In machine learning, models that exploit symmetry can be trained with less data and use fewer parameters compared to their asymmetric counterparts.
Early examples include augmenting data with known transformations (see \cite{shorten2019survey, van2001art}) or using convolutional neural networks (CNNs) to achieve translation invariance for image processing tasks (see \cite{fukushima1980neocognitron, lecun1989backpropagation}).
More recently, equivariant neural networks respecting Euclidean symmetries have achieved state-of-the-art performance for predicting potentials in molecular dynamics \cite{batzner2022e3atom}. 
As with physical laws, symmetries and invariances allow machine learning models to extrapolate beyond the training data, and achieve high performance with fewer modeling parameters.

However, many problems are only weakly symmetric.
Gravity, friction, and other external forces can cause some or all of the Poincar\'{e} or Galilean symmetries to be broken.
Interactions between particles can be viewed as breaking symmetries possessed by non-interacting particles.
Written characters have translation and scaling symmetry, but not rotation (cf. $6$ and $9$, d and p, N and Z) or reflection (cf. b and d, b and p).
One of the main contributions of this work is to propose a method of enforcing a new principle of parsimony in machine learning.
This principal of parsimony by maximal symmetry states that \emph{a model should break a symmetry within a set of physically reasonable transformations 
(such as Poincar\'{e}, Galilean, Euclidean, or conformal symmetry) 
only when there is sufficient evidence in the data.}

In this paper, we provide a unifying theoretical and methodological framework for incorporating symmetry into machine learning models in the following three ways:
\vspace{-.05in}
\begin{enumerate}[label=Task \arabic*., leftmargin=1.5cm,itemsep=-1.75pt]
    \item \textbf{Enforce.} Train a model with known symmetry. 
    \item \textbf{Discover.} Identify the symmetries of a given model or data set.
    \item \textbf{Promote.} Train a model with as many symmetries as possible (from among candidates), breaking symmetries only when there is sufficient evidence in the data.
\end{enumerate}

While these tasks have been studied to varying extents separately,
we show how they can be situated within a common mathematical framework whose central object is
the Lie derivative associated with fiber-linear Lie group actions on vector bundles.
As a special case, the Lie derivative recovers the linear constraints derived by \citet{Finzi2021practical} for weights in equivariant multilayer perceptrons.
In full generality, \emph{we show that known symmetries can be enforced as linear constraints derived from Lie derivatives for a large class of problems} including learning vector and tensor fields on manifolds as well as learning equivariant integral operators acting on such fields.
For example the kernels of \emph{steerable} CNNs developed by \cite{weiler20183d, Weiler2019general} are constructed to automatically satisfy these constraints for the groups $\SO(3)$ (rotations in three dimensions) and $\SE(2)$ (rotations and translations in two dimensions).
We show how analogous steerable networks for other groups, such as subgroups of $\SE(n)$, can be constructed by numerically enforcing linear constraints derived from the Lie derivative on integral kernels defining each layer. 
Symmetries, conservation laws, and symplectic structure can also be enforced when learning dynamical systems via linear constraints on the vector field.
Again these constraints come from the Lie derivative and can be 
incorporated into
neural network architectures and basis function regression models such as Sparse Identification of Nonlinear Dynamics (SINDy) \citep{Brunton2016discovering}.

\citet{Moskalev2022liegg} identifies the connected subgroup of symmetries of a trained neural network by computing the nullspace of a linear operator.
Likewise, \citet{Kaiser2018discovering,kaiser2021data} recovers the symmetries and conservation laws of a dynamical system by computing the nullspace of a different linear operator.
\emph{We observe that these operators and others whose nullspaces encode the symmetries of more general models can be derived directly from the Lie derivative in a manner dual to the construction of operators used to enforce symmetry.}
Specifically, the nullspaces of the operators we construct reveal the largest connected subgroups of symmetries for large classes of models.
This extends work by \cite{Gruver2022Lie} using the Lie derivative to test whether a trained neural network is equivariant with respect to a given one-parameter group, e.g., rotation of images.
Generalizing the ideas in \citet{Cahill2023Lie} to nonlinear actions on manifolds, we also show that the symmetries of point clouds approximating submanifolds could potentially be recovered by computing the nullspaces of associated linear operators.
This could allow for unsupervised mining of data for hidden symmetries.

The idea of relaxed symmetry has been introduced recently by \cite{Wang2022approximately}, along with architecture-specific symmetry-promoting regularization functions involving sums or integrals over the candidate group of transformations.
The Augerino method introduced by \cite{Benton2020learning} uses regularization to promote equivariance with respect to a larger collection of candidate transformations.
Promoting physical constraints through the loss function is also a core concept of Physics-Informed Neural Networks (PINNs) introduced by \cite{Raissi2019physics}.
\emph{Our approach to the third task (promoting symmetry) is to introduce a unified and broadly applicable class of convex regularization functions based on the Lie derivative to penalize symmetry breaking during training of machine learning models.}
As we describe above, the Lie derivative yields an operator whose nullspace corresponds to the symmetries of a given model.
Hence, the lower the rank of this operator, the more symmetric the model is.
The nuclear norm has been used extensively as a convex relaxation of the rank with favorable theoretical properties for compressed sensing and low-rank matrix recovery \citep{Recht2010guaranteed, Gross2011recovering}, as well as in robust PCA \citep{Candes2011robust, Bouwmans2018applications}.
Penalizing the nuclear norm of our symmetry-encoding operator yields a convex regularization function that can be added to the loss when training machine learning models, including basis function regression and neural networks.
Likewise, we show how nuclear norm penalties can be used to promote conservation laws and Hamiltonicity with respect to candidate symplectic structures when fitting dynamical systems to data.
This approach lets us train a model and enforce data-consistent symmetries simultaneously. 

\newpage
\section{Executive summary}
\label{sec:precap}

This paper provides a linear-algebraic framework to enforce, discover, and promote symmetry of machine learning models.
To illustrate, consider a model defined by a function $F: \R^m \to \R^n$.
By a symmetry, we mean an invertible transformation $\theta$ and an invertible linear transformation $T$ so that
\begin{equation}
    \label{eqn:exec_symmetry}
    F(\theta(x)) = T F(x) \qquad \mbox{for all $x\in\R^m$}.
\end{equation}
Examples to keep in mind are rotations and translations.
If $(\theta_a, T_a)$ is a symmetry, then so is its inverse $(\theta_a^{-1}, T_a^{-1})$, and if $(\theta_b, T_b)$ is another symmetry, then so is the composition $(\theta_b \circ \theta_a, T_b T_a)$.
We work with sets of transformations $\{ (\theta_g, T_g) \}_{g\in G}$, called groups, that have an identity element and are closed under composition and inverses.
Specifically, we consider Lie group actions on manifolds.

\subsection{Enforcing symmetry}
Given a group of transformations, the symmetry condition \eqref{eqn:exec_symmetry} imposes linear constraints on $F$ that can be enforced during the fitting process.
However, there is one constraint per transformation, making direct enforcement impractical for continuous Lie groups such as rotations or translations.
For smoothly-parametrized connected groups it suffices to enforce a finite collection of linear constraints $\mcal{L}_{\xi_i} F = 0$ where $\xi_i$ are elements in the \emph{Lie algebra} and $\mcal{L}_{\xi}$ is the \emph{Lie derivative} defined by
\begin{equation}
    \label{eqn:exec_Lie_derivative}
    \mcal{L}_\xi F(x) = \left(  \frac{\partial F(x)}{\partial x} \left.\frac{\partial \theta_g(x) }{\partial g}\right\vert_{g=\Id}  - \left.\frac{\partial T_g F(x) }{\partial g}\right\vert_{g=\Id}\right) \xi.
\end{equation}
This operation is linear with respect to $\xi$ and with respect to $F$.

\subsection{Discovering symmetry}
The symmetries of a given model $F$ form a subgroup that we seek to identify within a given group of candidates.
For continuous Lie groups of transformations, the component of the subgroup containing the identity is revealed by the nullspace of the linear map
\begin{equation}
    \label{eqn:exec_symmetry_detector}
    L_F: \xi \mapsto \mcal{L}_{\xi} F.
\end{equation}
More generally, the symmetries of a smooth surface in $\R^n$ can be determined from data sampled from this surface by computing the nullspace of a positive semidefinite operator.
When the surface is the graph of the function $F$ in $\R^m \times \R^n$, this operator is $L_F^* L_F$ with $L_F^*$ being an adjoint operator.

\subsection{Promoting symmetry}
Here, we seek to learn a model $F$ that both fits data and possesses as many symmetries as possible from a given candidate group of transformations.
Since the nullspace of the operator $L_F$ defined in \eqref{eqn:exec_symmetry_detector} corresponds with the symmetries of $F$, we seek to minimize the rank of $L_F$ during training.
To do this, we regularize optimization problems for $F$ using a convex relaxation of the rank given by the nuclear norm (sum of singular values)
    \begin{equation}
        \| L_F \|_* = \sum_{i=1}^{\dim G} \sigma_i(L_F).
    \end{equation}
This is convex with respect to $F$ because $F \mapsto L_F$ is linear and the nuclear norm is convex.

\newpage

\section{Related work}

\subsection{Enforcing symmetry}
\label{subsec:related_work_enforcing_symmetry}
Data-augmentation, as reviewed by \cite{shorten2019survey, van2001art}, is one of the simplest ways to incorporate known symmetry into machine learning models.
Usually this entails training a neural network architecture on training data to which known transformations have been applied.
The theoretical foundations of these methods are explored by \cite{Chen2020GroupAug}.
Data-augmentation has also been used by \cite{Benton2020learning} to construct equivariant neural networks by averaging the network's output over transformations applied to the data. Finally,  \cite{brandstetter2022lie} applied data-augmentation strategies with known Lie-point symmetries for improving neural PDE solvers.

Symmetry can also be enforced directly on the machine learning architecture.
For example, Convolutional Neural Networks (CNNs), introduced by \cite{fukushima1980neocognitron} and popularized by \cite{lecun1989backpropagation}, achieve translational equivariance by employing convolutional filters with trainable kernels in each layer.
CNNs have been generalized to provide equivariance with respect to symmetry groups other than translation. 
Group-Equivariant CNNs (G-CNNs) \citep{Cohen2016GroupEqConv} provide equivariance with respect to arbitrary discrete groups generated by translations, reflections, and rotations. 
Rotational equivariance can be enforced on three-dimensional scalar, vector, or tensor fields using the 3D Steerable CNNs developed by \cite{weiler20183d}. 
Spherical CNNs \cite{cohen2018spherical, esteves2018learning} allow for rotation-equivariant maps to be learned for fields (such as projected images of 3D objects) on spheres. 
Essentially any group equivariant linear map (defining a layer of an equivariant neural network) acting fields can be described by group convolution \citep{Kondor2018generalization, Cohen2019general}, with the spaces of appropriate convolution kernel characterized by \cite{Cohen2019general}. 
\cite{Finzi2020generalizing} provides a practical way to construct convolutional layers that are equivariant with respect to arbitrary Lie groups and for general data types.
For dynamical systems, \cite{Marsden:MS,Rowley2003reduction, Abraham2008foundations} describe techniques for symmetry reduction of the original problem to a quotient space where the known symmetry group has been factored out.
Related approaches have been used by \cite{peitz2023partial, steyert2022uncovering} to approximate Koopman operators for symmetric dynamical systems (see \cite{Koopman1931Hamiltonian, Mezic2005spectral, Mauroy2020koopman, Otto2021koopman, Brunton2022siamreview}).

A general method for constructing equivariant neural networks is introduced by \cite{Finzi2021practical}, and relies on the observation that equivariance can be enforced through a set of linear constraints. 
For graph neural networks, \cite{maron2018invariant} characterizes the subspaces of linear layers satisfying permutation equivariance. 
Similarly, \cite{Ahmadi2020learning_short, Ahmadi2023learning_long} show that discrete symmetries and other types of side information including interpolation at finite sets of points, nonnegativity, monotonicity, invariance of sets, and gradient or Hamiltonian structure can be enforced via linear or convex constraints in learning problems for dynamical systems.
Our work builds on the results of \cite{Finzi2021practical}, \cite{weiler20183d}, \cite{Cohen2019general}, and \cite{Ahmadi2020learning_short, Ahmadi2023learning_long} by showing that equivariance can be enforced in a systematic and unified way via linear constraints for large classes of functions and neural networks.
Concurrent work by \citep{yang2024symmetry} shows how to enforce known or discovered Lie group symmetries on latent dynamics using hard linear constraints or soft penalties.

\subsection{Discovering symmetry}
Early work by \cite{Rao1999learning, Miao2007learning} used nonlinear optimization to learn infinitesimal generators describing transformations between images. 
Later, it was recognized by \cite{Cahill2023Lie} that linear algebraic methods could be used to uncover the generators of continuous linear symmetries of arbitrary point clouds in Euclidean space. 
Similarly, \cite{Kaiser2018discovering} and \cite{Moskalev2022liegg} show how conserved quantities of dynamical systems and invariances of trained neural networks can be revealed by computing the nullspaces of associated linear operators.
We connect these linear-algebraic methods to the Lie derivative, and provide generalizations to nonlinear group actions on manifolds. 
The Lie derivative has been used by \cite{Gruver2022Lie} to quantify the extent to which a trained network is equivariant with respect to a given one-parameter subgroup of transformations.
Our results show how the Lie derivative can reveal the entire connected subgroup of symmetries of a trained model via symmetric eigendecomposition.

More sophisticated nonlinear optimization techniques use Generative Adversarial Networks (GANs) to learn the transformations that leave a data distribution unchanged. 
These methods include SymmetryGAN developed by \cite{Desai2022symmetry} and LieGAN developed by \cite{Yang2023generative}. 
In contrast, our methods for detecting symmetry are entirely linear-algebraic.

While symmetries may exist in data, their representation may be difficult to describe. \cite{yang2023latent} develop Latent LieGAN (LaLieGAN) to extend LieGAN to find linear representations of symmetries in a latent space.
Recently this has been applied to dynamics discovery \citep{yang2024symmetry}. 
Likewise, \cite{Liu2022hidden} discover hidden symmetries by optimizing nonlinear transformations into spaces where candidate symmetries hold. 
Similar to our approach for promoting symmetry, they use a cost function to measure whether a given symmetry holds. 
In contrast, our regularization functions enable subgroups of candidate symmetry groups to be identified.

\subsection{Promoting symmetry}
Biasing a network towards increased symmetry can be accomplished through methods such as symmetry regularization. Analogous to the physics-informed loss developed in PINNs by \cite{Raissi2019physics} that penalize a solution for violating known dynamics, one can penalize symmetry violation for a known group; for example,  \cite{akhound2024lie} extends the PINN framework to penalize deviations of known Lie-point symmetries of a PDE. More generally, however, one can consider a \textit{candidate} group of symmetries and promote as much symmetry as possible that is \textit{consistent with the available data}. \cite{Wang2022approximately} discusses these approaches, along with architecture-specific methods, including regularization functions involving summations or integrals over the candidate group of symmetries.
While our regularization functions resemble these for discrete groups, we use a different approach for continuous Lie groups.
By leveraging the Lie algebra, our regularization functions eliminate the need to numerically integrate complicated functions over the group---a task that is already prohibitive for the $10$-dimensional non-compact group of Galilean symmetries in classical mechanics.

Automated data augmentation techniques introduced by \cite{Cubuk2019autoaugment, Hataya2020faster, Benton2020learning} are another class of methods that arguably promote symmetry.
These techniques optimize the distribution of transformations applied to augment the data during training. 
For example \emph{Augerino} is an elegant method developed by \cite{Benton2020learning} which averages an arbitrary network's output over the augmentation distribution and relies on regularization to prevent the distribution of transformations from becoming concentrated near the identity. 
In essence, the regularization biases the averaged network towards increased symmetry.

In contrast, our regularization functions promote symmetry on an architectural level for the original network. 
This eliminates the need to perform averaging, which grows more costly for larger collections of symmetries. 
While a distribution over symmetries can be useful for learning interesting partial symmetries (e.g. $6$ stays $6$ for small rotations, before turning into $9$), as is done by \cite{Benton2020learning} and \cite{romero2022learning}, it is not clear how to use sampling from a continuous distribution of transformations to identify lower-dimensional subgroups, which have measure zero. 
On the other hand, a linear-algebraic approach provides a natural way to identify and promote symmetries in lower-dimensional connected subgroups.

\subsection{Additional approaches and applications}
There are several other approaches that incorporate various aspects of enforcing, discovering, and promoting symmetries. For example, \cite{baddoo2023physics} developed algorithms to enforce and promote known symmetries in dynamic mode decomposition, through manifold constrained learning and regularization, respectively.  \cite{baddoo2023physics} also showed that discovering unknown symmetries is a dual problem to enforcing symmetry.   
Exploiting symmetry has also been a central theme in the reduced-order modeling of fluids for decades~\citep{HLBR_turb}. As machine learning methods are becoming widely used to develop these models~\citep{Brunton2020arfm}, the themes of enforcing and discovering symmetries in machine models are increasingly relevant.  Known fluid symmetries have been enforced in SINDy for fluid systems~\citep{Loiseau2017jfm} through linear equality constraints; this approach was generalized to enforce more complex constraints~\citep{champion2020unified}. Unknown symmetries were similarly uncovered for electroconvective flows~\citep{guan2020sparse}. Symmetry breaking is also important in many turbulent flows~\citep{Callaham2022scienceadvances}. 


\section{Elementary theory of Lie group actions}
\label{sec:background_on_matrix_Lie_groups}
This section provides background and notation required to understand the main results of this paper in the less abstract, but still highly useful setting of Lie groups acting on vector spaces.
In Section~\ref{sec:fundamental_operators} we use this theory to study the symmetries of continuously differentiable functions between vector spaces.
Such functions form the basic building blocks of many machine learning models such as basis functions regression models, the layers of multilayer perceptrons, and the kernels of integral operators acting on spatial fields such as images.
We emphasize that this is not the most general setting for our results, but we provide this section and simpler versions of our main Theorems in Section~\ref{sec:fundamental_operators} in order to make the presentation more accessible.
We develop our main results in the more general setting of fiber-linear Lie group actions on sections of vector bundles in Section~\ref{sec:sections_of_vector_bundles}.

\subsection{Lie groups and subgroups}

Lie groups are ubiquitous in science and engineering.
Some familiar examples include the general linear group $\GL(n)$ consisting of all real, invertible, $n\times n$ matrices; the orthogonal group
\begin{equation}
    \label{eqn:orthogonal_group}
    \Ogrp(n) = \left\{ Q \in \R^{n\times n} \ : \ Q^T Q = I \right\};
\end{equation}
and the special Euclidean group
\begin{equation}\label{eqn:SEn}
    \SE(n) = \left\{ \begin{bmatrix} Q & b \\ 0 & 1\end{bmatrix} \ : \ Q\in\R^{n\times n}, \ b \in \R^n, \ Q^T Q = I, \ \det(Q) = 1 \right\},
\end{equation}
representing rotations and translations in real $n$-dimensional space, $\R^n$, embedded in $\R^{n+1}$ via $x \mapsto (x, 1)$. 
Observe that the sets $\GL(n)$, $\Ogrp(n)$, and $\SE(n)$ contain the identity matrix and are closed under matrix multiplication and inversion, making them into (non-commutative) groups.
They are also smooth manifolds, which makes them \emph{Lie groups}.
In general, a Lie group is a smooth manifold that is simultaneously an algebraic group whose composition and inversion operations are smooth maps.
The identity element is usually denoted $e$ for ``\emph{einselement}'', which for a matrix Lie group is the identity matrix $e = I$.
This section summarizes some basic results that can be found in references such as \citep{MarsdenMTAA, Lee2013introduction, Varadarajan1984Lie, Hall2015Lie}.


A key property of a Lie group is that it is almost entirely characterized by an associated vector space called its \emph{Lie algebra}.
This allows global nonlinear questions about the group --- such as which elements leave a function unchanged --- to be answered using linear algebra.
If $G$ is a Lie group, its Lie algebra, commonly denoted $\Lie(G)$ or $\mfrak{g}$, is the vector space consisting of all smooth vector fields on $G$ that remain invariant when pushed forward by left translation $L_g : h \mapsto g \cdot h$.
Translating back and forth from the identity element, the Lie algebra can be identified with the tangent space $\Lie(G) \cong T_e G$.
For example, the Lie algebra of the orthogonal group $\Ogrp(n)$ consists of all skew-symmetric matrices, and is denoted 
\begin{equation}
    \oalg(n)
    = \left\{ S \in \R^{n\times n} \ : \ S + S^T = 0 \right\}.
\end{equation}
A key fact is that the Lie algebra of $G$ is closed under the \emph{Lie bracket} of vector fields\footnote{In $\R^n$ a vector field $V = (V^1, \ldots, V^n)$ is equivalent to a directional derivative operator $V^1 \frac{\partial}{\partial x^1} + \cdots + V^n \frac{\partial}{\partial x^n}$. A vector field on a smooth manifold is defined as an analogous linear operator acting on the space of smooth functions. The Lie bracket is defined as the commutator of these operators.}
\begin{equation}
    [\xi, \eta] = \xi \eta - \eta \xi \in \Lie(G).
\end{equation}
For matrix Lie groups, this corresponds to the same commutator of matrices $\xi, \eta \in T_I G$. 

The key tool relating global properties of a Lie group back to its Lie algebra is the exponential map $\exp: \Lie(G) \to G$.
A vector field $\xi \in \Lie(G)$ has a unique integral curve $\gamma : (-\infty, \infty) \to G$ passing through the identity $\gamma(0) = e$ and satisfying $\gamma'(t) = \xi\vert_{\gamma(t)}$.
The exponential map defined by
\begin{equation}
    \exp(\xi) := \gamma(1)
\end{equation}
reproduces the entire integral curve via $\exp(t\xi) = \gamma(t)$. 
An exponential curve is illustrated in Figure~\ref{fig:Lie_group_action}.
For a matrix Lie group and $\xi \in T_I G$, the exponential map is the matrix exponential
\begin{equation}
    \exp(\xi) = e^{\xi} = \sum_{k=0}^{\infty} \frac{1}{k!} \xi^k.
\end{equation}
 Many of the basic properties of the exponential map, such as $\exp((s+t)\xi) = \exp(s\xi)\cdot \exp(t\xi)$, $\exp(\xi)^{-1} = \exp(-\xi)$, and $\D\exp(0) = \Id_{T_e G}$ are summarized by \cite{Lee2013introduction} in Proposition~20.8. 
We highlight that $\exp$ provides a diffeomorphism of an open neighborhood of the origin $0$ in $\Lie(G)$ and an open neighborhood of the identity element $e$ in $G$.

The connected component of $G$ containing the identity element is called the \emph{identity component} of the Lie group and is denoted $G_0$.
Any element in this component can then be expressed as a finite product of exponentials, that is
\begin{equation}
    G_0 = \big\{ \exp{(\xi_1)} \cdots \exp{(\xi_N)} \ : \ \xi_1, \ldots, \xi_N \in \Lie(G), \ N = 1, 2, 3, \ldots \big\}.
\end{equation}
Moreover, the identity component is a normal subgroup of $G$ and all of the other connected components of $G$ are diffeomorphic cosets of $G_0$, as we illustrate in Figure~\ref{fig:Lie_group_action}.
For example, the special Euclidean group $\SE(n)$ is connected, and thus equal to its identity component.
On the other hand, the orthogonal group $\Ogrp(n)$ is compact and has two components consisting of orthogonal matrices $Q$ whose determinants are $1$ and $-1$.
The identity component of the orthogonal group is called the special orthogonal group and is denoted $\SO(n)$.
It is a general fact that when a Lie group such as $\SO(n)$ is connected and compact, it is equal to the image of the exponential map without the need to consider products of exponentials (see \cite{Tao2011expsurjective} and Appendix~C.1 of \cite{Lezcano2019cheap}).

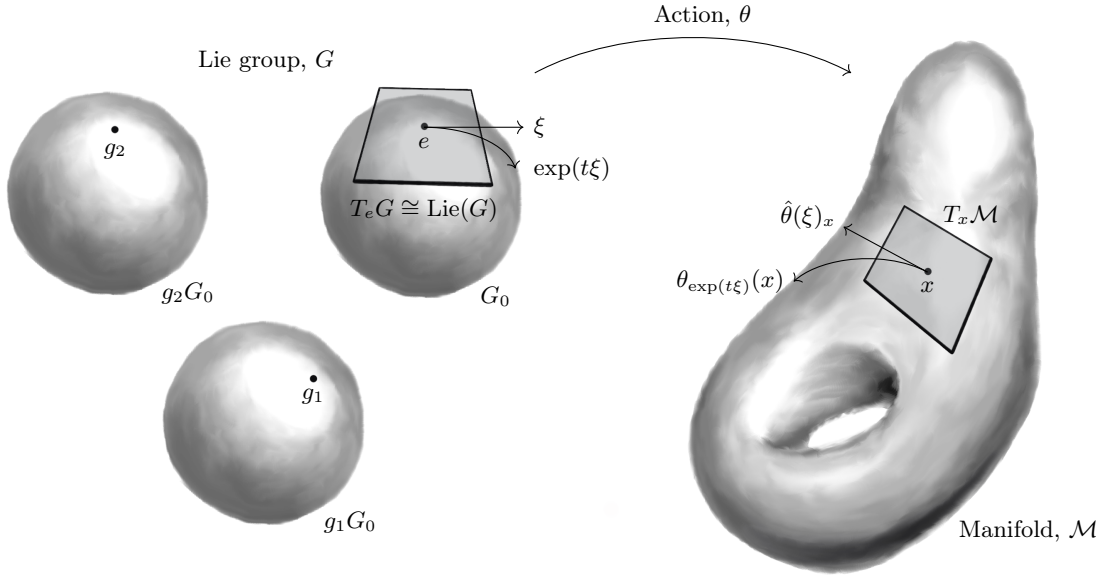
\begin{figure}[t]
    \centering
    \begin{tikzonimage}[trim=20 150 0 150, clip=true, width=0.9\textwidth]{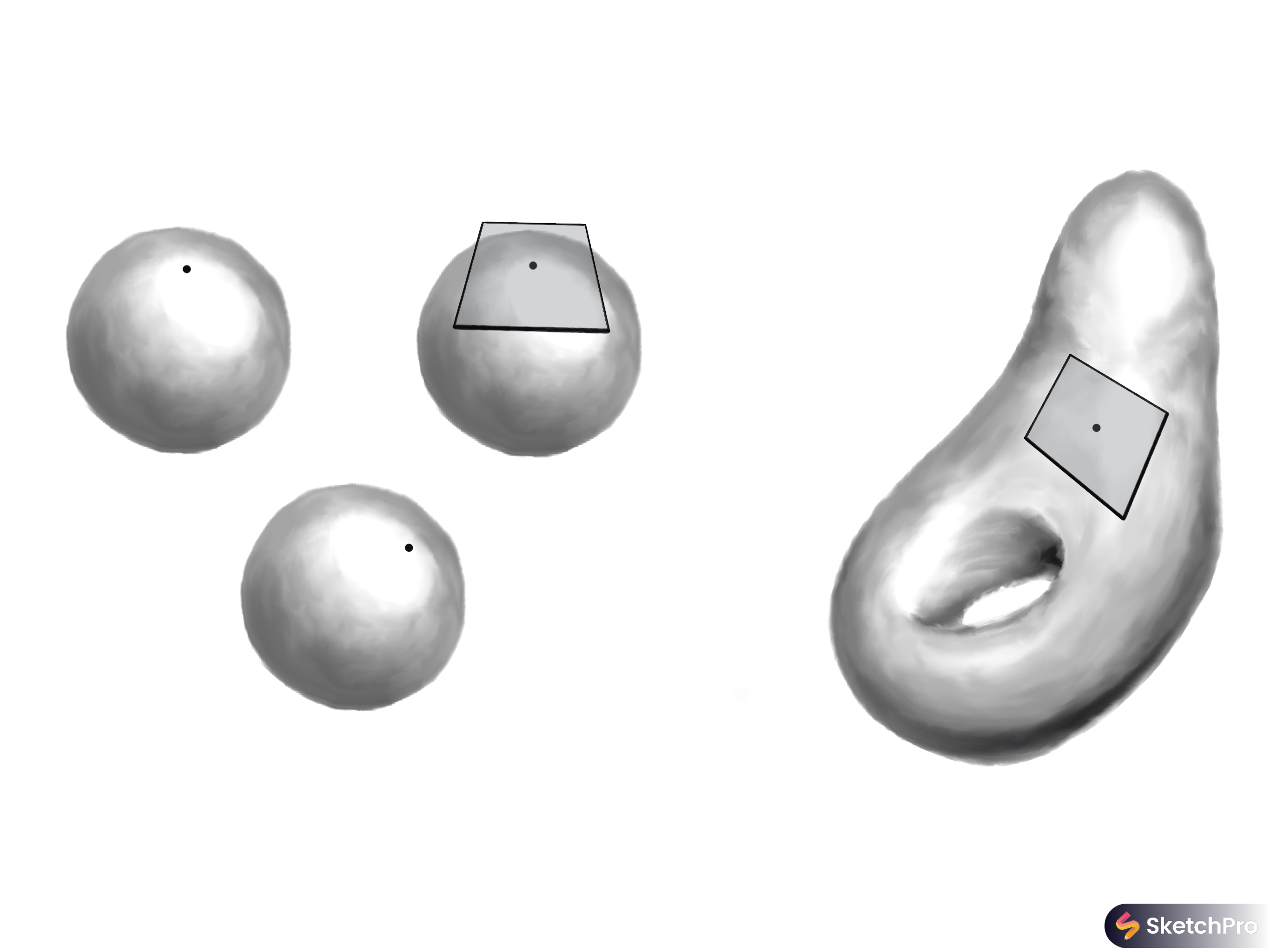}
        \node[rotate=0] at (0.27, 0.92) {\footnotesize Lie group, $G$};
        \node[rotate=0] at (0.475, 0.53) {\footnotesize $G_0$};
        \node[rotate=0, anchor=north] at (0.411, 0.71) {\footnotesize $T_e G \cong \Lie(G)$};
        \node[rotate=0, anchor=north] at (0.411, 0.81) {\footnotesize $e$};
        \draw[->] (0.411, 0.81) arc[radius=0.083, start angle=90, end angle=10];
        \node[rotate=0, anchor=west] at (0.50, 0.74) {\footnotesize $\exp(t\xi)$};
        \draw[->] (0.411, 0.81) -- (.50, 0.81);
        \node[rotate=0, anchor=west] at (0.50, 0.81) {\footnotesize $\xi$};
        \node[rotate=0, anchor=north] at (0.312, 0.393) {\footnotesize $g_1$};
        \node[rotate=0] at (0.345, 0.155) {\footnotesize $g_1 G_0$};
        \node[rotate=0, anchor=north] at (0.135, 0.80) {\footnotesize $g_2$};
        \node[rotate=0] at (0.20, 0.535) {\footnotesize $g_2 G_0$};
        \node[rotate=0, anchor=west] at (0.88, 0.14) {\footnotesize Manifold, $\mcal{M}$};
        \node[rotate=0] at (0.9, 0.665) {\footnotesize $T_x \mcal{M}$};
        \node[rotate=0, anchor=north] at (0.861, 0.570) {\footnotesize $x$};
        \draw[->] (0.861, 0.570) -- (0.861-0.075, 0.570+0.075);
        \node[rotate=0, anchor=east] at (0.788, 0.665) {\footnotesize $\hat{\theta}(\xi)_{x}$};
        \draw[->] (0.861, 0.570) arc[radius=0.076, start angle=45, end angle=150];
        \node[rotate=0, anchor=east] at (0.742, 0.554) {\footnotesize $\theta_{\exp(t\xi)}(x)$};
        \draw[->] (0.51, 0.90) arc[radius=0.2, start angle=135, end angle=45];
        \node[rotate=0, anchor=south] at (0.66, 0.96) {\footnotesize Action, $\theta$};
    \end{tikzonimage}
    \vspace{-.15in}
    \caption{A Lie group $G$ and its action $\theta$ on a manifold $\mcal{M}$.
    The Lie group $G$ consists of three connected components with $G_0$ being the one that contains the identity element $e$.
    Each non-identity component of $G$ is a coset $g_i G_0$ formed by translating the identity component by an arbitrary element $g_i$ in the component.
    The Lie algebra $\Lie(G)$ is identified with the tangent space $T_e G$ and an exponential curve $\exp(t\xi)$ generated by an element $\xi \in \Lie(G)$ is shown.
    The infinitesimal generator $\hat{\theta}(\xi)$ 
    is the vector field on $\mcal{M}$ whose flow corresponds with the action $\theta_{\exp(t\xi)}$ of group elements along $\exp(t\xi)$.}
    \label{fig:Lie_group_action}
\end{figure}

A subgroup $H$ of a Lie group $G$ is called a \emph{Lie subgroup} when $H$ is an immersed submanifold of $G$ and the group operations are smooth when restricted to $H$.
An immersed submanifold does not necessarily inherit its topology as a subset of $G$, but rather $H$ has a topology and smooth structure that makes the inclusion $\imath_H : H \hookrightarrow G$ smooth and its derivative injective.
The tangent space to a Lie subgroup $H \subset G$ at the identity, defined as $T_e H = \Range(\D \imath_H(I)) \subset \Lie(G)$, is closed under the Lie bracket and thus forms a \emph{Lie subalgebra} of $\Lie(G)$, denoted $\Lie(H)$.
Conversely, any subalgebra $\mfrak{h} \subset \Lie(G)$, that is, any subspace closed under the Lie bracket corresponds to a unique connected Lie subgroup $H \subset G_0$ satisfying $\Lie(H) = \mfrak{h}$ (see Theorem~19.26 in \cite{Lee2013introduction}).
Later on, we will use this fact to identify the connected subgroups of symmetries of machine learning models based on infinitesimal criteria. 
Another useful result is the \emph{closed subgroup theorem} (see Theorem~20.12 in \cite{Lee2013introduction}).
It says that if $H \subset G$ is a closed subset and is closed under the group operations of $G$, then $H$ is automatically an embedded Lie subgroup of $G$.
Interestingly, while a Lie subgroup $H \subset GL(n)$ need not be a closed subset, it turns out that $H$ can always be embedded as a closed subgroup in a larger $\GL(n')$, $n' \geq n$ thanks to Theorem~9 in \cite{Goto1950faithful}.

\subsection{Group representations, actions, and infinitesimal generators}
\label{subsec:representations_actions_generators}
A Lie group homomorphism is a smooth map $\Phi:G_1 \to G_2$ between Lie groups that respects the group product, that is,
\begin{equation}
    \Phi(g_1 g_2) = \Phi(g_1) \cdot \Phi(g_2).
\end{equation}
The tangent map $\phi := \D \Phi(e): T_e G_1 \cong \mfrak{g}_1 \to T_e G_2 \cong \mfrak{g}_2$ is a Lie algebra homomorphism, meaning that it is a linear map respecting the Lie bracket:
\begin{equation}
    \phi\big( [\xi_1, \xi_2] \big) = \big[\phi(\xi_1), \phi(\xi_2)\big].
\end{equation}
Moreover, the Lie group homomorphism and its induced Lie algebra homomorphism are related by 
the exponential maps
on $G_1$ and $G_2$ via the identity
\begin{equation}
    \Phi\big( \exp(\xi) \big) = \exp\big( \phi(\xi) \big).
    \label{eqn:intertwining_exp_and_rep}
\end{equation}
Another fundamental result  
is that any Lie algebra homomorphism $\Lie(G_1) \to \Lie(G_2)$ corresponds to a unique Lie group homomorphism $G_1 \to G_2$ when $G_1$ is simply connected.
When $G_2$ is the general linear group on a vector space, then the Lie group and Lie algebra homomorphisms are called Lie group and Lie algebra \emph{representations}.

A Lie group $G$ can act on a vector space $\mcal{V}$ via a representation $\Phi: G \to \GL(\mcal{V})$ according to
\begin{equation}
    \theta: (x, g) \mapsto \Phi(g^{-1}) x,
    \label{eqn:right_action_by_representation}
\end{equation}
with $x\in\mcal{V}$ and $g \in G$.
More generally, a nonlinear right action of a Lie group $G$ on a manifold $\mcal{M}$ is any smooth map $\theta: \mcal{M} \times G \to \mcal{M}$ satisfying 
\begin{equation}
    \theta(\theta(x, g_1), g_2) = \theta(x, g_1 g_2) \qquad \mbox{and} \qquad \theta(x, e) = x
\end{equation}
for every $x\in\mcal{M}$ and $g_1, g_2\in G$.
Figure~\ref{fig:Lie_group_action} depicts the action of a Lie group on a manifold.
We make frequent use of the maps $\theta_g = \theta (\cdot, g)$, which have smooth inverses $\theta_{g^{-1}}$, and the \emph{orbit maps} $\theta^{(x)} = \theta(x, \cdot)$.
For example, using a representation $\Phi : \SE(3) \to \GL(\R^7)$, the position $q$ and velocity $v$ of a particle in $\R^3$ 
can be rotated and translated via the action
\begin{equation}\label{eqn:SE_action}
    \theta\left( 
    \begin{bmatrix} 
    Q & b \\ 
    0 & 1\end{bmatrix}, \ \begin{bmatrix}
        q \\
        v \\
        1
    \end{bmatrix} \right)
    =
    \Phi\left( \begin{bmatrix} 
        Q^T & -Q^T b \\ 
        0 & 1 \end{bmatrix} \right) 
    \begin{bmatrix}
        q \\
        v \\
        1
    \end{bmatrix}
    = 
    \begin{bmatrix} 
    Q^T & 0 & -Q^T b \\ 
    0 & Q^T & 0 \\ 
    0 & 0 & 1\end{bmatrix}
    \begin{bmatrix}
        q \\
        v \\
        1
    \end{bmatrix} = 
    \begin{bmatrix}
        Q^T (q - b) \\
        Q^T v \\
        1
    \end{bmatrix}.
\end{equation}
The positions and velocities of $n$ particles arranged as a vector $(q_1, \ldots, q_n, v_1, \ldots, v_n, 1)$ can be simultaneously rotated and translated via an analogous representation $\Phi: \SE(3) \to \GL(\R^{6n + 1})$.

We will make use of the fact that a Lie group action is almost completely characterized by a linear map called its \emph{infinitesimal generator}.
The generator $\hat{\theta}$ assigns to each element $\xi\in\Lie(G)$ in the Lie algebra, a smooth vector field $\hat{\theta}(\xi)$ defined at $x\in\mcal{M}$ by
\begin{equation}
    \hat{\theta}(\xi)_{x} 
    = \left.\ddt \theta_{\exp(t\xi)}(x) \right\vert_{t=0}
    = \D \theta^{(x)}(e) \xi.
\end{equation}
The infinitesimal generator and its relation to the group action are illustrated in Figure~\ref{fig:Lie_group_action}.
For the linear action in \eqref{eqn:right_action_by_representation}, the infinitesimal generator is the linear vector field given by the matrix-vector product $\hat{\theta}(\xi)_{x} = -\phi(\xi) x$.
Crucially, the generator's flow is identical to the group action along the exponential curve $\exp(t\xi)$, i.e.,
\begin{equation}
    \Flow_{\hat{\theta}(\xi)}^t(x) 
    = \theta_{\exp(t\xi)}(x).
\end{equation}
For the linear right action in \eqref{eqn:right_action_by_representation}, this is easily verified by differentiation, applying \eqref{eqn:intertwining_exp_and_rep}, and the fact that solutions of smooth ordinary differential equations are unique.
For a nonlinear right action this follows from Lemma~20.14 and Proposition~9.13 in \cite{Lee2013introduction}.

\begin{remark}
    In contrast to a \emph{right action} $\theta: \mcal{M} \times G \to \mcal{M}$, a \emph{left action} $\theta^L: G \times \mcal{M} \to \mcal{M}$ satisfies $\theta^L(g_2, \theta^L(g_1,x)) = \theta^L(g_2 g_1, x)$.
    While our main results work for left actions too, e.g. $\theta^L(g,x) = \Phi(g) x$, right actions are slightly more natural because the infintesimal generator is a Lie alegbra homomorphism, i.e.,
    \begin{equation}
        \hat{\theta}([\xi, \eta]) = [\hat{\theta}(\xi), \hat{\theta}(\eta)],
    \end{equation}
    whereas this holds with a sign change for left actions.
    Every left action can be converted into an equivalent right action defined by $\theta^R(x,g) = \theta^L(g^{-1},x)$, and vice versa.
\end{remark}

\section{Fundamental operators for studying symmetry}
\label{sec:fundamental_operators}

Here we introduce our main theoretical results for studying symmetries of machine learning models by focusing on a concrete and useful special case.
The basic building blocks of the machine learning models we consider here are continuously differentiable functions $F: \mcal{V} \to \mcal{W}$ between finite-dimensional vector spaces.
The spaces of functions $\mcal{V} \to \mcal{W}$ with continuous derivatives up to order $k\in\mathbb{N}\cup\{\infty\}$ is denoted $C^k(\mcal{V};\mcal{W})$, with addition and scalar multiplication defined point-wise.
These functions could be layers of a multilayer neural network, integral kernels applied to spatio-temporal fields, or simply linear combinations of user-specified basis functions in a regression task as in~\cite{Brunton2016discovering}.
General versions of our results for sections of vector bundles are developed later in Section~\ref{sec:sections_of_vector_bundles}.
Our main results show that two families of fundamental linear operators encode the symmetries of these functions.
The fundamental operators allow us to enforce, discover, and promote symmetry in machine learning models as we describe in Sections~\ref{sec:enforcing_symmetry},~\ref{sec:discovering_symmetry},~and~\ref{sec:promoting_symmetry}.

Letting $G$ be a Lie group, we consider a right action $\Theta: \mcal{V} \times \mcal{W} \times G \to \mcal{V} \times \mcal{W}$ of the form
\begin{equation}\label{eqn:Theta_action_on_VW}
    \Theta(x,y,g) = \big(\theta(x, g),\ T(x,g) y\big),
\end{equation}
meaning $\theta$ is a general (perhaps nonlinear) right G-action on $\mcal{V}$ and $T: \mcal{V} \times G \to \GL(\mcal{W})$ is a smooth function satisfying
\begin{equation}
    T(x, g_1 g_2) = T(\theta(x, g_1), g_2) T(x, g_1) \qquad \mbox{and} \qquad 
    T(x, e) = I
\end{equation}
for all $x \in \mcal{V}$ and every $g_1, g_2 \in G$.
The infinitesimal generator of $\Theta$ at an element $\xi \in \Lie(G)$ is
\begin{equation}
    \hat{\Theta}(\xi)_{(x,y)} = \big( \hat{\theta}(\xi)_x, \ \hat{T}(\xi)_x y \big), 
    \quad \mbox{where} \quad
    \hat{T}(\xi)_x = \left.\ddt T(x, \exp(t\xi)) \right\vert_{t=0} = \D T(x,e) (0,\xi).
\end{equation}
For example, using representations $\Phi : G \to \GL(\mcal{V})$, $\Psi : G \to \GL(\mcal{W})$, the action can be defined by
\begin{equation}\label{eqn:VW_rep_actions}
    \theta(x, g) = \Phi(g^{-1}) x
    \qquad \mbox{and} \qquad
    T(x, g) y = \Psi(g^{-1}) y.
\end{equation}
With $\phi$ and $\psi$ being the corresponding Lie algebra representations, the generator has components
\begin{equation}\label{eqn:VW_rep_actions_generators}
    \hat{\theta}(\xi)_x = - \phi(\xi) x
    \qquad \mbox{and} \qquad
    \hat{T}(\xi)_x y = - \psi(\xi) y.
\end{equation}

The definition of equivariance, the symmetry group of a function, and the first family of fundamental operators are introduced by the following:
\begin{definition}
    \label{def:equivariance_real_map_version}
    We say that $F$ is \textbf{equivariant} with respect to a group element $g\in G$ if
    \begin{empheq}[box=\widefbox]{equation}
        (\mcal{K}_g F)(x) 
        := T(x, g)^{-1} F(\theta_g (x)) 
        = F(x)
        \label{eqn:transformation_operators_real_map}
    \end{empheq}
    for every $x \in \mcal{V}$.
    These elements form a subgroup of $G$ denoted $\Sym_G(F)$.
\end{definition}
Note that when the action is defined by representations \eqref{eqn:VW_rep_actions}, then \eqref{eqn:transformation_operators_real_map} becomes
\begin{equation}
    (\mcal{K}_g F)(x) 
        := \Psi(g) F(\Phi(g)^{-1} x ) 
        = F(x).
\end{equation}
The transformation operators $\mcal{K}_g$ are linear maps sending functions in $C^k(\mcal{V};\mcal{W})$ to functions in $C^k(\mcal{V};\mcal{W})$.
These fundamental operators form a group with composition $\mcal{K}_{g} \mcal{K}_{h} = \mcal{K}_{g h}$ and inversion $\mcal{K}_g^{-1} = \mcal{K}_{g^{-1}}$.
Thus, $g \mapsto \mcal{K}_g$ is an infinite-dimensional representation of $G$ in $C^k(\mcal{V};\mcal{W})$ for any $k$.
These operators are useful for studying discrete symmetries of functions.
However, for a continuous group $G$ it is impractical to work directly with the uncountable family $\{ \mcal{K}_g \}_{g\in G}$.

\begin{figure}
    \centering
    \vspace{-.3in}
    \begin{tikzonimage}[trim=20 150 100 150, clip=true, width=0.75\textwidth]{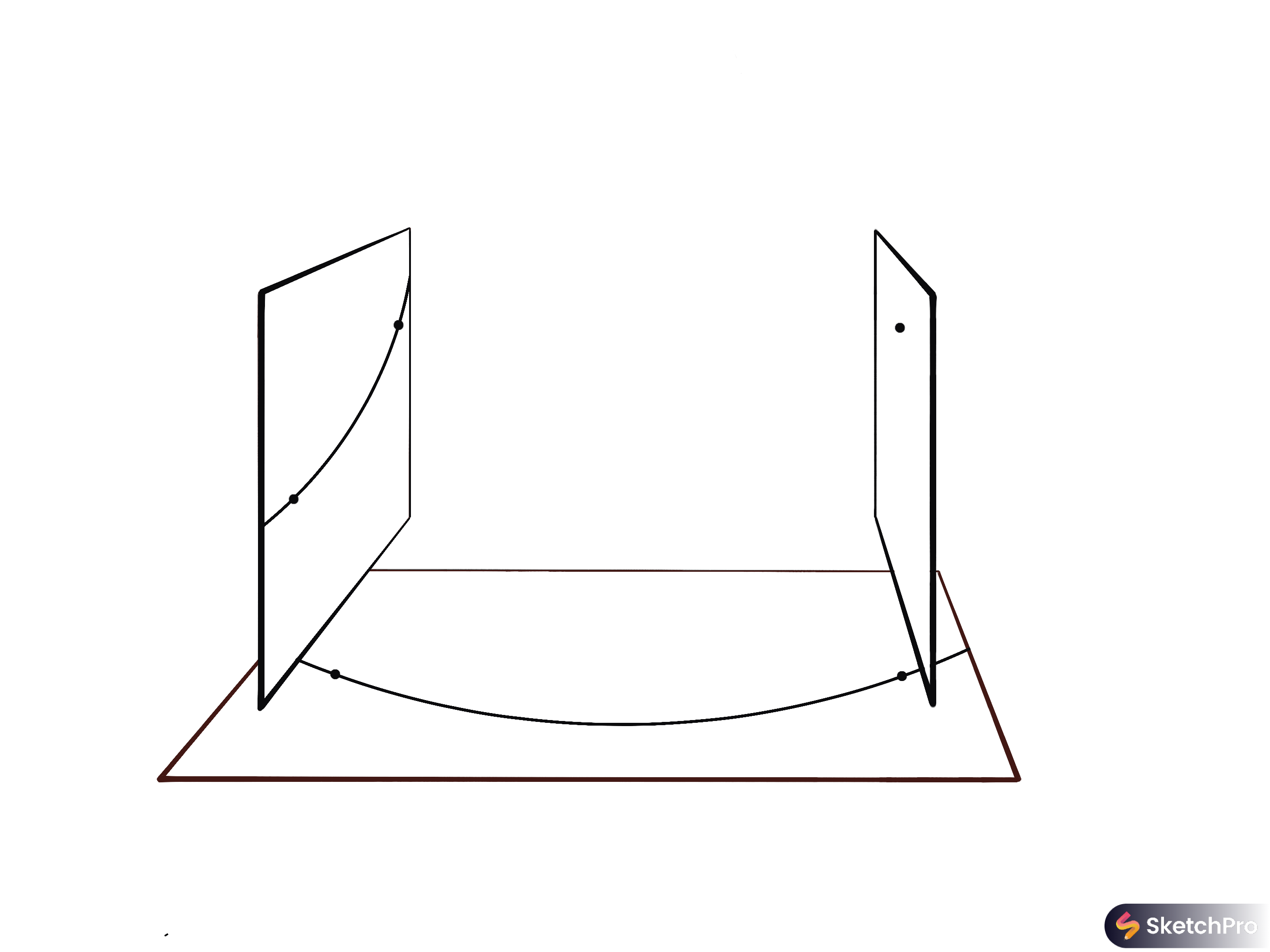}
        \node[rotate=0] at (0.5, 0.3) {\footnotesize $\mcal{V}$};
        \node[rotate=0, anchor=north] at (0.275, 0.2) {\footnotesize $x$};
        \node[rotate=0, anchor=north] at (0.749, 0.165) {\footnotesize $ \theta_{\exp(t\xi)}(x)$};
        \node[rotate=0, anchor=south] at (0.274, 0.83) {\footnotesize $\mcal{W}$};
        \node[rotate=0] at (0.260, 0.42) {\footnotesize $F(x)$};
        \draw[->] (0.240, 0.468) -- (0.330, 0.615);
        \node[rotate=0, anchor=west] at (0.335, 0.620) {\footnotesize $\mcal{L}_{\xi}F(x)$};
        \node[rotate=0, anchor=west] at (0.335, 0.790) {\footnotesize $\mcal{K}_{\exp(t\xi)}F(x)$};
        \node[rotate=0, anchor=south] at (0.770, 0.83) {\footnotesize $\mcal{W}$};
        \node[rotate=0, anchor=west] at (0.790, 0.719) {\footnotesize $F\big( \theta_{\exp(t\xi)}(x) \big)$};
        \draw[->] (0.750, 0.719) -- (0.345, 0.722);
        \node[rotate=0, anchor=north] at (0.600, 0.719) {\footnotesize $T(x, \exp(t\xi))^{-1}$};
    \end{tikzonimage}
    \vspace{-.2in}
    \caption{The fundamental operators for functions between vector spaces and linear Lie group actions defined by representations. 
    The finite transformation operators $\mcal{K}_{g}$ act on the function $F: \mcal{V} \to \mcal{W}$ by composing it with the transformation $\theta_g$ and then applying $T(x, g)^{-1}$ to the values in $\mcal{W}$. The function is $g$-equivariant when this process does not alter the function. 
    The Lie derivative $\mcal{L}_{\xi}$ is formed by differentiating $t \mapsto \mcal{K}_{\exp(t\xi)}$ at $t=0$. 
    Geometrically, $\mcal{L}_{\xi} F(x)$ is the vector in $\mcal{W}$ tangent to the curve $t \mapsto \mcal{K}_{\exp(t\xi)} F (x)$ in $\mcal{W}$ passing through $F(x)$ at $t=0$.}
    \label{fig:baby_Lie_derivative}
\end{figure}
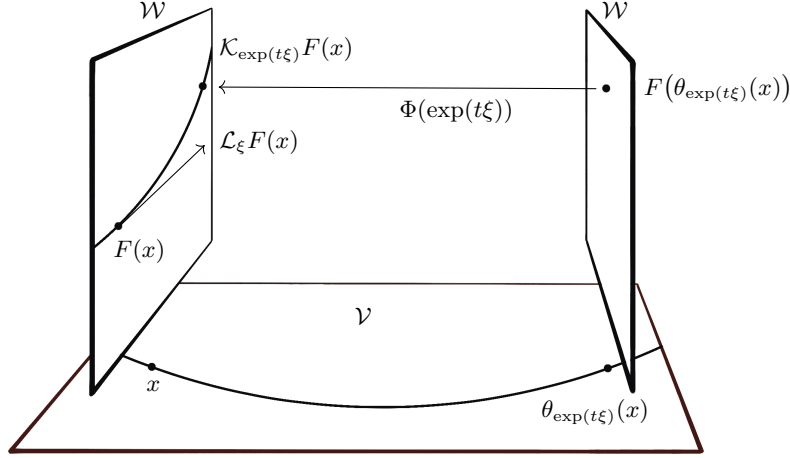

The second family of fundamental operators are the key objects we use to study continuous symmetries of functions.
These are the Lie derivatives $\mcal{L}_{\xi} : C^{1}(\mcal{V};\mcal{W}) \to C^{0}(\mcal{V};\mcal{W})$ defined along each $\xi \in \Lie(G)$ by
\begin{empheq}[box=\widefbox]{equation}
    (\mcal{L}_{\xi} F)(x) 
    = \left.\ddt\right\vert_{t=0} \big( \mcal{K}_{\exp(t \xi)} F \big)(x) 
    = \frac{\partial F(x)}{\partial x} \hat{\theta}(\xi)_{x} - \hat{T}(\xi)_x F(x).
    \label{eqn:Lie_derivative_of_real_map}
\end{empheq}
Note that when the action is defined using representations \eqref{eqn:VW_rep_actions}, then \eqref{eqn:Lie_derivative_of_real_map} becomes
\begin{equation}\label{eqn:Lie_derivative_with_reps}
    (\mcal{L}_{\xi} F)(x)
    = \psi(\xi) F(x) - \frac{\partial F(x)}{\partial x} \phi(\xi) x.
\end{equation}
Evident from \eqref{eqn:Lie_derivative_of_real_map} is the fact that the Lie derivative is linear with respect to both $\xi$ and $F$, and sends functions in $C^{k+1}$ to functions in $C^k$ for every $k \geq 0$.
The geometric construction of the fundamental operators $\mcal{K}_g$ and $\mcal{L}_{\xi}$ are depicted in Figure~\ref{fig:baby_Lie_derivative}.
It turns out (see Proposition~\ref{prop:properties_of_Lie_derivative}) that $\xi \mapsto \mcal{L}_{\xi}$ is the Lie algebra representation corresponding to $g \mapsto \mcal{K}_g$ on $C^{\infty}(\mcal{V};\mcal{W})$, meaning that on this space we have the handy relations
\begin{equation}
    \ddt \mcal{K}_{\exp(t\xi)} 
    = \mcal{L}_{\xi} \mcal{K}_{\exp(t\xi)} 
    = \mcal{K}_{\exp(t\xi)} \mcal{L}_{\xi} 
    \qquad \mbox{and} \qquad
    \mcal{L}_{[\xi,\eta]} = \mcal{L}_{\xi}\mcal{L}_{\eta} - \mcal{L}_{\eta} \mcal{L}_{\xi}.
\end{equation}
The results stated below are special cases of more general results developed later in Section~\ref{sec:sections_of_vector_bundles}.

Our first main result provides necessary and sufficient conditions for a continuously differentiable function $F:\mcal{V} \to \mcal{W}$ to be equivariant with respect to the Lie group actions on $\mcal{V}$ and $\mcal{W}$.
This generalizes the constraints derived by \cite{Finzi2021practical} for the linear layers of equivariant multilayer perceptrons by allowing for nonlinear group actions and nonlinear functions.
\begin{theorem}
    \label{thm:invariance_conditions_for_real_map}
    Let $\{ \xi_i \}_{i=1}^{q}$ generate (via linear combinations and Lie brackets) the Lie algebra $\Lie(G)$ and let $\{ g_j \}_{j=1}^{n_G-1}$ contain one element from each non-identity component of $G$.
    Then $F \in C^1( \mcal{V}; \mcal{W})$ is $G$-equivariant if and only if
    \begin{equation}
        \mcal{L}_{\xi_i} F = 0 \qquad \mbox{and} \qquad
        \mcal{K}_{g_j} F - F = 0
        \label{eqn:equivariance_everywhere_for_real_map}
    \end{equation}
    for every $i=1,\ldots,q$ and every $j=1,\ldots, n_G - 1$.
    This is a special case of Theorem~\ref{thm:equivariance_conditions_for_vb_section}.
\end{theorem}
Since the fundamental operators $\mcal{L}_{\xi}$ and $\mcal{K}_g$ are linear, Theorem~\ref{thm:invariance_conditions_for_real_map} provides linear constraints for a continuously differentiable function $F$ to be $G$-equivariant.

Our second main result shows that the continuous symmetries of a given continuously differentiable function $F:\mcal{V} \to \mcal{W}$ are encoded by its Lie derivatives.
\begin{theorem}
    \label{thm:symmetries_of_a_map}
    Given $F \in C^1( \mcal{V}; \mcal{W})$,
    the symmetry group $\Sym_G(F)$ is a closed, embedded Lie subgroup of $G$ with Lie subalgebra
    \begin{equation}\label{eqn:Sym_G_F}
        \sym_G(F) = \left\{ \xi \in \Lie(G) \ : \ \mcal{L}_{\xi} F = 0 \right\}.
    \end{equation}
    This is a special case of Theorem~\ref{thm:symmetries_of_sections}.
\end{theorem}
This result completely characterizes the identity component of the symmetry group $\Sym_G(F)$ because the connected Lie subgroups of $G$ are in one-to-one correspondence with Lie subalgebras of $\Lie(G)$ (see Theorem~19.26 in \cite{Lee2013introduction}).
Moreover, the Lie subalgebra of symmetries of a $C^1$ function $F$ can be identified via linear algebra because $\sym_g(F)$ is the nullspace of the linear operator $L_F : \Lie(G) \to C^0(\mcal{V};\mcal{W})$ defined by
\begin{equation} \label{eqn:L_F_op_for_vector_space_case}
    L_F : \xi \mapsto \mcal{L}_{\xi} F.
\end{equation}
Discretization methods suitable for linear-algebraic computations with the fundamental operators will be discussed in Section~\ref{sec:discretization}.
The key point is that when the functions $F$ lie in a finite-dimensional subspace $\mcal{F} \subset C^{1}(\mcal{V};\mcal{W})$, the ranges of the restricted Lie derivatives $\{ \left.\mcal{L}_{\xi}\right\vert_{\mcal{F}} \}_{\xi \in \Lie(G)}$, hence, also the ranges of $\{ L_F \}_{F\in\mcal{F}}$, are contained in a corresponding finite-dimensional subspace $\mcal{F}' \subset C^{0}(\mcal{V};\mcal{W})$ on which inner products can be defined using sampling or quadrature.

The preceding two theorems already show the duality between enforcing and discovering continuous symmetries with respect to the Lie derivative, viewed as a bilinear form $(\xi, F) \mapsto \mcal{L}_{\xi} F$.
To discover symmetries, we seek generators $\xi \in \Lie(G)$ satisfying $\mcal{L}_{\xi} F = 0$ for a known function $F$.
On the other hand, to enforce a connected group of symmetries, we seek functions $F$ satisfying $\mcal{L}_{\xi_i} F = 0$ with known generators $\xi_1, \ldots, \xi_{q}$ of $\Lie(G)$. 

\section{Enforcing symmetry with linear constraints}
\label{sec:enforcing_symmetry}
Methods to enforce symmetry in neural networks and other machine learning models have been studied extensively, as we reviewed briefly in Section~\ref{subsec:related_work_enforcing_symmetry}.
A unifying theme in these techniques has been the use of linear constraints to enforce symmetry~\citep{Finzi2021practical,Loiseau2017jfm,weiler20183d,Cohen2019general,Ahmadi2020learning_short}.
The purpose of this section is to show how several of these methods can be understood in terms of the fundamental operators and linear constraints provided by Theorem~\ref{thm:invariance_conditions_for_real_map}.

\subsection{Multilayer perceptrons}
Enforcing symmetry in multilayer percetrons was studied by \cite{Finzi2021practical}.
They provide a practical method based on enforcing linear constraints on the weights defining each layer of a neural network.
The network uses specialized nonlinearities that are automatically equivariant, meaning that the constraints need only be enforced on the linear component of each layer.
We show that the constraints derived by \cite{Finzi2021practical} are the same as those given by Theorem~\ref{thm:invariance_conditions_for_real_map}.

Specifically, each linear layer $F^{(l)}: \mcal{V}_{l-1} \to \mcal{V}_l$, for $l=1,\ldots,L$, is defined by
\begin{equation}
    F^{(l)}(x) = W^{(l)} x + b^{(l)},
\end{equation}
where $W^{(l)}$ are weight matrices and $b^{(l)}$ are bias vectors.
Defining group representations $\Phi_l: G \to \GL(\mcal{V}_l)$ for each layer, yields fundamental operators given by
\begin{align}
    \mcal{K}_{g} F^{(l)}(x) 
    &= \Phi_{l}(g) W^{(l)} \Phi_{l-1}(g)^{-1} x + \Phi_{l}(g) b^{(l)} \\
    \mcal{L}_{\xi} F^{(l)}(x) 
    &= \big(\phi_{l}(\xi) W^{(l)} - W^{(l)} \phi_{l-1}(\xi)\big) x +  \phi_{l}(\xi) b^{(l)}.
\end{align}
Let $\{ \xi_i \}_{i=1}^{q}$ generate $\Lie(G)$ and let $\{g_j\}_{j=1}^{n_G - 1}$ consist of an element from each non-identity component of $G$.
Using the fundamental operators and Theorem~\ref{thm:invariance_conditions_for_real_map}, it follows that the layer $F^{(l)}$ is $G$-equivariant if and only if the weights and biases satisfy 
\begin{equation}
    \phi_{l}(\xi_i) W^{(l)} = W^{(l)} \phi_{l-1}(\xi_i), 
    \quad \mbox{and} \quad
    \Phi_{l}(g_j) W^{(l)} = W^{(l)} \Phi_{l-1}(g_j),
\end{equation}
\begin{equation}
    \phi_{l}(\xi_i) b^{(l)} = 0,
    \quad \mbox{and} \quad
    \Phi_{l}(g_j) b^{(l)} = b^{(l)}
\end{equation}
for every $i = 1, \ldots,q$ and $j = 1, \ldots, n_g-1$.
These are the same as the linear constraints one derives using the method by \cite{Finzi2021practical}.
The equivariant linear layers are then combined with specialized equivariant nonlinearities $\sigma^{(l)}: \mcal{V}_l \to \mcal{V}_l$
to produce an equivariant network
\begin{equation}
    F = \sigma^{(L)} \circ F^{(L)} 
    \circ \cdots \circ 
    \sigma^{(1)} \circ F^{(1)}: \mcal{V}_0 \to \mcal{V}_L.
\end{equation}
The composition of equivariant functions is equivariant, as one can easily check using Definition~\ref{def:equivariance_real_map_version}.

\subsection{Neural operators acting on fields}
\label{subsec:NNs_acting_on_fields}
Enforcing symmetry in neural networks acting on spatial fields has been studied extensively by \cite{weiler20183d, cohen2018spherical, esteves2018learning, Kondor2018generalization, Cohen2019general} among others.
Many of these techniques use integral operators to define equivariant linear layers, which are coupled with equivariant nonlinearities, such as the gated nonlinearities proposed by \cite{weiler20183d}.
Networks built by composing integral operators with nonlinearities constitute a large class of \emph{neural operators} described by \cite{Kovachki2023neural, Goswami2023physics, Boulle2023mathematical}.
The key task is to identify appropriate bases for equivariant kernels.
For certain groups, such as the Special Euclidean group $G = \SE(3)$, bases can be constructed explicitly using spherical harmonics, as in \cite{weiler20183d}. 
We show that equivariance with respect to arbitrary group actions can be enforced via linear constraints on the integral kernels derived using the fundamental operators introduced in Section~\ref{sec:fundamental_operators}.
Appropriate bases of kernel functions can then be constructed numerically by computing an appropriate nullspace, as is done by \cite{Finzi2021practical} for multilayer perceptrons.

For the sake of simplicity we consider integral operators acting on vector-valued functions $F:\R^m\to \mcal{V}$, where $\mcal{V}$ is a finite-dimensional vector space.
Later on in Section~\ref{subsec:equivariant_integral_operators} we study higher-order integral operators acting on sections of vector bundles.
If $\mcal{W}$ is another finite-dimensional vector space, an integral operator acting on $F$ to produce a new function $\R^n\to \mcal{W}$ is defined by
\begin{equation}
    \mcal{T}_K F(x) = \int_{\R^m} K(x,y) F(y) \td y,
    \label{eqn:intergral_operator_on_Rn}
\end{equation}
where the \emph{kernel} function $K$ provides a linear map $K(x,y): \mcal{V} \to \mcal{W}$ at each $(x,y) \in \R^{n}\times\R^m$.
In other words, the kernel is a function on $\R^{n}\times\R^m$ taking values in the tensor product space $\mcal{W} \otimes \mcal{V}^*$, where $\mcal{V}^*$ denotes the algebraic dual of $\mcal{V}$.
Many of the neural operator architectures described by \cite{Kovachki2023neural, Goswami2023physics, Boulle2023mathematical} are constructed by composing layers defined by integral operators \eqref{eqn:intergral_operator_on_Rn} with nonlinear activation functions, usually acting pointwise.
The kernel functions $K$ are optimized during training of the neural operator.

With group actions defined by representations on $\R^m, \R^n, \mcal{V}, \mcal{W}$, functions $F:\R^m \to \mcal{V}$ transform according to 
\begin{equation}
    \mcal{K}_g^{(\R^m,\mcal{V})} F(x) = \Phi_{\mcal{V}}(g) F(\Phi_{\R^m}(g)^{-1} x)
\end{equation}
for $g \in G$.
Likewise, 
functions $\R^n \to \mcal{W}$ transform via an analogous operator $\mcal{K}_g^{(\R^n,\mcal{W})}$.
\begin{definition}
    \label{def:equivariance_of_integral_operator_linear_case}
    The integral operator $\mcal{T}_K$ in \eqref{eqn:intergral_operator_on_Rn} is \textbf{equivariant} with respect to $g \in G$ when
    \begin{equation}
         \mcal{K}_{g}^{(\R^n,\mcal{W})} \circ \mcal{T}_K \circ \mcal{K}_{g^{-1}}^{(\R^m,\mcal{V})} = \mcal{T}_K.
    \end{equation}
    The elements $g$ satisfying this equation form a subgroup of $G$ denoted $\Sym_G(\mcal{T}_K)$.
\end{definition}
By changing variables in the integral, the operator on the left is given by
\begin{equation}
    \mcal{K}_{g}^{(\R^n,\mcal{W})} \circ \mcal{T}_K \circ \mcal{K}_{g^{-1}}^{(\R^m,\mcal{V})} F(x)
    = \int_{\R^m} \mcal{K}_g K(x,y) F(y) \td y,
\end{equation}
where
\begin{empheq}[box=\widefbox]{equation}
    \mcal{K}_g K (x,y) 
    = \Phi_{\mcal{W}}(g) K\big(\Phi_{\R^n}(g)^{-1} x, \Phi_{\R^m}(g)^{-1} y\big) \Phi_{\mcal{V}}(g)^{-1}  \det\big[ \Phi_{\R^m}(g)^{-1} \big].
    \label{eqn:transformation_for_linear_integral_kernels_on_Rn}
\end{empheq}
The following result provides equivariance conditions in terms of the kernel, generalizing Lemma~1 in \cite{weiler20183d} beyond a specific action of $\SE(3)$.
\begin{proposition}
    \label{prop:symmetries_of_linear_integral_operator}
    Let $K$ be continuous and suppose that $\mcal{T}_K$ acts on a function space containing all smooth, compactly supported fields. Then 
    \begin{equation}
        \Sym_G(\mcal{T}_K) = \left\{ g \in G \ : \ \mcal{K}_g K = K \right\}.
    \end{equation}
    We give a proof in Appendix~\ref{app:proofs_of_minor_results}
\end{proposition}

The Lie derivative of a continuously differentiable kernel function is given by
\begin{empheq}[box=\widefbox]{multline}
    \mcal{L}_{\xi} K(x,y) 
    = \phi_{\mcal{W}}(\xi)K(x,y) 
    - K(x,y) \phi_{\mcal{V}}(\xi)
    - K(x,y) \Tr[\phi_{\R^m}(\xi)] \\
    - \frac{\partial K(x,y)}{\partial x} \phi_{\R^n}(\xi)x 
    - \frac{\partial K(x,y)}{\partial y} \phi_{\R^m}(\xi)y
    \label{eqn:Lie_derivative_for_linear_integral_kernels_on_Rn}
\end{empheq}
The operators $\mcal{K}_g$ and $\mcal{L}_{\xi}$ are the fundamental operators from Section~\ref{sec:fundamental_operators} because 
the transformation law for the kernel can be written as
\begin{equation}
    \mcal{K}_g K 
    = \Phi_{\mcal{W} \otimes \mcal{V}^*}(g) K \circ \Phi_{\R^m\times \R^m}(g)^{-1},
\end{equation}
where 
\begin{equation}
\begin{aligned}
    \Phi_{\R^n\times \R^m}(g) &: (x,y) \mapsto \left(\Phi_{\R^n}(g) x, \Phi_{\R^m}(g) y\right), \\
    \Phi_{\mcal{W} \otimes \mcal{V}^*}(g) &: A \mapsto \Phi_{\mcal{W}}(g) A \Phi_{\mcal{V}}(g)^{-1} \det\big[ \Phi_{\R^m}(g)^{-1} \big]
\end{aligned}
\end{equation}
are representations of $G$ in $\R^m\times \R^m$ and $\mcal{W} \otimes \mcal{V}^*$.

As an immediate consequence of Theorem~\ref{thm:invariance_conditions_for_real_map}, we have the following corollary establishing linear constraints for the kernel to produce an equivariant integral operator.
\begin{corollary}
    Let $\{ \xi_i \}_{i=1}^{q}$ generate the Lie algebra $\Lie(G)$ and let $\{ g_j \}_{j=1}^{n_G-1}$ contain one element from each non-identity component of $G$.
    Under the same hypotheses as Proposition~\ref{prop:symmetries_of_linear_integral_operator} and assuming $K$ is continuously differentiable, the integral operator $\mcal{T}_K$ in \eqref{eqn:intergral_operator_on_Rn} is $G$-equivariant in the sense of Definition~\ref{def:equivariance_of_integral_operator_linear_case} if and only if
    \begin{equation}
        \mcal{L}_{\xi_i} K = 0 \qquad \mbox{and} \qquad
        \mcal{K}_{g_j} K - K = 0
    \end{equation}
    for every $i=1,\ldots,q$ and every $j=1,\ldots, n_G - 1$.
\end{corollary}
These linear constraint equations must be satisfied to enforced equivariance with respect to a known symmetry $G$ in the machine learning process.
By discretizing the operators $\mcal{K}_g$ and $\mcal{L}_{\xi}$, as discussed later in Section~\ref{sec:discretization}, one can solve these constraints numerically to construct a basis of kernel functions for equivariant integral operators.

As an immediate consequence of Theorem~\ref{thm:symmetries_of_a_map}, the following result shows that the Lie derivative of the kernel encodes the continuous symmetries of a given integral operator.
\begin{corollary}
    Under the same hypotheses as Proposition~\ref{prop:symmetries_of_linear_integral_operator} and assuming $K$ is continuously differentiable, it follows that
    $\Sym_G(\mcal{T}_K)$ is a closed, embedded Lie subgroup of $G$ with Lie subalgebra
    \begin{equation}
        \sym_G(\mcal{T}_K) = \left\{ \xi\in\Lie(G) \ : \ \mcal{L}_{\xi} K = 0 \right\}.
    \end{equation}
\end{corollary}
This result will be useful for methods that promote symmetry of the integral operator, as we describe later in Section~\ref{sec:promoting_symmetry}.

\section{Discovering symmetry by computing nullspaces}
\label{sec:discovering_symmetry}
In this section we show that in a wide range of settings, the continuous symmetries of a manifold, point cloud, or map can be recovered by computing the nullspace of a linear operator.
For functions, this is already covered by Theorem~\ref{thm:symmetries_of_a_map}, which allows us to compute the connected subgroup of symmetries by identifying its Lie subalgebra
\begin{equation}
    \sym_G(F) = \Null(L_F)
\end{equation}
where $L_F: \xi \mapsto \mcal{L}_{\xi}$ is the linear operator defined by \eqref{eqn:L_F_op_for_vector_space_case}.
Hence, if a machine learning model $F$ has a symmetry group $\Sym_G(F)$, then its Lie algebra is equal to the nullspace of $L_F$.

This section explains how this is actually a special case of a more general result allowing us to reveal the symmetries of submanifolds via the nullspace of a closely related operator.
We begin with the more general case where we study the symmetries of a submanifold of Euclidean space, and we explain how to recover symmetries from point clouds approximating submanifolds.
The Lie derivative described in Section~\ref{sec:fundamental_operators} is then recovered when the submanifold is the graph of a function.
We also briefly describe how the fundamental operators from Section~\ref{sec:fundamental_operators} can be used to recover symmetries and conservation laws of dynamical systems.


\subsection{Symmetries of submanifolds}
\label{subsec:submanifolds_of_Rd}
We begin by studying the symmetries of smooth submanifolds $\mcal{M}$ of Euclidean space $\R^d$ using an approach similar to \cite{Cahill2023Lie}.
However, we use a different operator that generalizes more naturally to nonlinear group actions on arbitrary manifolds (see Section~\ref{sec:submanifolds_and_tangency}) and recovers the Lie derivative (see Section~\ref{eqn:functions_as_submanifolds}).
With right action $\theta:\R^d \times G \to \R^d$ of a Lie group, we define invariance of a submanifold as follows:
\begin{definition}
    A submanifold $\mcal{M}\subset\R^d$ is \textbf{invariant} with respect to a group element $g\in G$ if
    \begin{equation}
        \theta_g(z) \in \mcal{M}
    \end{equation}
    for every $z\in\mcal{M}$.
    These elements form a subgroup of $G$ denoted $\Sym_G(\mcal{M})$.
\end{definition} 
The subgroup of symmetries of a submanifold is characterized by the following theorem.
\begin{theorem} \label{thm:submanifolds_of_Rd}
    Let $\mcal{M}$ be a smooth, closed, embedded submanifold of $\R^d$.
    Then $\Sym_G(\mcal{M})$ is a closed, embedded Lie subgroup of $G$ whose Lie subalgebra is 
    \begin{equation}
        \sym_G(\mcal{M}) = \{ \xi \in \Lie(G) \ : \ \hat{\theta}(\xi)_{z} \in T_z \mcal{M} \quad \forall z\in\mcal{M} \}.
    \end{equation}
    This is a special case of Theorem~\ref{thm:symmetries_of_a_submanifold}.
\end{theorem}
The meaning of this result and its practical use for detecting symmetry are illustrated in Figure~\ref{fig:manifold_tangent_symmetry}.

To reveal the connected component of $\Sym_G(\mcal{M})$, we let $P_z:\R^d \to \R^d$ 
be a family of linear projections onto $T_z \mcal{M} \subset \R^d$.
These are assumed to vary continuously with respect to $z\in\mcal{M}$.
Then under the assumptions of the above theorem, $\sym_G(\mcal{M})$ is the nullspace of the symmetric, positive-semidefinite operator $S_{\mcal{M}}:\Lie(G) \to \Lie(G)$ defined by
\begin{empheq}[box=\widefbox]{equation}
    \big\langle \eta, \ S_{\mcal{M}} \xi \big\rangle_{\Lie(G)} 
    = \int_{\mcal{M}} \hat{\theta}(\eta)_{z}^T (I - P_z)^T (I - P_z) \hat{\theta}(\xi)_{z} \ \td \mu(z)
    \label{eqn:symmetry_operator_for_submanifolds_of_Rd}
\end{empheq}
for every $\eta,\xi\in\Lie(G)$.
We see in Figure~\ref{fig:manifold_tangent_symmetry} that $(I - P_z) \hat{\theta}(\xi)_{z}$ measures the component of the infinitesimal generator not tangent to the submanifold at $z$.
Here, $\mu$ is any strictly positive measure on $\mcal{M}$ that makes all of these integrals finite.
The above formula is useful for computing the matrix of $S_{\mcal{M}}$ in an orthonormal basis for $\Lie(G)$.

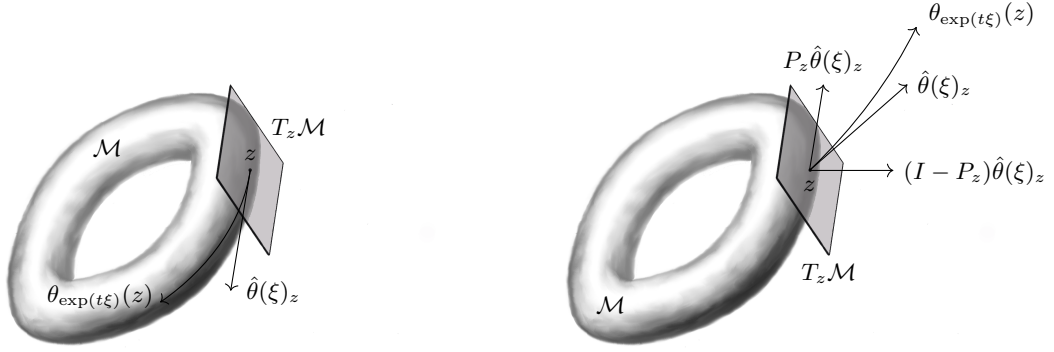
\begin{figure}
    \centering
    \begin{tikzonimage}[trim=250 200 150 150, clip=true, width=0.45\textwidth]{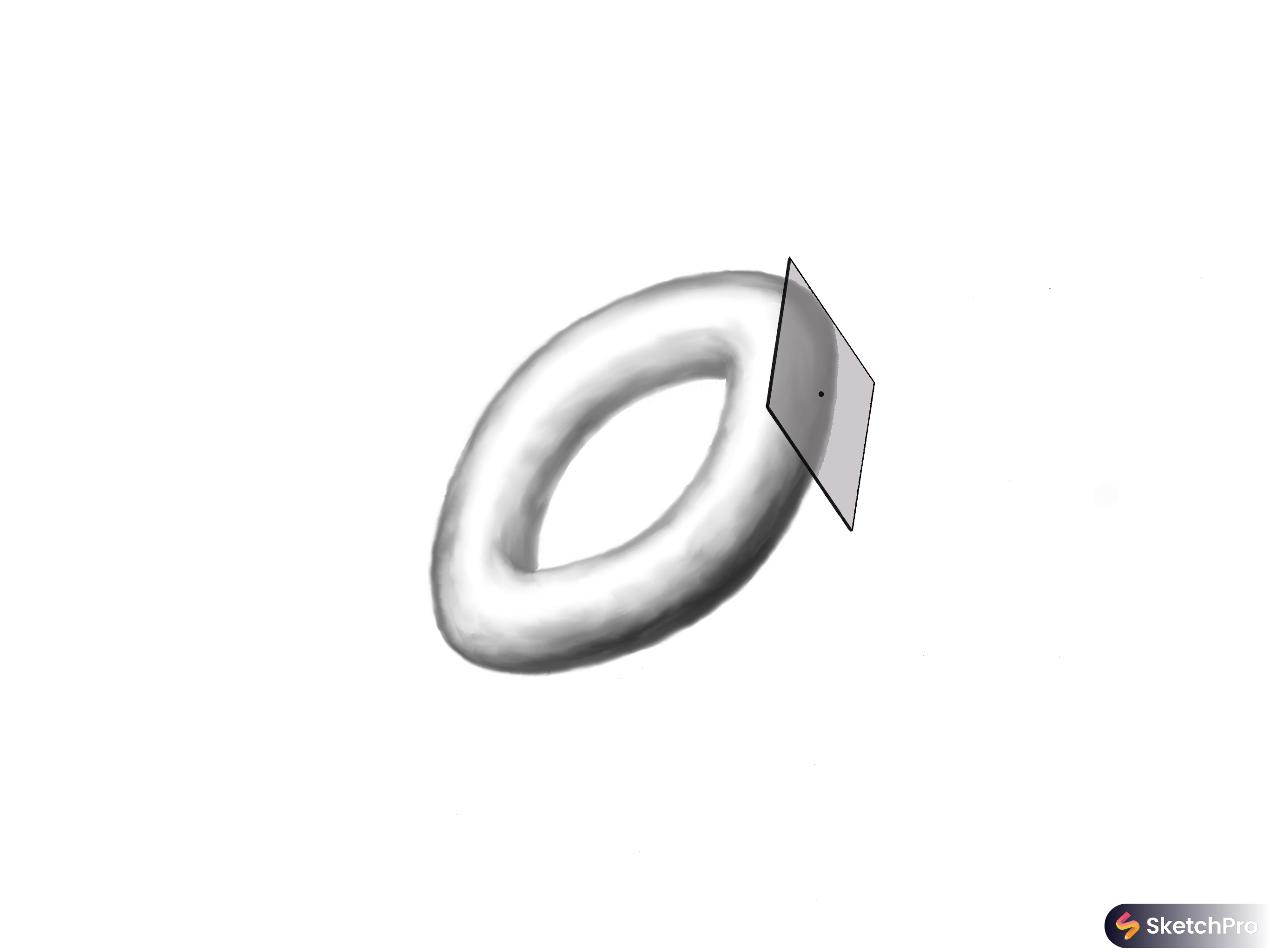}
        \node[rotate=0] at (0.40, 0.65) {\footnotesize $\mcal{M}$};
        \node[rotate=0,anchor=west] at (0.675, 0.70) {\footnotesize $T_z \mcal{M}$};
        \def\zx{0.655}; \def\zy{0.593};
        \def\vx{-0.7*0.050}; \def\vy{-0.7*0.440};
        \node[rotate=0,anchor=south] at (\zx, \zy) {\footnotesize $z$};
        \draw[->] (\zx, \zy) -- (\zx+\vx, \zy+\vy);
        \node[rotate=0,anchor=west] at (\zx+\vx+0.01, \zy+\vy) {\footnotesize $\hat{\theta}(\xi)_{z}$};
        \draw[->] (\zx, \zy) arc[radius=0.56, start angle=-6, end angle=-45];
        \node[rotate=0,anchor=east] at (0.5, 0.27) {\footnotesize $\theta_{\exp(t\xi)}(z)$};
    \end{tikzonimage}
    \hspace{-0.25cm}
    \begin{tikzonimage}[trim=250 200 150 150, clip=true, width=0.45\textwidth]{Figures/manifold_tangent.png}
        \node[rotate=0] at (0.30, 0.25) {\footnotesize $\mcal{M}$};
        \node[rotate=0,anchor=north] at (0.690, 0.380) {\footnotesize $T_z \mcal{M}$};
        \def\zx{0.655}; \def\zy{0.593};
        \def\Pvx{0.025}; \def\Pvy{0.220};
        \def\Nvx{0.15}; \def\Nvy{0.00};
        \def\vx{\Pvx+\Nvx}; \def\vy{\Pvy+\Nvy};
        \node[rotate=0,anchor=north] at (\zx, \zy) {\footnotesize $z$};
        \draw[->] (\zx, \zy) -- (\zx+\Pvx, \zy+\Pvy);
        \node[rotate=0,anchor=south] at (\zx+\Pvx, \zy+\Pvy) {\footnotesize $P_z \hat{\theta}(\xi)_{z}$};
        \draw[->] (\zx, \zy) -- (\zx+\Nvx, \zy+\Nvy);
        \node[rotate=0,anchor=west] at (\zx+\Nvx, \zy+\Nvy) {\footnotesize $(I-P_z) \hat{\theta}(\xi)_{z}$};
        \draw[->] (\zx, \zy) -- (\zx+\vx, \zy+\vy);
        \node[rotate=0,anchor=west] at (\zx+\vx, \zy+\vy) {\footnotesize $\hat{\theta}(\xi)_{z}$};
        \draw[->] (\zx, \zy) arc[radius=1.4, start angle=-36, end angle=-19];
        \node[rotate=0,anchor=west] at (0.850, 0.99) {\footnotesize $\theta_{\exp(t\xi)}(z)$};
    \end{tikzonimage}
    \caption{Tangency of infinitesimal generators and symmetries of submanifolds. The infinitesimal generator $\hat{\theta}(\xi)$ is everywhere tangent to the submanifold $\mcal{M}$ if and only if the curves $t\mapsto\theta_{\exp(t\xi)}(z)$, with $z\in\mcal{M}$, lie in $\mcal{M}$ for all $t$. 
    The Lie algebra elements $\xi$ satisfying this tangency condition form the Lie subalgebra of symmetries of $\mcal{M}$. 
    To test for tangency of the infinitesimal generator we use a family of projections $P_z$ onto the tangent spaces $T_z \mcal{M}$ for every $z\in\mcal{M}$.
    Specifically, $(I - P_z) \hat{\theta}(\xi)_z$ is the component of the infinitesimal generator at $z$ that does not lie tangent to $\mcal{M}$. 
    Hence, $\xi$ generates a symmetry of $\mcal{M}$ if and only if $(I - P_z) \hat{\theta}(\xi)_z = 0$ for all $z\in\mcal{M}$.}
    \label{fig:manifold_tangent_symmetry}
\end{figure}

Alternatively, when the dimension of $G$ is large, one can compute the nullspace using a Krylov algorithm such as the one described in \cite{Finzi2021practical}.
Such algorithms rely solely on queries of $S_{\mcal{M}}$ acting on vectors $\xi\in\Lie(G)$.
When $\theta_g(z) = \Phi(g^{-1}) z$ and $\hat{\theta}(\xi)_{z} = - \phi(\xi)z$ are given by a Lie group representation (see Section~\ref{subsec:representations_actions_generators}), then the operator defined in \eqref{eqn:symmetry_operator_for_submanifolds_of_Rd} is given explicitly by
\begin{equation}
    S_{\mcal{M}} \xi = \int_{\mcal{M}} \D\Phi(e)^*\left[ (I-P_z)^T (I-P_z) \phi(\xi) z z^T \right] \ \td \mu(z),
\end{equation}
where $\D\Phi(e)^*:\R^{d\times d} \to \Lie(G)$ is the adjoint of $\D\Phi(e):\Lie(G) \to \R^{d\times d}$.
If $G \subset \R^{d\times d}$ is a matrix Lie group and $\Phi$ is the identity representation, then $\D\Phi(e)$ is the injection $\Lie(G) \hookrightarrow \R^{d\times d}$. When $\Lie(G) \subset \R^{d\times d}$ inherits its inner product from $\R^{d\times d}$, then $\D\Phi(e)^*$ is the orthogonal projection of $\R^{d\times d}$ onto $\Lie(G)$.
For example, if $\Phi$ is the identity representation of $\SE(d)$ in $\R^{d+1}$ with the inner product on $\se(d) \subset \R^{(d+1) \times (d+1)}$ given by the usual inner product of matrices $\langle M_1, M_2 \rangle = \Tr(M_1^T M_2)$, then it can be readily verified that
\begin{equation}
    \D \Phi(e)^* 
    \left(\begin{bmatrix}
        A & b \\
        c^T & a
    \end{bmatrix}\right)
    = \begin{bmatrix}
        \frac{1}{2} (A - A^T) & b \\
        0 & 0
    \end{bmatrix},
    \qquad A\in\R^{d\times d}, \ b \in \R^d, \ c \in \R^d, \ a\in\R.
\end{equation}

In practice, one can use sample points $z_i$ on the manifold to obtain a Monte-Carlo estimate of $S_{\mcal{M}}$ with approximate projections $P_{z_i}$ computed using local principal component analysis (PCA), as described in \cite{Cahill2023Lie}.
More accurate estimates of the tangent spaces can be obtained using the methods in \cite{Berry2020spectral}.
Assuming the $P_{z_i}$ are accurate, the following proposition shows that the correct Lie subalgebra of symmetries is revealed using finitely many sample points $z_i$.
However, this result does not tell us how many samples to use, or even when to stop sampling.

\begin{proposition}\label{prop:manifold_nullspace_convergence}
    Let $\mu$ be a strictly positive probability measure on a smooth manifold $\mcal{M}$ such that $\langle \xi, S_{\mcal{M}} \xi \rangle < \infty$ for every $\xi \in \Lie(G)$.
    Let $z_i$ be drawn independently from the distribution $\mu$ and let $S_m: \Lie(G) \to \Lie(G)$ be defined by
    \begin{equation}
        \big\langle \eta, \ S_{m} \xi \big\rangle_{\Lie(G)} 
    = \frac{1}{m} \sum_{i=1}^{m} \hat{\theta}(\eta)_{z_i}^T (I - P_{z_i})^T (I - P_{z_i}) \hat{\theta}(\xi)_{z_i}.
    \end{equation}
    Then there is almost surely an integer $M_0$ such that for every $m \geq M_0$ we have $\Null(S_m) = \sym_G(\mcal{M})$.
    We provide a proof in Appendix~\ref{app:proofs_of_minor_results}.
\end{proposition}

\subsection{Symmetries of functions as symmetries of submanifolds}
\label{eqn:functions_as_submanifolds}

The method described for studying symmetries of submanifolds can be applied to smooth maps between finite-dimensional vector spaces by identifying the map $F:\R^m \to \R^n$ with its graph
\begin{equation}
    \graph(F) = \{ (x, F(x)) \in \R^m\times\R^n \ : \ x \in \R^m \}.
\end{equation}
The graph is a smooth, closed, embedded submanifold of the space $\R^m\times\R^n$ by Proposition~5.7 in \cite{Lee2013introduction}.
We show that this approach recovers the Lie derivative and our result in Theorem~\ref{thm:symmetries_of_a_map}.

We consider the action of $G$ on $\R^m\times\R^n$ defined by \eqref{eqn:Theta_action_on_VW} with $\mcal{V} = \R^m$ and $\mcal{W} = \R^n$.
Defining the smoothly-varying family of projections 
\begin{equation}
    P_{(x,F(x))} = \begin{bmatrix}
        I & 0 \\
        \D F(x) & 0
    \end{bmatrix}
    \label{eqn:graph_projection}
\end{equation}
onto $T_{(x,F(x))}\graph(F)$,
it is easy to check that
\begin{equation}
    \begin{bmatrix}
        0 \\
        \mcal{L}_\xi F (x)
    \end{bmatrix}
    = - \underbrace{\left( \begin{bmatrix}
        I & 0 \\
        0 & I
    \end{bmatrix} - \begin{bmatrix}
        I & 0 \\
        \D F(x) & 0
    \end{bmatrix} \right)
    \vphantom{\begin{bmatrix}
        \hat{\theta}(\xi)_x \\
        \hat{T}(\xi)_x F(x)
    \end{bmatrix}}
    }_{I-P_{(x,F(x))}} 
    \underbrace{
    \begin{bmatrix}
        \hat{\theta}(\xi)_x \\
        \hat{T}(\xi)_x F(x)
    \end{bmatrix}}_{\hat{\Theta}(\xi)_{(x,F(x))}}.
    \label{eqn:Lie_derivative_as_projection_for_real_maps}
\end{equation}
We note that this is a special case of Theorem~\ref{thm:characterization_of_Lie_derivative} describing the Lie derivative in terms of a projection onto the tangent space of a function's graph.
The resulting operator $S_{\graph(F)}$ in \eqref{eqn:symmetry_operator_for_submanifolds_of_Rd} is given by
\begin{equation}
    \left\langle \eta, S_{\graph(F)} \xi \right\rangle_{\Lie(G)} 
    = \int_{\R^m} (\mcal{L}_\eta F(x))^T \mcal{L}_\xi F(x) \ \td \mu(x),
\end{equation}
for $\eta,\xi\in\Lie(G)$ and an appropriate positive measure $\mu$ on $\R^m$ that makes the integrals finite.
Therefore, Theorem~\ref{thm:symmetries_of_a_map} is recovered from the result on symmetries of submanifolds in Theorem~\ref{thm:submanifolds_of_Rd}.

Related quantities have been used to study the symmetries of trained neural networks, with the $F$ being the network and its derivatives computed via back-propagation.
The quantity $\left\langle \xi, S_{\graph(F)} \xi \right\rangle_{\Lie(G)} = \Vert \mcal{L}_\xi F \Vert_{L^2(\mu)}$ was used by \citet{Gruver2022Lie} to construct the Local Equivariant Error or (LEE), measuring the extent to which a trained neural network $F$ fails to respect symmetries in the one-parameter group $\{ \exp(t\xi) \}_{t\in\R}$.
The nullspace of $\xi \mapsto \mcal{L}_{\xi} F$ in the special case where $T(x,g) = I$ acts trivially was used by \cite{Moskalev2022liegg} to identify the connected subgroup with respect to which a given network is invariant.

By viewing a function as a submanifold, we propose a simple data-driven technique for estimating the Lie derivative and subgroup of symmetries of the function.
To approximate $\mcal{L}_\xi F$, $S_{\graph(F)}$, and $\sym_G(F)$ using input-output pairs $(x_i, y_i=F(x_i))$, it suffices to approximate the projection in \eqref{eqn:graph_projection} using these data.
To do this, we could obtain $(n+m) \times m$ matrices $U_i$ with columns spanning $T_{(x_i,y_i)} \graph(F)$ by applying local PCA to the data $z_i = (x_i, y_i)$, or by pruning the frames computed in \cite{Berry2020spectral}.
With $E = \begin{bmatrix}
    I_{m\times m} & 0_{m\times n}
\end{bmatrix}$, the projection in \eqref{eqn:graph_projection} is given by 
\begin{equation}
    P_{z_i} = U_i (E U_i)^{-1} E
\end{equation}
because any projection is uniquely determined by its range and nullspace (see Section~5.9 of \cite{Meyer2000matrix}).
Evaluating the generator $\hat{\Theta}(\xi)_{(x_i,y_i)}$ and applying the projection as in \eqref{eqn:Lie_derivative_as_projection_for_real_maps} gives a potentially simple way to approximate $(\mcal{L}_\xi F)(z_i)$, $S_{\graph(F)}$, and $\sym_G(F)$ using the input-output pairs.
However, many such pairs would be needed since the tangent space to the graph of $F$ at $x_i$ is well-approximated by local PCA only when there are at least $m$ neighboring samples sufficiently close to $x_i$.
Even more samples are likely needed when they are noisy.
The convergence properties of the spectral methods in \cite{Berry2020spectral} are better, but they still require enough samples to estimate integrals over $\graph(F)$ using quadrature or Monte-Carlo.


\subsection{Symmetries and conservation laws of dynamical systems}
\label{subsec:dynamical_systems}
Here, we consider the case when $F:\R^n \to \R^n$ is a smooth function defining a dynamical system
\begin{equation}
    \ddt x(t) = F(x(t))
    \label{eqn:ds}
\end{equation}
with state variables $x(t)\in\R^n$.
The solution of this equation is described by the flow map $\Flow: (t, x(\tau)) \mapsto x(\tau + t)$, which is defined on a maximal connected open set $D$ containing $0\times\R^n$.
In many cases we write $\Flow^t(\cdot) = \Flow(t, \cdot)$.
Given a right action $\theta: \R^n \times G \to \R^n$, equivariance for the dynamical system is defined as follows:
\begin{definition}
    The dynamical system in \eqref{eqn:ds} is \textbf{equivariant} with respect to a group element $g\in G$ if the flow map satisfies
    \begin{equation}
        \mcal{K}_g \Flow^t (x) := \theta_{g^{-1}}\left( \Flow^t\left( \theta_{g}(x)\right) \right) = \Flow^t(x)
    \end{equation}
    for every $(t,x) \in D$.
\end{definition}
Differentiating at $t=0$ shows that equivariance of the dynamical system implies that $F$ is equivariant in the sense of Definition~\ref{def:equivariance_real_map_version} with $T(x,g) = \D \theta_{g}(x)$.
The converse is also true thanks to Corollary~9.14 in \cite{Lee2013introduction}, meaning that equivariance for the dynamical system is equivalent to equivariance of $F$.
Therefore, we can study equivariance of the dynamical system in \eqref{eqn:ds} by directly applying the tools developed in Section~\ref{sec:fundamental_operators} to the function $F$.
Thanks to Theorem~\ref{thm:symmetries_of_a_map}, identifying the connected subgroup of symmetries for the dynamical system is a simple matter of computing the nullspace of the linear map $\xi \mapsto \mcal{L}_{\xi} F$, that is
\begin{equation}\tag{\ref{eqn:Sym_G_F}}
    \sym_G(F) = \{ \xi \in \Lie(G) \ : \ \mcal{L}_{\xi} F = 0 \}.
\end{equation}
In this case $\hat{T}(\xi)_x = \frac{\partial}{\partial x} \hat{\theta}(\xi)_x$, meaning the Lie derivative is given by 
\begin{equation}
    \mcal{L}_{\xi} F(x) 
    = \frac{\partial F(x)}{\partial x} \hat{\theta}(\xi)_x - \frac{\partial \hat{\theta}(\xi)_x}{\partial x} F(x)
    = [\hat{\theta}(\xi), F]_x,
\end{equation}
where $[\hat{\theta}(\xi), F]_x$ is the Lie bracket of the vector fields $\hat{\theta}(\xi)$ and $F$ evaluated at $x$.
Symmetries can also be enforced as linear constraints on $F$ described by Theorem~\ref{thm:invariance_conditions_for_real_map}.
This was done by \cite{Ahmadi2020learning_short} for polynomial dynamical systems with discrete symmetries.
Later on in Section~\ref{subsec:symmetry_of_vf} we show that analogous results apply to dynamical systems on manifolds.

A conserved quantity for the system in \eqref{eqn:ds} is defined as follows:
\begin{definition}
    A scalar valued quantity $f: \R^n \to \R$ is said to be \textbf{conserved} when
    \begin{equation}
        \mcal{K}_t f(x) := f(\Flow^t(x)) = f(x) \qquad \forall (t,x) \in D.
    \end{equation}
\end{definition}
In this setting, the composition operators $\mcal{K}_t$ are often referred to as Koopman operators (see \cite{Koopman1931Hamiltonian, Mezic2005spectral, Mauroy2020koopman, Otto2021koopman, Brunton2022siamreview}).
It is easy to see that a smooth function $f$ is conserved if and only if
\begin{equation}
    \mcal{L}_F f := \left.\ddt\right\vert_{t=0} \mcal{K}_t f =  \frac{\partial f}{\partial x} F = 0.
\end{equation}
This relation is used by \cite{Kaiser2018discovering,kaiser2021data} to identify conserved quantities by computing the nullspace of $\mcal{L}_F$ restricted to finite-dimensional spaces of candidate functions.
When the flow is defined for all $t\in\R$, the operators $\mcal{K}_{t}$ and $\mcal{L}_F$ are the fundamental operators from Section~\ref{sec:fundamental_operators} for the right action of the Lie group $G=(\R,+)$ defined by $\theta(x,t) = \Flow^t(x)$ and $T(x,t) = I$.

\begin{remark}
    For Hamiltonian dynamical systems Noether's theorem establishes an equivalence between the symmetries of the Hamiltonian and conserved quantities of the system.
    We study Hamiltonian systems later in Section~\ref{subsec:Hamiltonian_conservation_laws}.
\end{remark}

\section{Promoting symmetry with convex penalties}
\label{sec:promoting_symmetry}
In this section we show how to design custom convex regularization functions to promote symmetries within a given candidate group during training of a machine learning model.
This allows us to train a model with as many symmetries as possible from among the candidates, while breaking candidate symmetries only when the data provides sufficient evidence.
We study both discrete and continuous groups of candidate symmetries.
We quantify the extent to which symmetries within the candidate group are broken using the fundamental operators described in Section~\ref{sec:fundamental_operators}.
For discrete groups we use the transformation operators $\{\mcal{K}_g\}_{g\in G}$ and for continuous groups we use the Lie derivatives $\{\mcal{L}_{\xi}\}_{\xi\in\Lie(G)}$.
In the continuous case we penalize a convex relaxation of the codimension of the subgroup of symmetries given by a nuclear norm (Schatten $1$-norm) of the operator $\xi \mapsto \mcal{L}_{\xi} F$ defined by \eqref{eqn:L_F_op_for_vector_space_case}; minimizing this codimension via the proxy nuclear norm will promote the largest nullspace possible, and hence the largest admissible symmetry group.
Once these regularization functions are developed abstractly in Sections~\ref{subsec:discrete_symmetries}~and~\ref{subsec:continuous_symmetries}, we explain how to apply the approach to basis function regression (Section~\ref{subsec:promoting_symmetry_in_regression}) 
and neural networks (Section~\ref{subsec:multilayer_perceptrons}).
We demonstrate and evaluate these methods on numerical examples in Section~\ref{sec:numerical_experiments}.



As in Section~\ref{sec:fundamental_operators}, the basic building blocks of the machine learning models we consider are continuously differentiable ($C^1$) functions $F: \mcal{V} \to \mcal{W}$ between finite-dimensional vector spaces.
While we consider this restricted setting here, our results readily generalize to sections of vector bundles, as we describe later in Section~\ref{sec:sections_of_vector_bundles}.
These functions could be layers of a multilayer perceptron, integral kernels to be applied to spatio-temporal fields, or simply linear combinations of user-specified basis functions in a regression task.
We consider parametric models where $F$ is constrained to lie in a given finite-dimensional subspace $\mcal{F} \subset C^{1}(\mcal{V};\mcal{W})$ of continuously differentiable functions.
Working within a finite-dimensional subspace of functions will allow us to discretize the fundamental operators in Section~\ref{sec:discretization}.

We consider the same setting as Section~\ref{sec:fundamental_operators}, i.e., candidate symmetries are described by a Lie group $G$ acting via \eqref{eqn:Theta_action_on_VW} on the domain and codomain of functions $F \in \mcal{F}$.
Equivariance in this setting is described by Definition~\ref{def:equivariance_real_map_version}.
When fitting the function $F$ to data, our regularization functions penalize the size of $G \setminus \Sym_G(F)$.
For reasons that will become clear, we use different penalties corresponding to different notions of ``size'' when $G$ is a discrete group versus when $G$ is continuous.
The main result describing the continuous symmetries of $F$ is Theorem~\ref{thm:symmetries_of_a_map}.

\subsection{Discrete symmetries}
\label{subsec:discrete_symmetries}
When the group $G$ has finitely many elements, one can measure the size of $G \setminus \Sym_G(F)$ simply by counting its elements:
\begin{equation}
    R_{G,0}(F) = \vert G \setminus \Sym_G(F) \vert.
\end{equation}
However, this penalty is impractical for optimization owing to its discrete values and nonconvexity.
Letting $\Vert \cdot \Vert$ be any norm on the space $\mcal{F}'' = \vspan \{ \mcal{K}_g F \ : \ g\in G, \ F \in \mcal{F} \}$ yields a convex relaxation of the above penalty given by
\begin{empheq}[box=\widefbox]{equation}
    R_{G,1}(F) = \sum_{g\in G} \Vert  \mcal{K}_g F - F \Vert.
    \label{eqn:discrete_penalty_for_real_maps}
\end{empheq}
This is a convex function on $\mcal{F}$ because $\mcal{K}_g$ is a linear operator and vector space norms are convex.
For example, if $\vect{c} = (c_1, \ldots, c_N)$ are the coefficients of $F$ in a basis for $\mcal{F}''$ and $\mat{K}_g$ is the matrix of $\mcal{K}_g$ in this basis, then the Euclidean norm can be used to define
\begin{equation}
    R_{G,1}(F) = \sum_{g\in G} \Vert \mat{K}_g \vect{c} - \vect{c} \Vert_2.
\end{equation}
This is directly analogous to the group sparsity penalty proposed in \cite{Yuan2006model}.

\subsection{Continuous symmetries}
\label{subsec:continuous_symmetries}
We now consider the case where $G$ is a Lie group with positive dimension and Lie algebra $\mfrak{g}$.
Here we use the dimension of $\Sym_G(F)$ to measure the symmetry of $F$, seeking to penalize the complementary dimension or \emph{codimension}, given by
\begin{equation}
    R_{\mathfrak{g},0}(F) = \codim(\Sym_G(F)) = \dim(G) - \dim(\Sym_G(F)).
\end{equation}
We take this approach in the continuous case because it is no longer possible to simply count the number of broken symmetries.
While it is possible in principle to replace the sum in \eqref{eqn:discrete_penalty_for_real_maps} by an integral of $\| \mcal{K}_g F - F \|$ over $g \in G$, the numerical quadrature required to approximate it becomes prohibitive for higher-dimensional candidate groups.
This difficulty is exacerbated by the fact that the integrand is not smooth.
The space $\mcal{F}''$ can also become infinite-dimensional when $G$ has positive dimension, making it challenging to compute the norm $\| \cdot \|$.

In contrast, it is much easier to measure the ``size'' of a continuous symmetry group using its dimension because this can be computed with linear algebra.
Specifically, the dimension of $\Sym_G(F)$ is equal to that of its Lie algebra.
Thanks to Theorem~\ref{thm:symmetries_of_a_map}, this is the nullspace of a linear operator $L_F : \Lie(G) \to C^{0}(\mcal{V};\mcal{W})$ defined by
\begin{equation}\tag{\ref{eqn:L_F_op_for_vector_space_case}}
    L_F: \xi \mapsto \mcal{L}_{\xi} F,
\end{equation}
where $\mcal{L}_{\xi}$ is the Lie derivative in \eqref{eqn:Lie_derivative_of_real_map}.
By the rank and nullity theorem, the codimension of $\Sym_G(F)$ is equal to the rank of this operator:
\begin{equation}
    R_{\mathfrak{g},0}(F) 
    = \codim(\sym_G(F))
    = \rank(L_F).
\end{equation}
Penalizing the rank of an operator is impractical for optimization owing to its discrete values and nonconvexity.
A commonly used convex relaxation of the rank is provided by the Schatten $1$-norm, also known as the \emph{nuclear norm}, given by
\begin{empheq}[box=\widefbox]{equation}
    R_{\mathfrak{g},*}(F) 
    = \Vert L_{F} \Vert_{*} 
    = \sum_{i=1}^{\dim(G)} \sigma_i(L_F).
    \label{eqn:nuclear_norm_penalty_for_real_maps}
\end{empheq}
Here $\sigma_i(L_F)$ denotes the $i$th singular value of $L_F$ with respect to inner products on $\Lie(G)$ and $\mcal{F}' = \vspan\{ \mcal{L}_{\xi} F \ : \ \xi\in\Lie(G), \ F\in\mcal{F} \}$.
This space is finite-dimensional, being spanned by $\{ \mcal{L}_{\xi_i} F_j \}_{i,j}$ where $\xi_i$ and $F_j$ are basis elements for $\Lie(G)$ and $\mcal{F}$. 
This enables computations with discrete inner products on $\mcal{F}'$, as we describe in Section~\ref{sec:discretization}.
For certain rank minimization problems, penalizing the nuclear norm is guaranteed to recover the true minimum rank solution \citep{Candes2009exact, Recht2010guaranteed, Gross2011recovering}.
Our numerical results in Section~\ref{subsec:symmetric_recovery} suggest that nuclear norm minimization can also recover symmetric functions using \eqref{eqn:nuclear_norm_penalty_for_real_maps}.

The proposed regularization function \eqref{eqn:nuclear_norm_penalty_for_real_maps} is convex on $\mcal{F}$ because $F \mapsto L_F$ is linear and the nuclear norm is convex.
For example, if $(c_1, \ldots, c_N)$ are the coefficients of $F$ in a basis $\{ F_1, \ldots, F_N \}$ for $\mcal{F}$ and $\mat{L}_{F_i}$ are the matrices of $L_{F_i}$ in orthonormal bases for $\Lie(G)$ and $\mcal{F}'$, then
\begin{equation}
    R_{\mathfrak{g},*}(F) = \Vert c_1 \mat{L}_{F_1} + \cdots + c_N \mat{L}_{F_N} \Vert_*.
\end{equation}
With $\{ \xi_1, \ldots, \xi_{\dim(G)} \}$ and $\{ u_{1}, \ldots, u_{N'} \}$ being the orthonormal bases for $\Lie(G)$ and $\mcal{F}'$, one can compute and store the rank-$3$ tensor
    $[\mat{L}_{F_i}]_{j,k} 
    = \left\langle u_j,\ \mcal{L}_{\xi_k} F_i \right\rangle_{\mcal{F}'}$.
Practical methods for constructing and computing with inner products on $\mcal{F}'$ will be discussed in Section~\ref{sec:discretization}.

\subsection{Promoting symmetry in basis function regression}
\label{subsec:promoting_symmetry_in_regression}
To demonstrate how the symmetry-promoting regularization functions proposed above can be used in practice, 
consider a regression problem for a function $F:\R^m\to\R^n$.
It is common to parameterize this problem by expressing $F(x) = W \fndict(x)$ in a dictionary $\fndict:\R^m \to \R^N$ consisting of user-defined smooth functions with a matrix of weights $W\in\R^{n\times N}$ to be fit during the training process.  
For example, the sparse identification of nonlinear dynamics (SINDy) algorithm~\citep{Brunton2016discovering} belongs to this type of learning, among other machine learning algorithms~\citep{Brunton2022book}. 
The fundamental operators (Section~\ref{sec:fundamental_operators}) for this class of functions are given by
\begin{align}
    \mcal{K}_g F(x) &= T(x,g)^{-1} W \fndict(\theta_{g}(x)), \\
    \mcal{L}_\xi F(x) &= W\frac{\partial \fndict(x)}{\partial x}\hat{\theta}(\xi)_{x} - \hat{T}(\xi)_x W \fndict(x).
\end{align}
These can be used directly in \eqref{eqn:discrete_penalty_for_real_maps} and \eqref{eqn:nuclear_norm_penalty_for_real_maps} to construct symmetry-promoting regularization functions $R_G(W)$ that are convex with respect to the weight matrix $W$.
Given a training dataset consisting of input-output pairs $\{ (x_j, y_j) \}_{j=1}^M$ we can seek a regularized least-squares fit by solving the convex optimization problem
\begin{equation}
    \minimize_{W\in\R^{n\times N}} \frac{1}{M}\sum_{j=1}^M \Vert y_j - W\fndict(x_j) \Vert^2 + \gamma R_G(W\mcal{D}).
    \label{eqn:basis_function_regression_prob}
\end{equation}
Here, $\gamma \geq 0$ is a parameter controlling the strength of the regularization that can be determined using cross-validation.
We demonstrate this approach in Section~\ref{subsec:multi_spring_mass_system} by using it to identify a symmetric linear dynamical system in a setting that requires extrapolation from limited data.

\begin{remark}
    The solutions $F = W\mcal{D}$ of \eqref{eqn:basis_function_regression_prob} do not depend on how the dictionary functions are normalized due to the fact that the function being minimized can be written entirely in terms of $F$ and the data $(x_j,y_j)$.
    This is in contrast to other types of regularized regression problems that penalize the weights $W$ directly, and therefore depend on how the functions in $\mcal{D}$ are normalized.
\end{remark}

\subsection{Promoting symmetry in neural networks}
\label{subsec:multilayer_perceptrons}
In this section we describe a convex regularizing penalty to promote $G$-equivariance in feed-forward neural networks
\begin{equation}
    F = F^{(L)} \circ \cdots \circ F^{(1)}
\end{equation}
composed of layers 
$F^{(l)}: \mcal{V}_{l-1} \to \mcal{V}_l$
acted upon by group representations 
$\Phi_l : G \to \GL(\mcal{V}_l)$.
We penalize the breaking of $G$-symmetries as a means to regularize the neural network and to learn which subgroup of symmetries are compatible with the data and which are not.
Since the composition is $g$-equivariant if every layer is $g$-equivariant, the main idea is to measure the symmetries shared by all of the layers.
Specifically, we aim to maximize the ``size'' of the subgroup
\begin{equation}
    \bigcap_{l=1}^L \Sym_G\big(F^{(l)}\big) 
    = \{ g \in G \ : \ \mcal{K}_g F^{(l)} = F^{(l)},\ l=1, \ldots, L \}
    \subset \Sym_G(F),
    \label{eqn:intersection_subgroup}
\end{equation}
where the notion of ``size'' we adopt depends on whether $G$ is discrete or continuous.
The same ideas can be applied to neural operators acting on fields with layers defined by integral operators as described in Section~\ref{subsec:NNs_acting_on_fields}.
In this case we consider symmetries shared by all of the integral kernels.

As in Section~\ref{subsec:discrete_symmetries}, when $G$ is a discrete group with finitely many elements, a convex relaxation of the cardinality of $G \setminus \bigcap_{l=1}^L \Sym_G(F^{(l)})$ is
\begin{empheq}[box=\widefbox]{equation}
    R_{G,1}\big(F^{(1)}, \ldots, F^{(l)}\big)
    = \sum_{g\in G} \sqrt{ \sum_{l=1}^L \big\Vert \mcal{K}_g F^{(l)} - F^{(l)} \big\Vert^2 }.
\end{empheq}
Again, this is analogous to the group-LASSO penalty developed in \citet{Yuan2006model}.

When $G$ is a Lie group with nonzero dimension, we follow the approach in Section~\ref{subsec:continuous_symmetries} using the following observation:
\begin{proposition}
    \label{prop:intersection_subalgebra}
    The subgroup in \eqref{eqn:intersection_subgroup} is closed and embedded in $G$; its Lie subalgebra is
    \begin{equation}
        \bigcap_{l=1}^L \sym_G\big(F^{(l)}\big) = \left\{ \xi \in \Lie(G) \ : \ \mcal{L}_{\xi} F^{(l)} = 0, \ l=1, \ldots, L \right\}.
    \end{equation}
    We provide a proof in Appendix~\ref{app:proofs_of_minor_results}.
\end{proposition}
The Lie subalgebra in the proposition is equal to the nullspace of the linear operator $L_{F^{(1)},\ldots,F^{(L)}} : \Lie(G) \to \bigoplus_{l=1}^L C^{\infty}(\mcal{V}_{l-1}; \mcal{V}_l)$ defined by 
\begin{equation}
    L_{F^{(1)},\ldots,F^{(L)}}: \xi \mapsto
    \big( \mcal{L}_\xi F^{(1)}, \ldots, \mcal{L}_\xi F^{(L)} \big).
    \label{eqn:operator_for_multiple_real_maps}
\end{equation}
By the rank and nullity theorem, minimizing the rank of this operator is equivalent to maximizing the dimension of the subgroup of symmetries shared by all of the layers in the network.
As in Section~\ref{subsec:continuous_symmetries}, a convex relaxation of the rank is provided by the nuclear norm
\begin{empheq}[box=\widefbox]{equation} \label{eqn:compositional_nuclear_norm_penalty}
    R_{\mathfrak{g},*}\big( F^{(1)}, \ldots, F^{(l)} \big)
    = \big\Vert L_{F^{(1)},\ldots,F^{(L)}} \big\Vert_* 
    = \left\Vert 
    \begin{bmatrix}
        \mat{L}_{F^{(1)}} \\
        \vdots \\
        \mat{L}_{F^{(L)}}
    \end{bmatrix}
    \right\Vert_*,
\end{empheq}
where $\mat{L}_{F^{(l)}}$ are the matrices of $L_{F^{(l)}}$ in orthonormal bases for $\Lie(G)$ and the associated spaces $\mcal{F}_l'$.

\begin{figure}
    \centering
    \begin{tikzonimage}[trim=0 0 0 0, clip=true, width=0.95\textwidth]{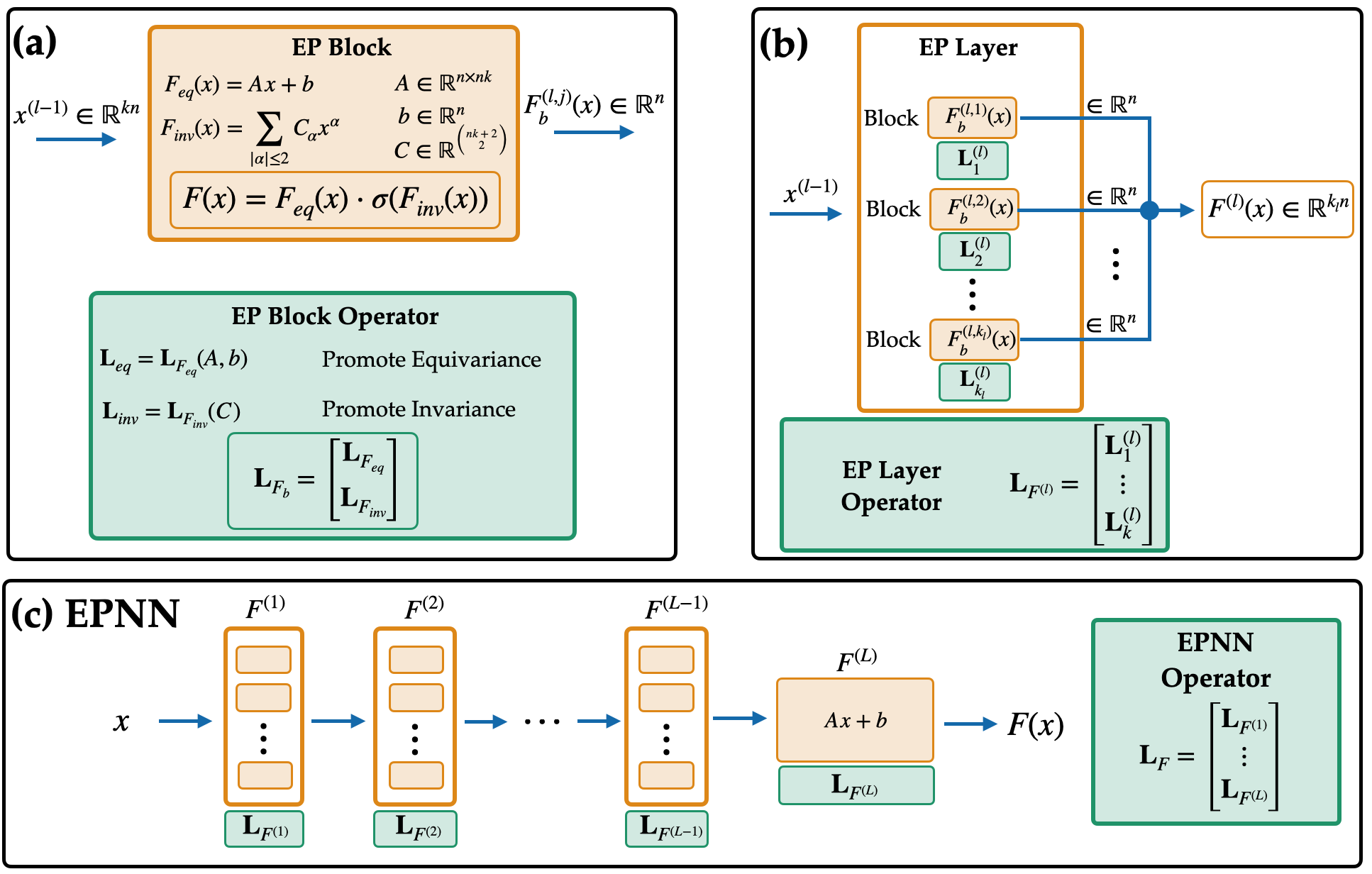}
    \end{tikzonimage}
    \caption{The Equivariance-Promoting Neural Network (EPNN) is a composition of Equivariance-Promoting (EP) layers built from \emph{EP blocks}. An EP block is parametrized by weights $A, b, C$ defining a vector-valued affine function $F_{\text{eq}}$ and a scalar-valued quadratic function $F_{\text{inv}}$. The symmetries shared by these components are captured by the nullspace of the \emph{EP block operator}. These are assembled into operators capturing shared symmetries of components within each layer, and then into an EPNN operator for the whole network.}
    \label{fig:symmetry_net}
\end{figure}

We introduce the Equivariance-Promoting Neural Network (EPNN) architecture shown in Figure~\ref{fig:symmetry_net} to promote Lie group symmetries via the convex regularization \eqref{eqn:compositional_nuclear_norm_penalty}.
This draws from the approach of \cite{Finzi2021practical}, where trainable $G$-equivariant linear layers are composed with $G$-equivariant nonlinearities.
The EPNN is built from Equivariance-Promoting (EP) blocks,
\begin{equation}\label{eqn:symmetry_block}
    F_{\text{b}}(x) = (\underbrace{A x + b \vphantom{\sum_{|\alpha| \leq 2}}}_{F_{\text{eq}}(x)}) \cdot 
    \sigma \Bigg( \underbrace{ 
    \sum_{|\alpha| \leq 2} C_{\alpha} x^{\alpha}
    }_{F_{\text{inv}}(x)} \Bigg),
\end{equation}
where $\alpha$ is a multi-index; $A$, $b$, and $C = (C_{\alpha})_{|\alpha|\leq 2}$ are the trainable weights of the block; and $\sigma$ is a nonlinear activation function such as the sigmoid $\sigma(z) = (1 + e^{-z})^{-1}$.
Using the sigmoid ensures the gradient of the EP block is bounded.
An EP block is $g$-equivariant with respect to linear group actions when $F_{\text{eq}}$ is $g$-equivariant and $F_{\text{inv}}$ is $g$-invariant.
The Lie subalgebra of symmetries shared by $F_{\text{eq}}$ and $F_{\text{inv}}$ forms the nullspace of the discretized EP block operator $\mat{L}_{F_{\text{b}}}$ (see Figure~\ref{fig:symmetry_net}).
The matrix entries of this operator are linear functions of the trainable parameters.

An \emph{EP layer} is built from EP blocks $F^{(l,j)}_{\text{b}}: \mcal{V}_{l-1} \to \tilde{\mcal{V}}_l$ by concatenation
\begin{equation}\label{eqn:symmetry_layer}
    F^{(l)}(x) = \left( F^{(l,1)}_{\text{b}}(x), \ \ldots,\ F^{(l, k_l)}_{\text{b}}(x) \right),
\end{equation}
producing an element in $\mcal{V}_l = \oplus_{j=1}^{k_l} \tilde{\mcal{V}}_l$.
An EPNN is then formed by composing $L-1$ of these layers, with a final affine EP output layer.
The $G$-actions on the hidden layers $l = 1, \ldots, L-1$ are defined by direct sum representations
\begin{equation}\label{eqn:hidden_layer_representation}
    \Phi_l(g) = 
    \begin{bmatrix}
        \tilde{\Phi}_l(g) & & \\
        & \ddots & \\
        & & \tilde{\Phi}_l(g)
    \end{bmatrix}, 
    \qquad \mbox{where} \qquad
    \tilde{\Phi}_l: G \to \GL(\tilde{\mcal{V}}_l).
\end{equation}
The key observation is that if the components $F^{(l,j)}_{\text{eq}}$ and $F^{(l,j)}_{\text{inv}}$ of each EP block are respectively $g$-equivariant and $g$-invariant, i.e.,
\begin{equation}
    \tilde{\Phi}_l(g) F^{(l,j)}_{\text{eq}}\big(\Phi_{l-1}(g)^{-1} x\big) = F^{(l,j)}_{\text{eq}}(x), \qquad \mbox{and} \qquad
    F^{(l,j)}_{\text{inv}}\big(\Phi_{l-1}(g)^{-1} x\big) = F^{(l,j)}_{\text{inv}}(x),
\end{equation}
and the affine EP output layer is $g$-equivariant, then the entire EPNN is $g$-equivariant.
The Lie subalgebra of symmetries shared by all components in an EPNN is given by the nullspace of the discretized EPNN operator $\mat{L}_F$ (see Figure~\ref{fig:symmetry_net}).
We obtain a convex symmetry-promoting penalty \eqref{eqn:compositional_nuclear_norm_penalty} by taking the nuclear norm of the EPNN operator, whose matrix entries are linear functions of the trainable weights.
Of course, it is also possible to enforce $G$-equivariance in the EPNN through a collection of linear constraints on the trainable weights using an approach analogous to \cite{Finzi2021practical} (see Section~\ref{sec:enforcing_symmetry}).
We evaluate EPNNs on a simple example in Section~\ref{subsec:double_spring_pendulum}.

\section{Discretizing the operators}
\label{sec:discretization}
This section describes how to construct matrices for the operators $\mcal{L}_\xi$ and $L_F$ for continuously differentiable functions $F$ in a user-specified finite-dimensional subspace $\mcal{F} \subset C^{1}(\mcal{V};\mcal{W})$.
By choosing bases for the finite-dimensional vector spaces $\mcal{V}$ and $\mcal{W}$, it suffices without loss of generality to consider the case in which $\mcal{V} = \R^m$ and $\mcal{W} = \R^n$.
We assume that $\Lie(G)$ and $\mcal{F}$ are endowed with inner products and that $\{ \xi_1, \ldots \xi_{\dim(G)} \}$ and $\{ F_1, \ldots F_{\dim(\mcal{F})} \}$ are orthonormal bases for these spaces, respectively.
The key task is to endow the finite-dimensional subspace
\begin{equation}
    \mcal{F}' = \vspan \left\{ \mcal{L}_\xi F \ : \ \xi\in\Lie(G),\ F\in\mcal{F} \right\} \subset C^{0}(\R^m;\R^n)
\end{equation}
with a convenient inner product.
Once this is done, an orthonormal basis $\{ u_1, \ldots, u_N \}$ for $\mcal{F}'$ can be constructed by applying a Gram-Schmidt process to the functions $\mcal{L}_{\xi_i} F_j$, which span $\mcal{F}'$.
Matrices for $\mcal{L}_\xi$ and $L_F$ are then easily obtained by computing
\begin{equation}
    \big[ \mat{\mcal{L}}_\xi \big]_{i,j} = \big\langle u_i, \ \mcal{L}_\xi F_j \big\rangle_{\mcal{F}'}, \qquad
    \big[ \mat{L}_F \big]_{i,k} = \big\langle u_i, \ \mcal{L}_{\xi_k} F \big\rangle_{\mcal{F}'}.
\end{equation}
The issue at hand is to choose the inner product on $\mcal{F}'$ in a way that makes computing these matrices easy.
A natural choice is to equip $\mcal{F}'$ with an $L^2(\R^m, \mu; \R^n)$ inner product where $\mu$ is a positive measure on $\R^m$ (such as a Guassian distribution) for which the $L^2$ norms of function in $\mcal{F}'$ are finite.
The problem is that it is usually challenging or inconvenient to compute the required integrals
\begin{equation}
    \big\langle \mcal{L}_{\xi_i} F_j,\ \mcal{L}_{\xi_k} F_l \big\rangle_{L^2(\mu)} = \int_{\R^m} \big(\mcal{L}_{\xi_i} F_j\big)(x)^T \big(\mcal{L}_{\xi_k} F_l\big)(x) \td\mu(x)
    \label{eqn:L2_mu_inner_prod_of_Lie_derivs}
\end{equation}
analytically.
In this section we discuss inner products that are easy to compute in practice.

\subsection{Numerical quadrature and Monte-Carlo}
When \eqref{eqn:L2_mu_inner_prod_of_Lie_derivs} cannot be computed analytically, 
numerical quadrature or Monte-Carlo can be used.
In both cases the integral is approximated by a weighted sum, yielding a semi-inner product
\begin{equation}
    \langle f, \ g \rangle_{L^2(\mu_M)} = \frac{1}{M}\sum_{i=1}^M w_i f(x_i)^T g(x_i)
    \label{eqn:empirical_L2_mu_inner_prod}
\end{equation}
that converges to $\langle f, \ g \rangle_{L^2(\mu)}$ as $M\to\infty$.
The following proposition means that we do not have to pass to the limit $M\to\infty$ in order to obtain a valid inner product defined by \eqref{eqn:empirical_L2_mu_inner_prod} on $\mcal{F}'$.
\begin{proposition}
    \label{prop:Monte_Carlo_eventually_gives_an_inner_product}
    Suppose that $\mcal{F}'$ is finite-dimensional and $\langle f, \ g \rangle_{L^2(\mu_M)} \to \langle f, \ g \rangle_{L^2(\mu)}$ as $M\to\infty$ for every $f,g\in\mcal{F}'$.
    Then there is an $M_0$ such that 
    \eqref{eqn:empirical_L2_mu_inner_prod} is an inner product on $\mcal{F}'$ for every $M \geq M_0$.
    We give a proof in Appendix~\ref{app:proofs_of_minor_results}.
\end{proposition}

In Monte-Carlo approximation, the samples $x_i$ are drawn independently from a distribution $\nu$ with the assumption that both $\mu$ and $\nu$ are $\sigma$-finite and $\mu$ is absolutely continuous with respect to $\nu$.
The weights are given by the Radon-Nikodym derivative $w_i = \frac{\td \mu}{\td \nu}(x_i)$.
Then for every $f, g\in L^2(\mu)$ the approximate integral converges $\langle f, \ g \rangle_{L^2(\mu_M)} \to \langle f, \ g \rangle_{L^2(\mu)}$ as $M\to\infty$ almost surely thanks to the strong law of large numbers (see Theorem~7.7 in \cite{Koralov2012theory}).
By the proposition, there is almost surely a finite $M_0$ such that \eqref{eqn:empirical_L2_mu_inner_prod} is an inner product on $\mcal{F}'$ for every $M \geq M_0$.

\subsection{Subspaces of polynomials}
Here we consider the special case when $\mcal{F}$ is a finite-dimensional subspace consisting of polynomial functions $\R^m \to \R^n$ and the action \eqref{eqn:Theta_action_on_VW} is defined by representations as in \eqref{eqn:VW_rep_actions}.
Examining \eqref{eqn:Lie_derivative_with_reps}, it is evident that $\mcal{L}_\xi F$ is also a polynomial function $\R^m \to \R^n$ with degree not greater than that of $F\in\mcal{F}$.
Thus, $\mcal{F}'$ is also a space of polynomial functions with degree not exceeding the maximum degree in $\mcal{F}$.
Since a polynomial that vanishes on an open set must be identically zero, we can take the integrals defining the inner product in \eqref{eqn:empirical_L2_mu_inner_prod} over a cube, such as $[0,1]^m\subset \R^m$.
This is convenient because polynomial integrals over the cube can be calculated analytically.

We can also use the sample-based inner product in \eqref{eqn:empirical_L2_mu_inner_prod} with randomly chosen points $x_i$ and positive weights $w_i$.
The following proposition tells us exactly how many sample points we need.
\begin{proposition}
    \label{prop:generic_polynomial_map_sampling}
    Let $\mcal{F}'$ be a space of real polynomial functions $\R^m \to \R^n$ and let $\pi_i:\R^n \to \R$ be the $i$th coordinate projection $\pi(c_1, \ldots, c_n) = c_i$.
    Let 
    \begin{equation}
        M \geq M_0 := \max_{1\leq i \leq n} \dim(\pi_i (\mcal{F}'))
    \end{equation}
    and let $w_1, \ldots, w_M > 0$ be positive weights.
    Then for almost every set of points $(x_1, \ldots, x_M) \in (\R^m)^M$ with respect to Lebesgue measure, \eqref{eqn:empirical_L2_mu_inner_prod} is an inner product on $\mcal{F}'$.  
    We give a proof in Appendix~\ref{app:generic_polynomial_map_sampling}.
\end{proposition}
This means that we can draw $M \geq M_0$ sample points independently from any absolutely continuous measure (such as a Gaussian distribution or the uniform distribution on a cube), and with probability one,  \eqref{eqn:empirical_L2_mu_inner_prod} will be an inner product on $\mcal{F}'$.
When $\mcal{F}$ consists of polynomials with degree at most $d$, then it is sufficient to take
\begin{equation}
    M \geq 
    \left\vert\left\{ (p_0, \ldots, p_m) \in \mathbb{N}_0^{m+1} \ : \ p_0 + \cdots + p_m = d \right\}\right\vert
    =  \binom{d + m}{m}.
\end{equation}

\section{Numerical experiments using new symmetry-promoting penalties}
\label{sec:numerical_experiments}
We perform several numerical experiments to evaluate our proposed approach for promoting Lie group symmetry using nuclear-norm regularization.
All of our code was written in Python and is available at \url{https://github.com/nzolman/unified-symmetry-ml}.

\subsection{Sample complexity to recover symmetric polynomials}
\label{subsec:symmetric_recovery}
Can promoting symmetry help us learn an unknown symmetric function using less data? 
To begin answering this question, we perform numerical experiments studying the amount of sampled data needed to recover structured polynomial functions on $\R^n$ of the form
\begin{align}
    F_{\text{rad}}(x) &= \varphi_{\text{rad}}\big( \| x - c_1 \|_2^2,\ \ldots,\  \| x - c_r \|_2^2 \big) \quad \mbox{and} \label{eqn:multiradial_fun} \\
    F_{\text{lin}}(x) &= \varphi_{\text{lin}}\big( u_1^T x,\ \ldots,\ u_r^T x\big) \label{eqn:low_dim_fun}.
\end{align}
These possess various rotation and translation invariances when $r < n$, as characterized in detail below by Proposition~\ref{prop:symmetries_of_basic_functions} and its corollaries.

We aim to recover the unknown function $F_*$ within the space $\mcal{P}_d(\R^n)$ of polynomial functions on $\R^n$ with degrees up to $d = \deg F_*$ based on the values $y_j = F_*(x_j)$ at sample points $x_1, \ldots, x_N$.
Our approximation $\hat{F}$ of $F_*$ is computed by solving the convex optimization problem
\begin{equation}\label{eqn:symmetric_recovery_relaxed_optimization_problem}
    \minimize_{F \in \mcal{P}_d(\R^n)} R_{\mathfrak{g},*}(F) = \| L_{F} \|_* 
    \qquad \text{s.t.} \qquad
    F(x_j) = y_j, \quad j=1, \ldots, N,
\end{equation}
where $G$ is a candidate Lie group of symmetries acting on $\R^n$.
This was done using the CVXPY software package developed by \cite{Diamond2016cvxpy, Agrawal2018rewriting}.
The nuclear norm in \eqref{eqn:symmetric_recovery_relaxed_optimization_problem} was defined with respect to inner products on the corresponding Lie algebras given by $\langle \xi, \ \eta \rangle_{\Lie(G)} = \Tr(\xi^T \eta)$.
As described in Section~\ref{sec:discretization}, the inner product on the space $\mcal{F}'$ containing the ranges of every $L_F$ was defined by \eqref{eqn:empirical_L2_mu_inner_prod} with unit weights $w_i = 1$ and $M = \dim \mcal{P}_d(\R^n)$ points drawn uniformly from the cube $[-1,1]^n$.
When using the group of translations $G = (\R^n, +)$, the Lie derivative always lowers the degree by $1$, so we only use $M = \dim \mcal{P}_{d-1}(\R^n)$ points in this case.
Note that the discretization points are not the same as the sample points $x_j$ in \eqref{eqn:symmetric_recovery_relaxed_optimization_problem}.

The dimensions of the resulting discretized operators are summarized in Table~\ref{tab:L_tensor_shapes_for_polynomial_recovery}.
One of the largest experiments we performed is shown in Figure~\ref{fig:linear_features_recovery_SEn} with $\deg(\phi_{\text{lin}}) = 2$, $\dim(x) = 30$, and $G = \SE(30)$. 
We computed and stored the entries of the $496 \times 496 \times 465$ tensor $[\mat{L}_{F_i}]_{j,k}$, which were accessed by CVXPY. 
This problem took roughly $40$ minutes to solve using an Intel Xeon Gold 6230 CPU.
On the other hand, the higher-dimensional problem with $\dim(x) = 50$ shown in Figure~\ref{fig:linear_features_recovery_Tn} was solved in less than $1$ minute using the translation group $G = (\R^{50}, +)$.
Here, the tensor had dimensions $1326\times 51 \times 50$, meaning the problem had more optimization variables ($1326$), but involved nuclear norms of much smaller $51 \times 50$ matrices.
\begin{table}[h]
    \centering
    \begin{tabular}{c|c|c|c}
         Group & $\mbox{$i$-dim.} = \dim \mcal{F} = \dim \mcal{P}_d(\R^n)$ & $\mbox{$j$-dim.} = \dim \mcal{F}' = M$ & $\mbox{$k$-dim.} = \dim G$ \\
         \hline
         $\SE(n)$ & $\binom{n+d}{d}$ & $\binom{n+d}{d}$ & $\frac{n(n+1)}{2}$ \\
         $(\R^n, +)$ & $\binom{n+d}{d}$ & $\binom{n+d-1}{d-1}$ & $n$
    \end{tabular}
    \caption{Dimensions of discretized rank-3 tensors $[\mat{L}_{F_i}]_{j,k}$ for symmetric polynomial recovery.
    }
    \label{tab:L_tensor_shapes_for_polynomial_recovery}
\end{table}

To study the sample complexity for \eqref{eqn:symmetric_recovery_relaxed_optimization_problem} to recover functions in the form of $F_{\text{rad}}$ and $F_{\text{lin}}$, we perform multiple experiments using random realizations of these functions sampled at random points $x_j$.
In each experiment, the vectors $c_i$ were drawn uniformly from the cube $[-1,1]^n$ and the vectors $u_i$ were formed from the columns of a random $n\times r$ orthonormal matrix (specifically, the left singular vectors of an $n\times r$ matrix with standard Gaussian entries).
The coefficients of $\varphi_{\text{rad}}$ and $\varphi_{\text{lin}}$ in a basis of monomials up to a specified degree were sampled uniformly from the interval $[0,1]$.
This yielded random polynomial functions $F_{\text{rad}}$ and $F_{\text{lin}}$ with degrees $\deg F_{\text{rad}} = 2\deg \varphi_{\text{rad}}$ and $\deg F_{\text{lin}} = \deg \varphi_{\text{lin}}$.

The sample points $x_j$ for each experiment were drawn uniformly from the cube $[-1,1]^n$.
A total of $\dim \mcal{P}_d(\R^n)$ with $d = \deg F_*$ sample points were drawn, which is sufficient to recover the function almost surely regardless of regularization.
For each experiment we determine the smallest $N_*$ so that recovery is achieved by \eqref{eqn:symmetric_recovery_relaxed_optimization_problem} with $\hat{F} = F_*$ using the sample points $x_1, \ldots, x_N$ for every $N \geq N_*$.
To be precise, successful recovery is declared when all coefficients describing $\hat{F}$ and $F_*$ in the monomial basis for $\mcal{P}_d(\R^n)$ agree to a tolerance of $5\times 10^{-3}$ times the magnitude of the largest coefficient of $F_*$.
The range of values for $N_*$ across $10$ such random experiments provides an estimate of the sample complexity.
In Figures~\ref{fig:radial_features_recovery_SEn},~\ref{fig:linear_features_recovery_SEn},~and~\ref{fig:linear_features_recovery_Tn} we plot the range of values for $N_*$ as shaded regions with the average values displayed as a solid lines.

In Figure~\ref{fig:radial_features_recovery_SEn}, we use the special Euclidean group $G=SE(n)$ as a candidate group to recover functions of the form $F_{\text{rad}}$ with the degree of $\varphi_{\text{rad}}$ specified.
The number of radial features $r \leq \deg \varphi_{\text{rad}}$ is selected in accordance with Corollary~\ref{cor:symmetries_of_polynomial_Fa} in order to ensure that $\sym_G(F_{\text{rad}}) = \mfrak{g}_{\text{rad}}$ has the known form and dimension stated in Proposition~\ref{prop:symmetries_of_basic_functions}.
The symmetry-promoting regularization significantly reduces the number of samples needed to recover $F_{\text{rad}}$ compared to the number of samples needed to solve the linear system specifying this function within the space of polynomials with the same or lesser degree.
As the number of radial features $r$ increases, so does the sample complexity to recover $F_{\text{rad}}$.
This is likely due to the decreased dimension of $\sym_G(F_{\text{rad}})$.

\begin{figure}
    \centering
    \begin{tikzonimage}[trim=0 0 0 0, clip=true, width=0.75\textwidth]{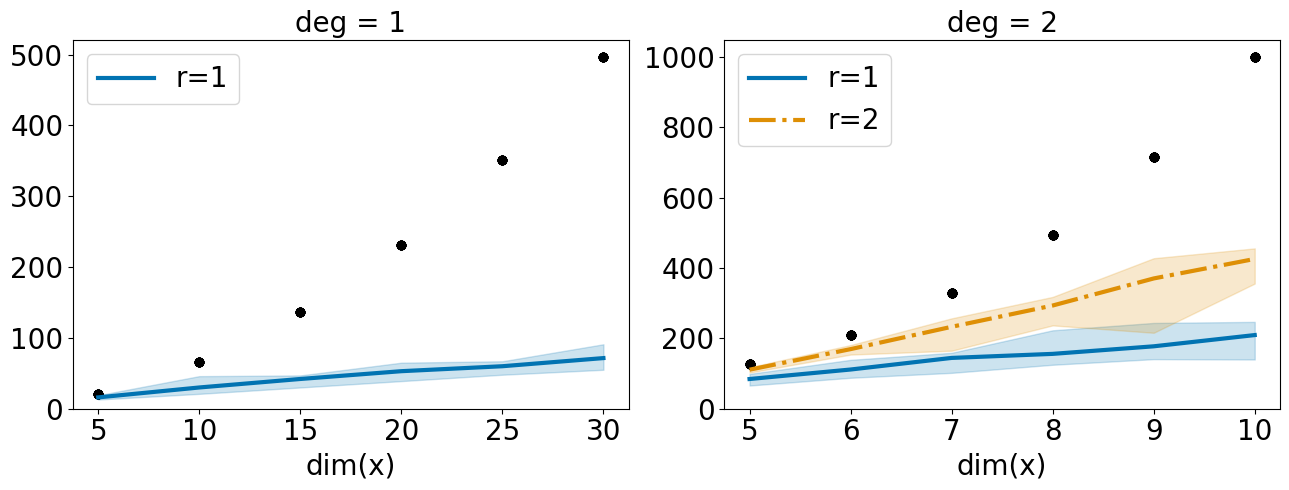}
        \node[rotate=90] at (-0.015, 0.5) {samples, $N_*$};
    \end{tikzonimage}
    \vspace{-.1in}
    \caption{Sample complexity to recover polynomial functions $F_{\text{rad}}$ of $r$ radial features, i.e., \eqref{eqn:multiradial_fun} with polynomial $\varphi_{\text{rad}}$ of the specified degree, by solving \eqref{eqn:symmetric_recovery_relaxed_optimization_problem} using the special Euclidean group $G = SE(n)$. Black dots indicate the number of dictionary functions, $\dim \mcal{P}_d(\R^n)$, hence the number of samples needed to recover $F_{\text{rad}}$ without regularization.
      }
    \label{fig:radial_features_recovery_SEn}
\end{figure}

\begin{figure}
    \centering
    \begin{tikzonimage}[trim=0 0 0 0, clip=true, width=0.95\textwidth]{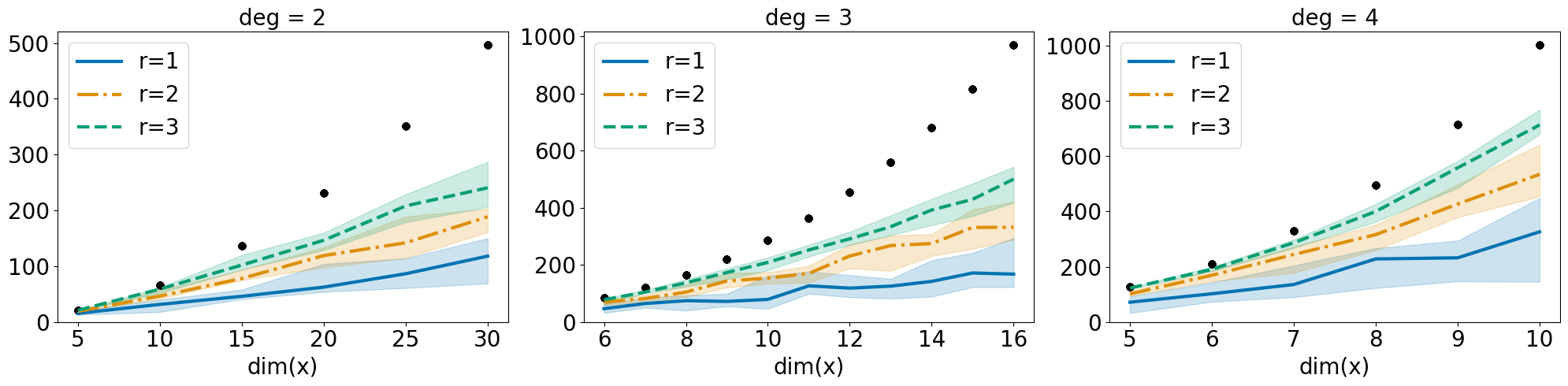}
        \node[rotate=90] at (-0.015, 0.5) {samples, $N_*$};
    \end{tikzonimage}
    \vspace{-.1in}
    \caption{Sample complexity to recover polynomial functions $F_{\text{lin}}$ of $r$ linear features, i.e., \eqref{eqn:low_dim_fun} with polynomial $\varphi_{\text{lin}}$ of the specified degree, by solving \eqref{eqn:symmetric_recovery_relaxed_optimization_problem} using the special Euclidean group $G = SE(n)$. Black dots indicate the number of dictionary functions, $\dim \mcal{P}_d(\R^n)$, hence the number of samples needed to recover $F_{\text{lin}}$ without regularization.
      }
    \label{fig:linear_features_recovery_SEn}
\end{figure}

\begin{figure}
    \centering
    \begin{tikzonimage}[trim=0 0 0 0, clip=true, width=0.95\textwidth]{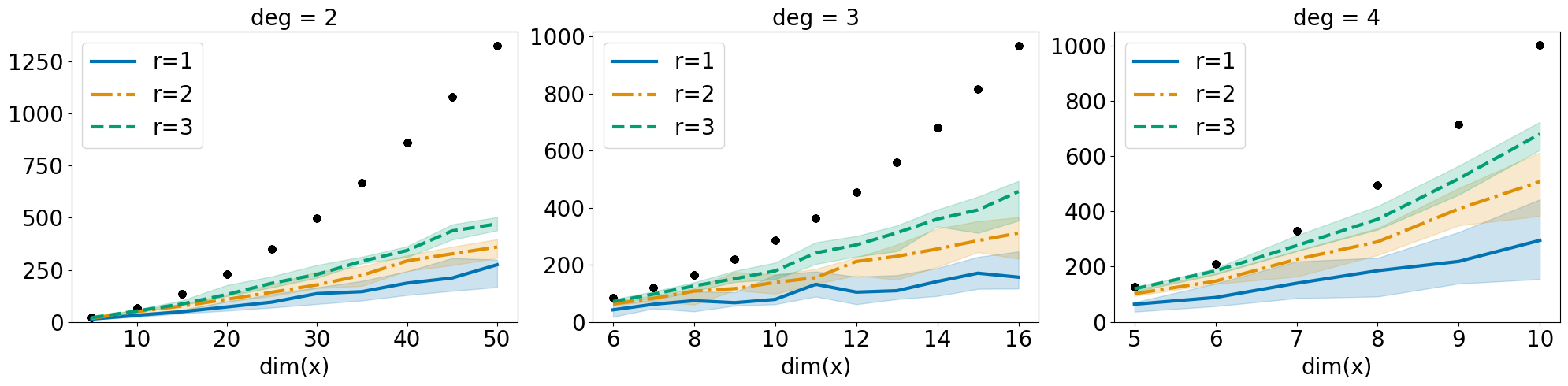}
        \node[rotate=90] at (-0.015, 0.5) {samples, $N_*$};
    \end{tikzonimage}
    \vspace{-.1in}
    \caption{Analogue of Fig.~\ref{fig:linear_features_recovery_SEn} using the group of translations $G = (\R^n, +)$.
      }
    \label{fig:linear_features_recovery_Tn}
\end{figure}

In Figures~\ref{fig:linear_features_recovery_SEn}~and~\ref{fig:linear_features_recovery_Tn}, we use $G = SE(n)$ and the group of translations $G = (\R^n, +)$ as candidate symmetry groups to recover function of the form $F_{\text{lin}}$ with the degree of $\varphi_{\text{lin}}$ specified.
Obviously, $F_{\text{lin}}$ has an $(n-d)$-dimensional subgroup of translation symmetries orthogonal to $\vspan\{ u_1, \ldots, u_r \}$.
By Corollary~\ref{cor:symmetries_of_polynomial_Fb}, choosing $\deg \varphi_{\text{lin}} \geq 2$ is sufficient to ensure that $\sym_{\SE(n)}(F_{\text{lin}}) = \mfrak{g}_{\text{lin}}$ has the known form and dimension stated in Proposition~\ref{prop:symmetries_of_basic_functions}.
The results in Figures~\ref{fig:linear_features_recovery_SEn}~and~\ref{fig:linear_features_recovery_Tn} show that the symmetry-promoting regularization reduces the sample complexity to recover $F_{\text{lin}}$.
Moreover, fewer samples are needed when $F_{\text{lin}}$ depends on fewer linear features, as might be expected because the dimension of $\sym_G(F_{\text{rad}})$ increases as $r$ decreases.

\begin{proposition}
    \label{prop:symmetries_of_basic_functions}
    Let $r \leq n$ and suppose that
    $\{ c_k - c_1 \}_{k=2}^r$ and $\{ u_k \}_{k=1}^r$ are sets of linearly-independent vectors in $\R^n$.
    Then, $\sym_{\SE(n)}(F_{\text{rad}})$ contains the $\frac{1}{2}(n - r)(n - r + 1)$-dimensional subalgebra
    \begin{equation}
        \mathfrak{g}_{\text{rad}} = \left\{ 
            \begin{bmatrix}
                S & v \\
                0 & 0
            \end{bmatrix} \ : \ 
            S^T=-S \ \mbox{and} \ S c_1 = \cdots = S c_r = - v
        \right\}
    \end{equation}
    and $\sym_{\SE(n)}(F_{\text{lin}})$ contains the $\frac{1}{2}(n - r)(n - r + 1)$-dimensional subalgebra
    \begin{equation}
        \mathfrak{g}_{\text{lin}} = \left\{ 
            \begin{bmatrix}
                S & v \\
                0 & 0
            \end{bmatrix} \ : \ 
            S^T=-S, \ S u_1 = \cdots = S u_r = 0, \ \mbox{and} \ u_1^T v = \cdots = u_r^T v = 0
        \right\}.
    \end{equation}
    Either every polynomial $\varphi_{\text{rad}}$ with degree $\leq d$ gives $\sym_{\SE(n)}(F_{\text{rad}}) \neq \mathfrak{g}_{\text{rad}}$ or
    the set of polynomials $\varphi_{\text{rad}}$ with degree $\leq d$ satisfying $\sym_{\SE(n)}(F_{\text{rad}}) \neq \mathfrak{g}_{\text{rad}}$ is a set of measure zero.
    Likewise, for $\varphi_{\text{lin}}$, $F_{\text{lin}}$, and $\mathfrak{g}_{\text{lin}}$.
    See Appendix~\ref{app:proofs_of_minor_results} for a proof.
\end{proposition}

\begin{corollary}
    \label{cor:symmetries_of_polynomial_Fa}
    With the same hypotheses as Proposition~\ref{prop:symmetries_of_basic_functions}, let $d \geq r$.
    The set of polynomials $\varphi_{\text{rad}}$ with degree $\leq d$ satisfying $\sym_{\SE(n)}(F_{\text{rad}}) \neq \mathfrak{g}_{\text{rad}}$ is a set of measure zero.
    A proof is given in Appendix~\ref{app:proofs_of_minor_results}.
\end{corollary}

\begin{corollary}
    \label{cor:symmetries_of_polynomial_Fb}
    With the same hypotheses as Proposition~\ref{prop:symmetries_of_basic_functions}, let $d \geq 2$.
    The set of polynomials $\varphi_{\text{lin}}$ with degree $\leq d$ satisfying $\sym_{\SE(n)}(F_{\text{lin}}) \neq \mathfrak{g}_{\text{lin}}$ is a set of measure zero.
    A proof is given in Appendix~\ref{app:proofs_of_minor_results}.
\end{corollary}

\subsection{Extrapolatory recovery of a multi-spring-mass system}
\label{subsec:multi_spring_mass_system}
We demonstrate how symmetry-promoting nuclear-norm regularization can identify equivariant dynamics that correctly extrapolate from a restricted dataset.
We consider a system of $n_p = 5$ particles with identical masses $m_i=1$, whose positions and momenta are vectors $\vect{q}^i = (q^i_x, q^i_y, q^i_z)$ and $\vect{p}^i = (p^i_x, p^i_y, p^i_z)$ in $\R^3$.
However, the training data are confined to the $x,z$-plane, as we explain below.
The particles fall under the influence of gravitational acceleration $\vect{g} = (0, 0, -1)$ and are coupled together by springs with random stiffness $K_{i,j} = \sum_{k=1}^{n_p} (1-\delta_{i,j}) B_{i,k} B_{j,k}$, where $B_{i,j}$ are drawn uniformly from the interval $[0.1,1]$.
The result is a linear dynamical system
\begin{equation}\label{eqn:spring_mass_system}
    \ddt
    \begin{bmatrix}
        \vect{q}^i \\
        \vect{p}^i
    \end{bmatrix}
    = 
    \sum_{j=1}^{n_p}
    \underbrace{
    \begin{bmatrix}
        \mat{0} & \delta_{i,j} m_i^{-1} \mat{I} \\
        \left(K_{i,j}-\delta_{i,j}\sum_{k=1}^{n_p} K_{i,k}\right) \mat{I} & \mat{0}
    \end{bmatrix}
    }_{\mat{A}_{i,j}}
    \begin{bmatrix}
        \vect{q}^j \\
        \vect{p}^j
    \end{bmatrix}
    + 
    \begin{bmatrix}
        \vect{0} \\
        \vect{g}
    \end{bmatrix},
\end{equation}
expressible in the form \eqref{eqn:ds} with $\vect{\dot{x}} = \mat{A}\vect{x}$ using state variable $\vect{x} = (\vect{q}^1, \vect{p}^1, \ldots, \vect{q}^{n_p}, \vect{p}^{n_p}, 1)$.
The entries of the matrix $\mat{A}$ are shown in the leftmost plot of Figure~\ref{fig:multi_spring_mass}(c).

\begin{figure}
    \centering
    \begin{tikzonimage}[trim=0 0 0 0, clip=true, width=0.95\textwidth]{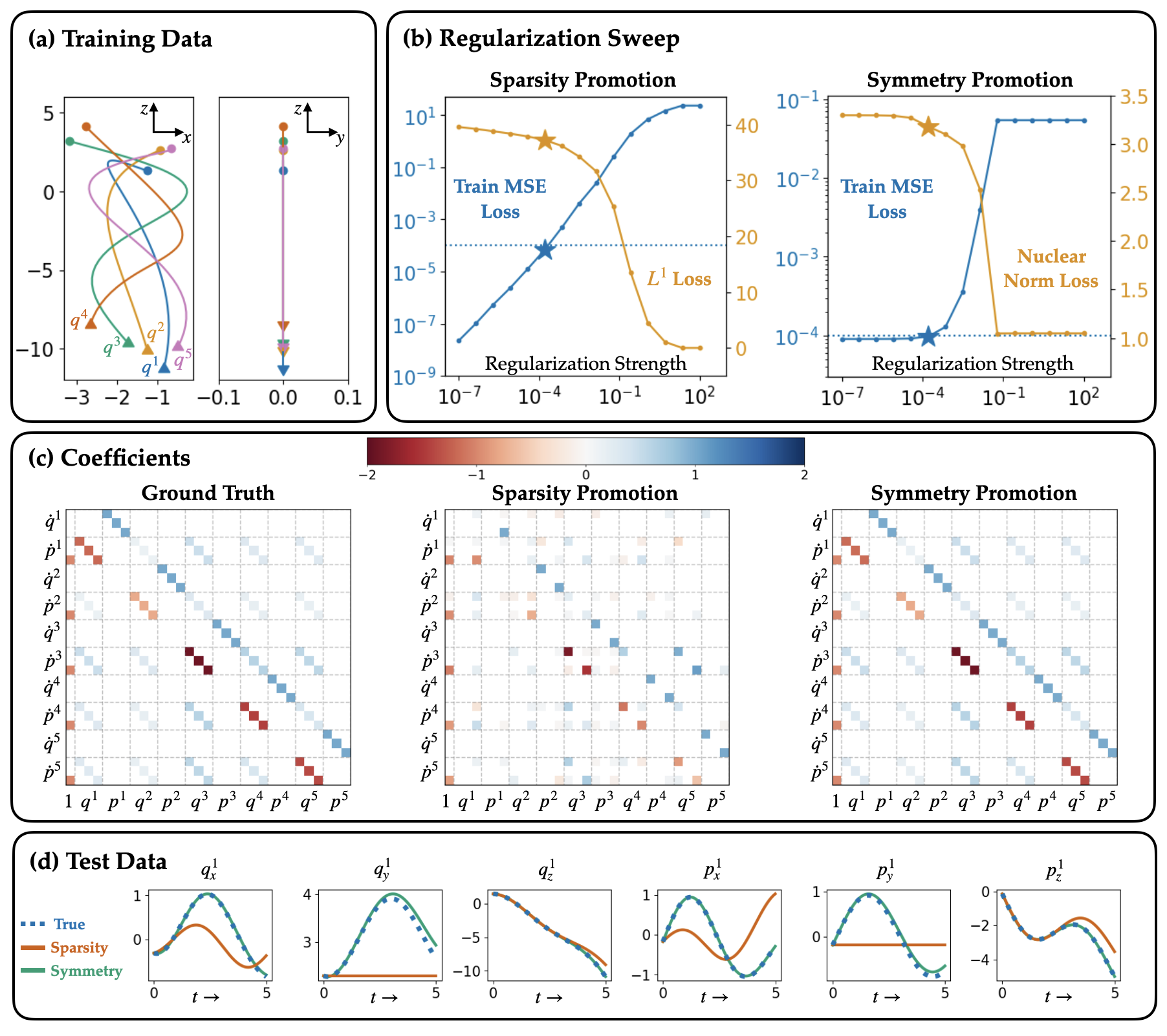}
    \end{tikzonimage}
    \caption{The multi-spring-mass system is learned from training data confined to the $x,z$-plane using $\SE(3)$-symmetry-promoting regularized regression. Results are compared to Sparse Identification of Nonlinear Dynamics (SINDy) using sparsity-promoting $\ell^1$ regularization.}
    \label{fig:multi_spring_mass}
\end{figure}

When there is no gravity, the system is $\SE(3)$-equivariant under the action \eqref{eqn:SE_action} applied simultaneously to all particles.
Gravity breaks the rotational symmetries about all but the $z$-axis, leaving the system equivariant to translations and rotations of the entire system of particles about the $z$-axis, corresponding to the subgroup of $\SE(3)$ given by
\begin{equation}\label{eqn:spring_mass_symmetries}
    \Sym_{\SE(3)}(\mat{A}) = 
    \left\{ 
    \begin{bmatrix}
        \cos(\theta) & -\sin(\theta) & 0 & b_x \\
        \sin(\theta) & \cos(\theta) & 0 & b_y \\
        0 & 0 & 1 & b_z \\
        0 & 0 & 0 & 1
    \end{bmatrix}
    \ : \ \theta, b_x, b_y, b_z \in\R \right\}.
\end{equation}

We use nuclear-norm-based symmetry-promotion \eqref{eqn:nuclear_norm_penalty_for_real_maps} within the group $G = \SE(3)$ to learn the dynamics \eqref{eqn:spring_mass_system} and its symmetries \eqref{eqn:spring_mass_symmetries} from trajectories that only evolve in the $x,z$-plane.
The training data consists of $M=100$ state-time-derivative pairs $(\vect{x}_j, \vect{\dot{x}}_j)$ collected along two trajectories, each with $50$ samples spaced uniformly at time intervals $\Delta t = 0.1$.
One of the trajectories is shown in Figure~\ref{fig:multi_spring_mass}(a).
These were generated by choosing normally-distributed initial conditions
\begin{gather}
    \vect{q}^i \sim \mcal{N} \left(\vect{q}_0,\ \diag\begin{bmatrix} 1 & 0 & 1 \end{bmatrix} \right), \qquad
    \vect{p}^i \sim \mcal{N} \left(\vect{p}_0,\ \diag\begin{bmatrix} 0.01 & 0 & 0.01 \end{bmatrix} \right),
    \qquad i=1, \ldots, n_p, \\
    \vect{q}_0 \sim \mcal{N}(\vect{0},\ \diag\begin{bmatrix} 1 & 0 & 1 \end{bmatrix}), \qquad
    \vect{p}_0 \sim \mcal{N}(\vect{0},\ \diag\begin{bmatrix} 0.01 & 0 & 0.01 \end{bmatrix}),
\end{gather}
with the stated means and covariance so that every particle has zero $y$-position and zero $y$-momentum.

We employ the basis function regression framework described in Section~\ref{subsec:promoting_symmetry_in_regression} with loss function \eqref{eqn:basis_function_regression_prob} and linear dictionary $\mcal{D}$ to fit the system matrix $\mat{A}$ to the training data.
This approach is compared to Sparse Identification of Nonlinear Dynamics (SINDy) employing sparsity-promoting $\ell^1$ regularization (see \cite{Brunton2016discovering, Tibshirani1996regression}) on the entries of $\mat{A}$.
The components of the training loss are plotted against the regularization strength $\gamma$ in Figure~\ref{fig:double_spring_pendulum}(b).
The chosen regularization strengths (indicated by stars) were the maximum values for which the training mean square error (MSE) remained less than $10^{-4}$.

The system matrices learned using sparsity-promoting $\ell^1$ regularization and $\SE(3)$-symmetry-promoting regularization are shown in Figure~\ref{fig:double_spring_pendulum}(c).
We observe that symmetry-promoting regularization recovers the correct matrix structure, while promoting sparsity does not.
Since the training data have zero $y$-positions and momenta, a maximally sparse fit to these data also has zero $y$-components.
On the other hand, promoting $\SE(3)$-symmetry enables identification of the rotational symmetry about the $z$-axis, allowing the learned model to extrapolate from training data that lives only in the $x,z$-plane.
This is confirmed in Figure~\ref{fig:double_spring_pendulum}(d), where we plot the true and predicted trajectories of the first particle on a testing data set where all three components of each particle's initial position and momentum were chosen at random according to
\begin{gather}
    \vect{q}^i \sim \mcal{N} \left(\vect{q}_0,\ \mat{I} \right), \qquad
    \vect{p}^i \sim \mcal{N} \left(\vect{p}_0,\ 0.1\mat{I} \right),
    \qquad i=1, \ldots, n_p, \\
    \vect{q}_0 \sim \mcal{N}(\vect{0},\ \mat{I}), \qquad
    \vect{p}_0 \sim \mcal{N}(\vect{0},\ 0.1\mat{I}).
\end{gather}

\subsection{Neural ODE model of a double-spring-pendulum using EPNNs}
\label{subsec:double_spring_pendulum}
We use the double-spring-pendulum example from  \cite{Finzi2020generalizing} to study the nuclear-norm regularization technique for neural networks presented in Section~\ref{subsec:multilayer_perceptrons}.
As illustrated in Figure~\ref{fig:double_spring_pendulum}(a), this system consists of two masses $m_1 = m_2 = 1$ connected in series to the origin by springs with stiffness $k_1 = k_2 = 1$ and length $\ell_1 = \ell_2 = 1$.
The masses are subject to downward gravitational acceleration $\vect{g} = (0,0,-1)$.
The positions $\vect{q}^i = (q^i_x, q^i_y, q^i_z)$ and momenta $\vect{p}^i = (p^i_x, p^i_y, p^i_z)$ of the masses $i=1,2$ are described by a nonlinear Hamiltonian dynamical system
\begin{gather}\label{eqn:double_spring_pendulum}
    \vect{\dot{q}}^T = \frac{\partial H}{\partial \vect{p}}(\vect{q}, \vect{p}), 
    \qquad 
    \vect{\dot{p}}^T = - \frac{\partial H}{\partial \vect{q}}(\vect{q}, \vect{p}), \\
    H(\vect{q}, \vect{p}) = 
    \frac{\| \vect{p}^1 \|^2}{2 m_1} + \frac{\| \vect{p}^2 \|^2}{2 m_2}
    + \frac{k_1}{2}\big( \| \vect{q}^1 \| - \ell_1 \big)^2
    + \frac{k_2}{2}\big( \| \vect{q}^1 - \vect{q}^2 \| - \ell_2 \big)^2
    - m_1 \vect{g}^T \vect{q}^1 - m_1 \vect{g}^T \vect{q}^2,
\end{gather}
expressible in the form $\vect{\dot{x}} = \vect{F}(\vect{x})$ using the $12$-dimensional state variable $\vect{x} = (\vect{q}^1, \vect{p}^1, \vect{q}^{2}, \vect{p}^{2})$.
In the absence of gravity, the dynamics $\vect{F}$ would be equivariant with respect to rotations about the origin, that is, with respect to the simultaneous action of the orthogonal group $\Ogrp(3)$ on the positions and momenta of both masses.
Gravity breaks most of these symmetries, preserving only the subgroup $\Ogrp(2)$ of rotational symmetries about the vertical $z$-axis.


We compare approximations of the nonlinear function $F$ using EPNNs (see Figure~\ref{fig:symmetry_net} and Section~\ref{subsec:multilayer_perceptrons}) with different candidate symmetry groups to a standard multilayer perceptron (MLP) with a similar number of parameters.
These are trained using the Neural ODE framework of \cite{Chen2018neural} and data gathered in short chunks along trajectories.
In order to study the data efficiency of the resulting symmetry-promoting learning algorithms we train the models on different amounts of data.
For each training data set size, $20$ random experiments are performed, where the training data, testing data, and initial network weights are re-drawn from their respective distributions. 
The box-plots in Figure~\ref{fig:double_spring_pendulum} capture variance resulting from randomness across all of these factors.

\begin{figure}
    \centering
    \begin{tikzonimage}[trim=0 0 0 0, clip=true, width=0.95\textwidth]{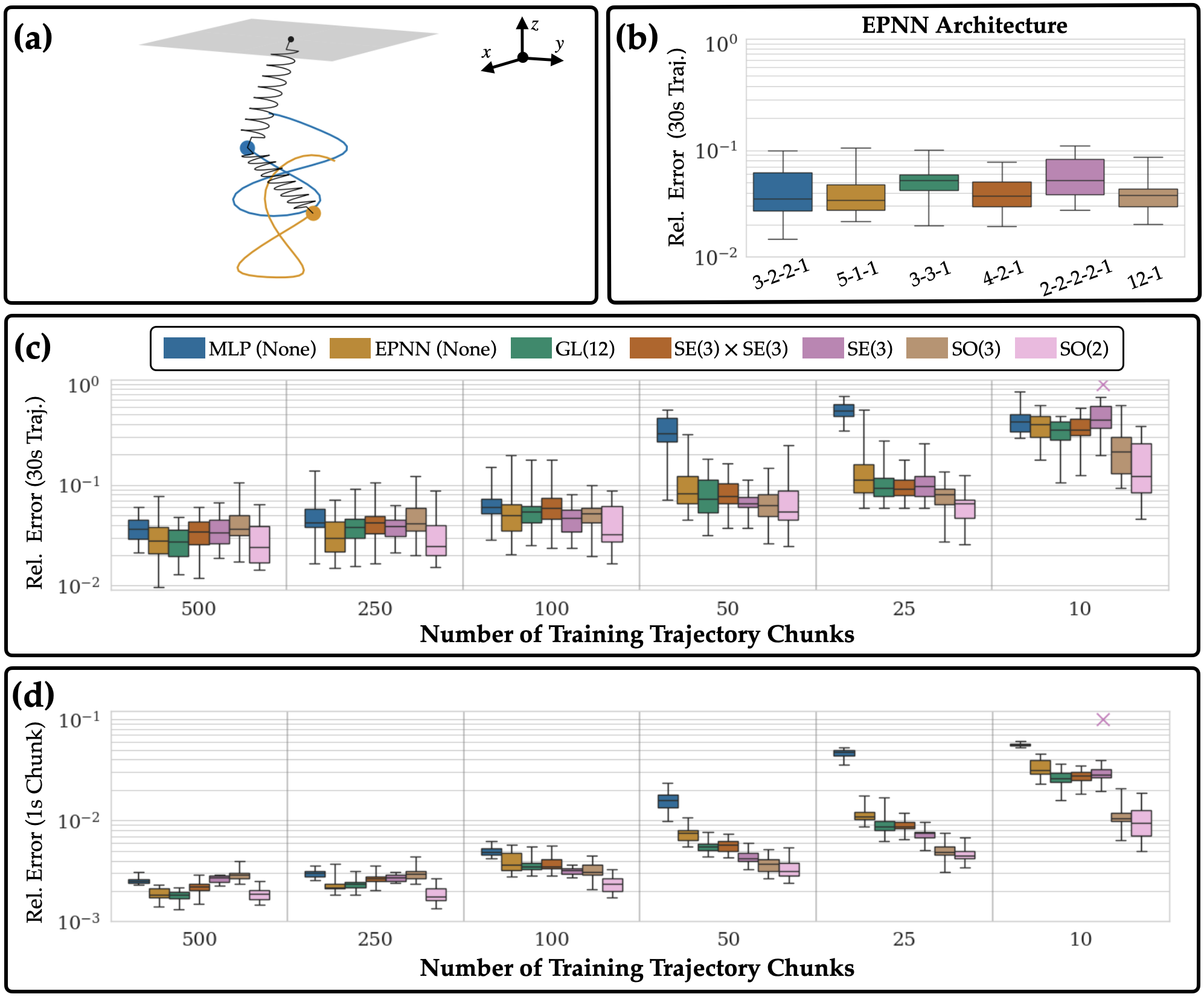}
    \end{tikzonimage}
    \caption{(a) Double-spring-pendulum system from \cite{Finzi2021practical}. (b) Prediction performance on $30$s trajectories using different $\SO(3)$-EPNN architectures trained on $M=100$ trajectory chunks. (c) Prediction performance using 3-3-1 EPNNs trained using different amounts of data and candidate groups. 
    (d) Prediction performance on $1$s trajectory chunks using the same setup as (c). The ``x'' represents a prediction that diverged.}
    \label{fig:double_spring_pendulum}
\end{figure}

\textbf{Data:} The training data for each experiment is generated using the same procedure as \cite{Finzi2020generalizing}, and consists of $M = 10$, $25$, $50$, $100$, $250$, or $500$ trajectory chunks $\big( \vect{x}_j(t_j),\ \vect{x}_j(t_j+\Delta t),\ \ldots,\ \vect{x}_j(t_j + n_t \Delta t)\big)$ sampled at $n_t = 5$ consecutive intervals of length $\Delta t = 0.2$s.
The $j$th chunk is generated by simulating the double-spring-pendulum for $30$s from normally distributed initial conditions,
\begin{equation}\label{eqn:double_spring_pendulum_ICs}
    \vect{q}_j^1(0) \sim \mcal{N}\big( (0,0,-1.5), \ 0.04 \mat{I} \big), \ \
    \vect{q}_j^2(0) \sim \mcal{N}\big( (0,0,-3), \ 0.04 \mat{I} \big), \ \
    \vect{p}_j^1(0),\ \vect{p}_j^2(0) \sim \mcal{N}\big( \vect{0}, \ 0.16 \mat{I} \big),
\end{equation}
and choosing the starting time of the chunk $t_j$ uniformly at random from the interval $[0,29]$.
An evaluation data set of $M$ trajectory chunks generated in the same way is used for early stopping.
Two testing data sets are used in each experiment.
The first consists of a single randomly generated $30$s trajectory with initial conditions drawn according to \eqref{eqn:double_spring_pendulum_ICs}.
The networks are evaluated by forecasting the entire trajectory from its initial condition at $t=0$.
The second testing data set in each experiment consists of $500$ trajectory chunks generated in the same way as the training data.

\textbf{Optimization:} With $\theta$ being the parameters of an EPNN approximation $\vect{F}_{\theta}$, and $\vect{\hat{x}}_j(t;\theta)$ being the trajectory of $\vect{\dot{\hat{x}}}_j = \vect{F}_{\theta}(\vect{\hat{x}}_j)$ with the same initial condition as the chunk $\vect{\hat{x}}_j(t_j;\theta) = \vect{x}_j(t_j)$, we use Neural ODE and the Adam optimizer of \cite{Kingma2014adam} to minimize the loss function
\begin{equation}\label{eqn:double_spring_pendulum_loss}
    J(\theta) = \underbrace{\frac{1}{M}\sum_{j=1}^{M} \sum_{k=1}^{n_t} \big\| \vect{x}_j(t_j + k\Delta t) - \vect{\hat{x}}_j(t_j + k\Delta t; \theta)  \big\|^2}_{\text{MSE}}
    + \gamma \underbrace{\vphantom{\sum_{j=1}^{M}}\big\| \mat{L}_{\vect{F}_{\theta}} \big\|_*}_{R_{\mathfrak{g},*}(\theta)}.
\end{equation}
Here, $\mat{L}_{\vect{F}_{\theta}}$ is the discretized EPNN operator described in Section~\ref{subsec:multilayer_perceptrons} with respect to the action of the candidate group $G$.
The regularization strength $\gamma = 10^{-4}$ was used for $G = \SO(2)$, $\SO(3)$, and $\SE(3)$; $\gamma = 10^{-6}$ was used for $G = \SE(3)\times \SE(3)$; and $\gamma = 10^{-8}$ was used for $G = \GL(12)$.
We use the default ODE solver in JAX (\cite{jax2018github}), which is the Dormand-Prince Runge-Kutta scheme with adaptive step-size and $\text{atol} = \text{rtol} = 1.4 \times 10^{-8}$. 
We fix a learning rate of $3\times 10^{-4}$ for all experiments and use optax's (\cite{deepmind2020jax}) default parameters $\beta_1 = 0.9$, $\beta_2 = 0.999$, $\epsilon=10^{-8}$ for Adam.
We use batches containing all of the training data and we perform $10^{4}$ optimization steps with early-stopping based on the validation data set performance.

\textbf{Architecture:} The EP blocks in our EPNNs are maps $F_{\text{b}}: \R^{12}\to\R^{12}$ defined by \eqref{eqn:symmetry_block} with the sigmoid activation, and the same $G$-action across all layers defined by a fixed representation
\begin{equation}
    \Phi_0 = \tilde{\Phi}_1 = \cdots = \tilde{\Phi}_{L-1} = \Phi_L.
\end{equation}
We compare six different candidate Lie group actions: 
\begin{enumerate}[left=0pt, itemsep=0pt, parsep=0pt,topsep=0pt]
    \item No group action,
    \item $\GL(12)$ acting via matrix multiplication with the state vector,
    \item $\SE(3)\times \SE(3)$ with each copy of $\SE(3)$ acting via \eqref{eqn:SE_action} on the position and momentum of a corresponding mass,
    \item $\SE(3)$ acting via \eqref{eqn:SE_action} on the positions and momenta of both masses,
    \item $\SO(3)$ acting to rotate the positions and momenta of both masses, and
    \item $\SO(2)$ acting to rotate the positions and momenta of both masses about the $z$-axis.
\end{enumerate}

All network weights were initialized uniformly at random in the interval $[-0.1, 0.1]$.
The components $\mat{L}_{F_{\text{eq}}}$ and $\mat{L}_{F_{\text{inv}}}$ of each EP block operator were discretized using the Monte-Carlo method presented in Section~\ref{sec:discretization} with uniformly-sampled points from axis-aligned cubes centered at the origin.
The cubes in the first layer have the minimum width to contain the training data, while the cubes in subsequent layers have width $2$.
The number of Monte-Carlo samples was taken to be $4(n_{\text{params}} + 1)$, where $n_{\text{params}}$ is the number of trainable parameters in the respective component $F_{\text{eq}}$ or $F_{\text{inv}}$.

We consider several architectures such as \emph{4-2-1}, meaning there are $4$ EP blocks in the first EP layer, $2$ EP blocks in the second EP layer, and a final affine EP output layer.
These are compared to a standard fully-connected MLP with layer widths 12-42-42-42-12 and the Swish activation function from \cite{Ramachandran2017searching}.
The MLP was trained and initialized in the same way as the EPNNs, except that the nuclear norm term was not included in loss function \eqref{eqn:double_spring_pendulum_loss}.
The number of parameters for each network is reported in Table~\ref{tab:double_spring_pendulum_architectures}.
All networks were trained using a single NVIDIA RTX 2080 Ti GPU.
The EPNNs took roughly $10$ minutes to train, which was about $5\times$ longer than the time to train the baseline MLP.
\begin{table}[h]
    \centering
    \begin{tabular}{cc|c|c|c|c|c|c}
        architecture: & 3-2-2-1 & 5-1-1 & 3-3-1 & 4-2-1 & 2-2-2-2-1 & 12-1 & 12-42-42-42-12 MLP \\
        parameters: & $4585$ & $4014$ & $4626$ & $4914$ & $2474$ & $4704$ & $4674$
    \end{tabular}
    \caption{Number of trainable parameters in each network used for the double-spring-pendulum.}
    \label{tab:double_spring_pendulum_architectures}
\end{table}

\textbf{Results:} We compare the performance of different $\SO(3)$-EPNN architectures trained on $M=100$ trajectory chunks in Figure~\ref{fig:double_spring_pendulum}(b).
In each random experiment, we follow \cite{Finzi2021practical} and compute the geometric mean of the relative error for predictions $\vect{\hat{x}}(t)$ along the $30$s testing trajectory $\vect{x}(t)$ given by
\begin{equation}\label{eqn:double_spring_pendulum_error}
    \text{Relative Error} = \exp\left[ \frac{1}{150} \sum_{k=1}^{150} \ln\left( \frac{\| \vect{x}(k\Delta t) - \vect{\hat{x}}(k\Delta t) \|}{\| \vect{x}(k\Delta t)\| + \|\vect{\hat{x}}(k\Delta t) \|} \right) \right].
\end{equation}
Comparable performance is observed across a wide range of different architectures.
We do not use these results to perform architecture selection, and the 3-3-1 EPNN is used as a representative architecture for the following sample-efficiency comparisons.

The sample efficiency of the MLP and 3-3-1 EPNNs with different candidate symmetry groups are compared in Figures~\ref{fig:double_spring_pendulum}(c)~and~\ref{fig:double_spring_pendulum}(d) by training on different numbers of trajectory chunks.
This allows us to study the degree to which larger candidate symmetry groups can still serve as useful inductive biases to improve generalization using smaller training datasets.
We report results across experiments using the same error metric \eqref{eqn:double_spring_pendulum_error}.
The MLP exhibits an abrupt loss of performance when the number of training chunks is reduced from $M=100$ to $M=50$, while the performance of EPNNs deteriorates more gradually.
Using smaller candidate symmetry groups yields performance benefits in terms of median relative error when $M=50$ or fewer trajectory chunks are used for training.
These effects are most significant when compared on the $1$s chunk prediction task shown in Figure~\ref{fig:double_spring_pendulum}(d).

\section{Generalization to sections of vector bundles}
\label{sec:sections_of_vector_bundles}
The machinery for enforcing, discovering, and promoting symmetry of maps $F:\mcal{V} \to \mcal{W}$ between finite-dimensional vector spaces is a special case of more general machinery for sections of vector bundles presented here.
Applications of this more general framework include studying the symmetries of vector fields, tensor fields, dynamical systems, and integral operators manifolds with respect to nonlinear group actions.
We rely heavily on background, definitions, and results that can be found in \cite{Lee2013introduction}, \cite{MarsdenMTAA}, and \cite{Kolar1993natural}.

\begin{figure}
    \centering
    \begin{tikzonimage}[trim=120 80 200 170, clip=true, width=0.8\textwidth]{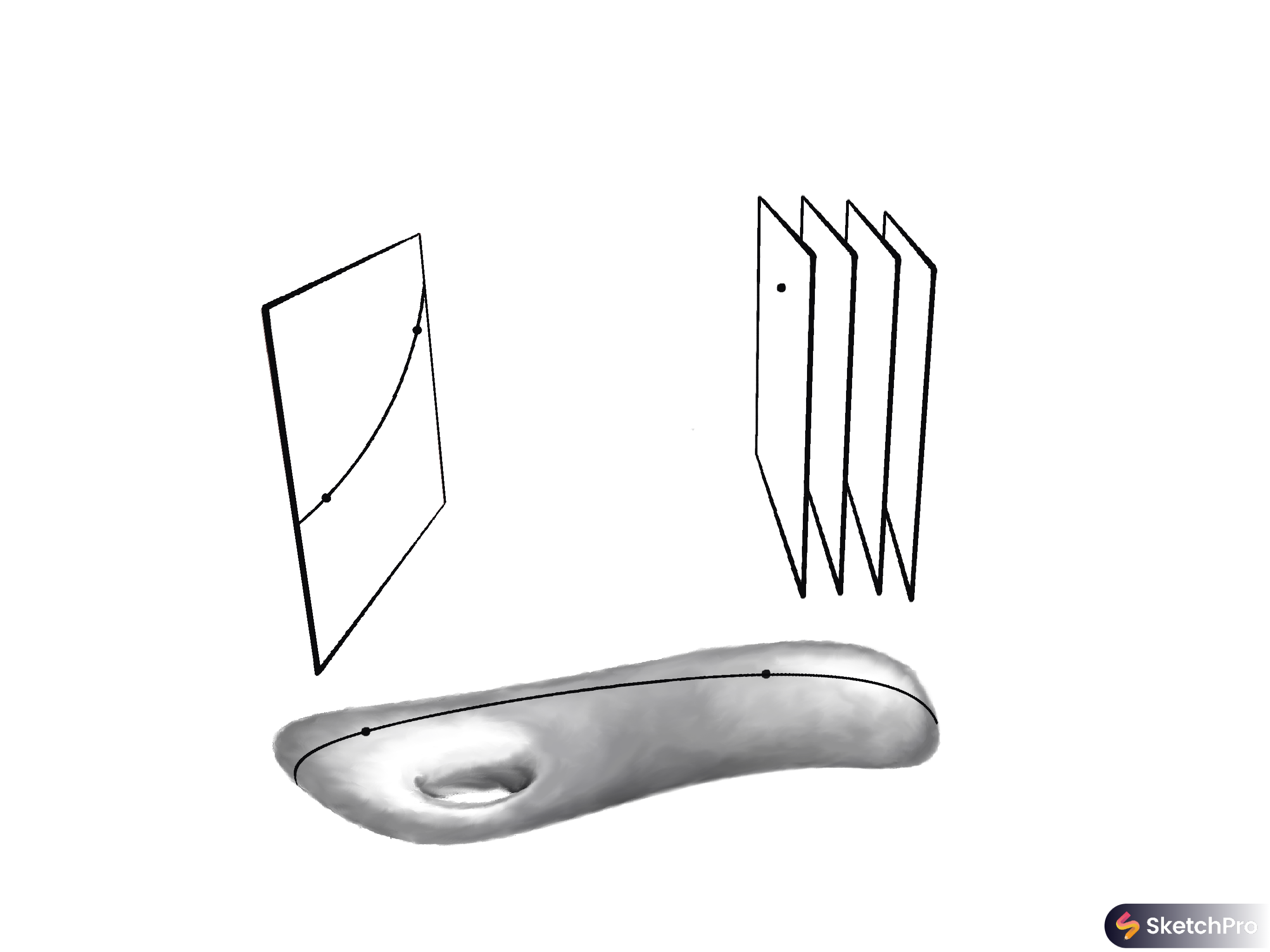}
        \node[rotate=0] at (0.6, 0.1) {\footnotesize $\mcal{M}$};
        \node[rotate=0, anchor=north] at (0.260, 0.2) {\footnotesize $p$};
        \node[rotate=0, anchor=north] at (0.675, 0.285) {\footnotesize $ q = \theta_{\exp(t\xi)}(p)$};
        \node[rotate=0, anchor=south] at (0.24, 0.85) {\footnotesize $E_p$};
        \draw[->] (0.261, 0.205) -- (0.23, 0.45);
        \node[rotate=0] at (0.24, 0.49) {\footnotesize $F(p)$};
        \draw[->] (0.225, 0.530) -- (0.320, 0.660);
        \node[rotate=0, anchor=west] at (0.335, 0.675) {\footnotesize $\mcal{L}_{\xi}F(p)$};
        \node[rotate=0, anchor=east] at (0.325, 0.775) {\footnotesize $\mcal{K}_{\exp(t\xi)}F(p)$};
        \node[rotate=0] at (0.640, 0.90) {\footnotesize $E_q$};
        \draw[->] (0.672, 0.285) -- (0.686, 0.73);
        \node[rotate=0] at (0.690, 0.775) {\footnotesize $F(q)$};
        \draw[->] (0.690, 0.818) -- (0.34, 0.763);
        \node[rotate=0, anchor=south] at (0.500, 0.800) {\footnotesize $\Theta_{\exp(-t\xi)}$};
        \node[rotate=0] at (0.87, 0.6) {\footnotesize $E$};
        \draw[->] (0.87, 0.55) -- (0.87, 0.25);
        \node[rotate=0, anchor=west] at (0.87, 0.4) {\footnotesize $\pi$};
    \end{tikzonimage}
    \caption{Fundamental operators for sections of vector bundles equipped with fiber-linear Lie group actions. 
    The action $\Theta_{g}$ is linear on each fiber $E_p$ and descends under the bundle projection $\pi: E \to \mcal{M}$ to an action $\theta_g$ on $\mcal{M}$. 
    Given a section $F: \mcal{M} \to E$, the finite transformation operators $\mcal{K}_g$ produce a new section $\mcal{K}_g F$ whose value at $p$ is given by evaluating $F$ at $q=\theta_g(p)$ and pulling the value in $E_{q}$ back to $E_p$ via the linear map $\Theta_{g^{-1}}$ on $E_q$. 
    The operators $\mcal{K}_g$ are linear thanks to linearity of $\Theta_{g^{-1}}$ on every $E_q$. 
    The Lie derivative $\mcal{L}_{\xi}$ is the operator on sections formed by differentiating $t\mapsto \mcal{K}_{\exp(t\xi)}$ at $t=0$. 
    Geometrically, $\mcal{L}_{\xi} F (p)$ is the vector in $E_p$ lying tangent to the curve $t\mapsto \mcal{K}_{\exp(t\xi)} F(p)$ in $E_p$ passing through $F(p)$ at $t=0$.}
    \label{fig:Lie_derivative}
\end{figure}

First, we provide some background on smooth vector bundles.
A rank-$k$ smooth vector bundle $E$ is a collection of $k$-dimensional vector spaces $E_p$, called \emph{fibers}, organized smoothly over a base manifold $\mcal{M}$.
This fibers are organized by the \emph{bundle projection} $\pi:E \to \mcal{M}$, a surjective map whose preimages are the fibers $E_p = \pi^{-1}(p)$.
Smoothness means that $\pi$ is a smooth submersion where $E$ is a smooth manifold covered by smooth local trivializations
$$
    \psi_{\alpha}: \pi^{-1}(\mcal{U}_{\alpha})\subset E \to \mcal{U}_{\alpha} \times \R^k,
$$
with $\{\mcal{U}_{\alpha}\}_{\alpha\in\mcal{A}}$ being open subsets covering $\mcal{M}$.
The transition functions between local trivializations are $\R^k$-linear, meaning that there are smooth matrix-valued functions $\mat{T}_{\alpha, \beta}: \mcal{U}_{\alpha}\cap\mcal{U}_{\beta} \to \R^{k\times k}$ satisfying
\begin{equation}
    \psi_{\alpha} \circ \psi_{\beta}^{-1}(p, \vect{v})
    = (p, \mat{T}_{\alpha, \beta}(p)\vect{v})
\end{equation}
for every $p \in \mcal{U}_{\alpha}\cap\mcal{U}_{\beta}$ and $\vect{v} \in \R^k$.
The bundle with this structure is often denoted $\pi: E \to \mcal{M}$.

A \emph{section} of the rank-$k$ vector bundle $\pi:E \to \mcal{M}$ is a map $F: \mcal{M} \to E$ satisfying $\pi \circ F = \Id_{\mcal{M}}$.
The space of (possibly rough) sections, denoted $\sect(E)$, is a vector space with addition and scalar multiplication defined pointwise in each fiber $E_p$.
We equip $\sect(E)$ with the topology of pointwise convergence, making it into a locally-convex space.
The space of sections possessing $m$ continuous derivatives is denoted $C^m(\mcal{M},E)$, with the space of merely continuous sections being $C(\mcal{M},E) = C^0(\mcal{M},E)$ and the space of smooth sections being $C^{\infty}(\mcal{M},E)$.
We use $\vf(\mcal{M}) = C^{\infty}(\mcal{M},T\mcal{M})$ to denote the space of smooth vector fields on a manifold.
A vector bundle and a section are depicted in Figure~\ref{fig:Lie_derivative}, along with the fundamental operators for a group action introduced below.

We consider a smooth \emph{\textbf{fiber-linear}} right $G$-action $\Theta : E \times G \to E$, meaning that every $\Theta_g = \Theta(\cdot, g):E \to E$ is a smooth vector bundle homomorphism.
In other words, $\Theta$ descends under the bundle projection $\pi$ to a unique smooth right $G$-action $\theta : \mcal{M} \times G \to \mcal{M}$ so that the diagram
\begin{equation}
    \begin{tikzcd}
        E \arrow[r, "\Theta_g"] \arrow[d, "\pi"'] & E \arrow[d, "\pi"] \\ 
        \mcal{M} \arrow[r, "\theta_g"'] & \mcal{M}
    \end{tikzcd}
    \label{cd:action_by_vb_homomorphisms}
\end{equation}
commutes and the restricted maps $\Theta_g\vert_{E_p}: E_p \to E_{\theta(p,g)}$ are linear.
It is clear that \eqref{eqn:Theta_action_on_VW} is such an action on the trivial vector bundle $\pi: \mcal{V} \times \mcal{W} \to \mcal{V}$.
We define what it means for a section to be symmetric with respect to the action $\Theta$ as follows:
\begin{definition}
    \label{def:equivariance_of_section}
    A section $F\in\Sigma(E)$ is \textbf{equivariant} with respect to a transformation $g\in G$ if
    \begin{equation}
        \mcal{K}_g F := \Theta_{g^{-1}} \circ F \circ \theta_{g} = F.
        \label{eqn:transformation_operators_vb}
    \end{equation}
    These transformations form a subgroup of $G$ denoted $\Sym_G(F)$.
\end{definition}
The operators $\mcal{K}_{g}$ are depicted in Figure~\ref{fig:Lie_derivative}, which is analogous to Figure~\ref{fig:baby_Lie_derivative}.
Thanks to the vector bundle homomorphism properties of $\Theta_{g^{-1}}$, the operators $\mcal{K}_g: \sect(E) \to \sect(E)$ are well-defined and linear.
Moreover, they form a group under composition $\mcal{K}_{g_1} \mcal{K}_{g_2} = \mcal{K}_{g_1\cdot g_2}$,
with inverses $\mcal{K}_g^{-1} = \mcal{K}_{g^{-1}}$.

The \emph{infinitesimal generator} of the group action is the linear map $\hat{\Theta}:\Lie(G) \to \vf(E)$ defined by
\begin{equation}
    \hat{\Theta}(\xi) = \left. \ddt \right\vert_{t=0} \Theta_{\exp(t\xi)}.
\end{equation}
It turns out that this vector field is $\Theta$-related to $(0, \xi) \in \vf(\mcal{M}\times G)$ (see Lemma~5.13 in \cite{Kolar1993natural}, Lemma~20.14 in \cite{Lee2013introduction}),
meaning that the flow of $\hat{\Theta}(\xi)$ is given by
\begin{equation}
    \Flow_{\hat{\Theta}(\xi)}^t = \Theta_{\exp(t\xi)}.
    \label{eqn:exp_action_and_flow_on_vb}
\end{equation}
Likewise, $\theta_{\exp(t\xi)}$ is the flow of $\hat{\theta}(\xi) = \left. \ddt \right\vert_{t=0} \theta_{\exp(t\xi)} \in \vf(\mcal{M})$, which is $\pi$-related to $\hat{\Theta}(\xi)$.

Differentiating the smooth curves $t \mapsto \mcal{K}_{\exp(t\xi)} F (p)$ lying in $E_p$ for each $p\in\mcal{M}$ gives rise to the Lie derivative $\mcal{L}_{\xi}:D(\mcal{L}_{\xi}) \subset \sect(E) \to \sect(E)$ along $\xi \in \Lie(G)$ defined by
\begin{equation}
    \boxed{
    \mcal{L}_{\xi} F 
    = \left.\ddt \right\vert_{t=0} \mcal{K}_{\exp(t\xi)} F 
    = \lim_{t\to 0} \frac{1}{t}\left( \Theta_{\exp(-t\xi)}\circ F \circ \theta_{\exp(t\xi)} - F \right),
    }
    \label{eqn:Lie_derivative}
\end{equation}
where $F \in D(\mcal{L}_{\xi})$ if and only if the limit converges in $\sect(E)$, i.e., pointwise.
Note that we implicitly identify $T_{F(p)} E_p \cong E_p$.
This construction is illustrated in Figure~\ref{fig:Lie_derivative}.
We emphasize that the Lie derivative \eqref{eqn:Lie_derivative} is a linear operator on sections of the vector bundle $E$.
Moreover, it is linear with respect to $\xi$, as we will show in Proposition~\ref{prop:properties_of_Lie_derivative}.
This allows us to directly apply approaches described in Sections~\ref{sec:enforcing_symmetry}-\ref{sec:promoting_symmetry} to enforce, discover, and promote symmetry using Lie derivatives for more general sections of vector bundles.
In fact, the fundamental operators described in Section~\ref{sec:fundamental_operators} are versions of the operators described here for smooth functions $F:\mcal{V} \to \mcal{W}$ viewed as sections $x \mapsto (x,F(x))$ of the trivial bundle $\pi:\mcal{V}\times\mcal{W} \to \mcal{V}$ and acted upon by the fiber-liner right $G$-action \eqref{eqn:Theta_action_on_VW}.

\begin{remark}[Lie derivatives using flows]
    Thanks to \eqref{eqn:exp_action_and_flow_on_vb}, the Lie derivative defined in \eqref{eqn:Lie_derivative} only depends on the infinitesimal generator $\hat{\Theta}(\xi) \in \vf(E)$, and its flow for small time $t$.
    Hence, any vector field in $\vf(E)$ whose flow is fiber-linear, but not necessarily defined for all $t\in\R$, gives rise to an analogously-defined Lie derivative acting linearly on $\sect(E)$.
    These are the so-called \emph{linear vector fields} described by \cite{Kolar1993natural} in Section~47.9.
    In fact, more general versions of the Lie derivative based on flows for maps between manifolds are described by \cite{Kolar1993natural} in Chapter~11.
    However, these generalizations are nonlinear operators, destroying the convex properties of the symmetry-promoting regularization functions in Section~\ref{sec:promoting_symmetry}.
\end{remark}

Some useful properties of the fundamental operators $\mcal{K}_g$ and $\mcal{L}_\xi$ for studying symmetries are:
\begin{proposition}
    \label{prop:properties_of_Lie_derivative}
    For every $\xi, \eta \in\Lie(G)$, and $\alpha, \beta, t \in\R$, we have
    \begin{equation}
        \ddt \mcal{K}_{\exp(t\xi)}F 
        = \mcal{L}_{\xi} \mcal{K}_{\exp(t\xi)} F
        = \mcal{K}_{\exp(t\xi)} \mcal{L}_{\xi} F
        \qquad \forall F \in D(\mcal{L}_{\xi}),
        \label{eqn:Lie_derivative_as_generator}
    \end{equation}
    \begin{equation}
        \mcal{L}_{\alpha\xi + \beta \eta} F = \big(\alpha \mcal{L}_{\xi} + \beta \mcal{L}_{\eta} \big) F
        \qquad \forall F \in C^1(\mcal{M},E),
        \label{eqn:linearity_of_Lie_derivative_wrt_Lie_algebra}
    \end{equation}
    and
    \begin{equation}
        \mcal{L}_{[\xi, \eta]}F 
        = \big(\mcal{L}_{\xi}\mcal{L}_{\eta} - \mcal{L}_{\eta}\mcal{L}_{\xi}\big)F
        \qquad \forall F \in C^2(\mcal{M},E).
        \label{eqn:Lie_derivative_commutator}
    \end{equation}
    We give a proof in Appendix~\ref{app:properties_of_Lie_derivative}.
\end{proposition}
Taken together, these results mean that $\Pi:g \mapsto \mcal{K}_g$ and $\Pi_*:\xi \mapsto \mcal{L}_{\xi}$ are (infinite-dimensional) representations of $G$ and $\Lie(G)$ in $C^{\infty}(\mcal{M},E)$.

The main results of this section are the following two theorems.
The first completely characterizes the identity component of $\Sym_G(F)$ by correspondence with its Lie subgalgebra (see Theorem~19.26 in \cite{Lee2013introduction}).
The second gives necessary and sufficient conditions for a section to be $G$-equivariant.
\begin{theorem}
    \label{thm:symmetries_of_sections}
    If $F \in C(\mcal{M}, E)$ is a continuous section, then $\Sym_G(F)$ is a closed, embedded Lie subgroup of $G$ whose Lie subalgebra is
    \begin{equation}
        \sym_G(F) = \left\{ \xi \in \Lie(G) \ : \ F\in D(\mcal{L}_{\xi})\ \mbox{and}\ \mcal{L}_{\xi} F = 0 \right\}.
        \label{eqn:sym_G_for_vb_section}
    \end{equation}
    We give a proof in Appendix~\ref{app:symmetries_of_sections}.
\end{theorem}
\begin{theorem}
    \label{thm:equivariance_conditions_for_vb_section}
    Suppose that $G$ has $n_G$ connected components with $G_0$ being the component containing the identity element.
    Let $\xi_1, \ldots, \xi_q$ generate $\Lie(G)$ and let $g_1, \ldots, g_{n_G - 1}$ be elements from each non-identity component of $G$. 
    A continuous section $F \in C(\mcal{M},E)$ is $G_0$-equivariant if and only if
    \begin{equation}
        F \in D(\mcal{L}_{\xi_i}) \quad \mbox{and} \quad \mcal{L}_{\xi_i} F = 0, \qquad  i=1,\ldots, q.
    \end{equation}
    If, in addition, we have 
    \begin{equation}
        \mcal{K}_{g_j} F = F, \qquad  j=1, \ldots, n_G-1,
    \end{equation}
    then $F$ is $G$-equivariant.
    We give a proof in Appendix~\ref{app:equivariance_conditions_for_vb_section}.
\end{theorem}
These results allow us to promote, enforce, and discover symmetries for sections of vector bundles in fundamentally the same way we did for maps between finite-dimensional vector spaces in Sections~\ref{sec:enforcing_symmetry},~\ref{sec:discovering_symmetry},~and~\ref{sec:promoting_symmetry}.
In particular, symmetry can be enforced through analogous linear constraints, discovered through nullspaces of analogous operators, and promoted through analogous convex penalties based on the nuclear norm.

\begin{remark}[Left actions]
    Theorems~\ref{thm:equivariance_conditions_for_vb_section}~and~\ref{thm:symmetries_of_sections} hold without any modification for left G-actions $\Theta^L:G\times E \to E$.
    This is because we can define a corresponding right $G$-action by $\Theta^R(p,g) = \Theta^L(g^{-1},p)$ with associated operators related by 
    \begin{equation}
        \mcal{K}^R_{g} = \mcal{K}^L_{g^{-1}}
        \qquad \mbox{and}\qquad
        \mcal{L}^R_{\xi} = - \mcal{L}^L_{\xi}.
    \end{equation}
    The symmetry group $\Sym_G(F)$ does not depend on whether it is defined by the condition $\mcal{K}^R_{g} F = F$ or by $\mcal{K}^L_{g} F = F$.
    It is slighly less natural to work with left actions because $\Pi^L: g \mapsto \mcal{K}^L_g$ and $\Pi^L_*: \xi \mapsto \mcal{L}^L_\xi$ are Lie group and Lie algebra anti-homomorphisms, that is, 
    \begin{equation}
        \Pi^L(g_1 g_2) = \Pi^L(g_2) \Pi^L(g_1) 
        \qquad \mbox{and} \qquad
        \Pi^L_*\big([\xi,\eta]\big) 
        = \big[\Pi^L_*(\eta), \Pi^L_*(\xi)\big].
    \end{equation}
\end{remark}

\subsection{Vector fields}
\label{subsec:symmetry_of_vf}
    Here we study the symmetries of a vector field $V\in\vf(\mcal{M})$ under a right $G$-action $\theta:\mcal{M}\times G \to\mcal{M}$.
    This allows us to extend the discussion in Section~\ref{subsec:dynamical_systems} to dynamical systems described by vector fields on smooth manifolds and acted upon nonlinearly by arbitrary Lie groups.
    The tangent map of the diffeomorhpism $\theta_g = \theta(\cdot, g):\mcal{M} \to \mcal{M}$ transforms vector fields via the pushforward map $(\theta_{g})_* : \vf(\mcal{M}) \to \vf(\mcal{M})$ defined by
    \begin{equation}
        ((\theta_{g})_* V)_{p \cdot g} = \D \theta_g(p) V_p
    \end{equation}
    for every $p\in\mcal{M}$.
    \begin{definition}
        \label{def:invariance_of_vf}
        Given $g\in G$, we say that a vector field $V\in\vf(\mcal{M})$ is $g$-\textbf{invariant} if and only if $(\theta_g)_* V = V$, that is,
        \begin{equation}
            V_{p\cdot g} = \D \theta_g(p) V_p \qquad \forall p\in\mcal{M}.
        \end{equation}
    \end{definition}
    Because $(\theta_{g^{-1}})_* (\theta_g)_* = (\theta_{g^{-1}}\circ\theta_g)_* = \Id_{\vf(\mcal{M})}$, it is clear that a vector field is $g$-invariant if and only if it is $g^{-1}$-invariant.
    
    Recall that vector fields $V\in\vf(\mcal{M})$ are smooth sections of the tangent bundle $E = T\mcal{M}$. 
    The right $G$-action $\theta$ on $\mcal{M}$ induces a right $G$-action $\Theta: T\mcal{M} \times G \to T\mcal{M}$ on the tangent bundle defined by
    \begin{equation}
        \Theta_g(v_p) = \D\theta_g(p) v_p.
    \end{equation}
    It is easy to see that each $\Theta_g$ is a vector bundle homomorphism descending to $\theta_g$ under the bundle projection $\pi$.
    The key observation that
    \begin{equation}
        \mcal{K}_g V 
        = \Theta_{g^{-1}} \circ V \circ \theta_g
        = (\theta_{g^{-1}})_* V
    \end{equation}
    means a vector field $V\in\vf(\mcal{M})$ is $g$-invariant if and only if it is $g$-equivariant as a section of $T\mcal{M}$ with respect to the action $\Theta$.
    Recall that (by Lemma~20.14 in \citet{Lee2013introduction}) the left-invariant vector field $\xi \in \Lie(G) \subset \vf(G)$ and its infinitesimal generator $\hat{\theta}(\xi) \in \vf(\mcal{M})$ are $\theta^{(p)}$-related, where $\theta^{(p)}: g \mapsto \theta(p,g)$ is the orbit map.
    This means that $\theta_{\exp(t\xi)}$ is the time-$t$ flow of $\hat{\theta}(\xi)$ by Proposition~9.13 in \citet{Lee2013introduction}.
    As a result, the Lie derivative in \eqref{eqn:Lie_derivative} agrees with the standard Lie derivative of $V$ along $\hat{\theta}(\xi)$ (see \citet[p.228]{Lee2013introduction}),
    that is,
    \begin{equation}
    \boxed{
        \mcal{L}_{\xi} V(p) 
        = \lim_{t\to 0} \frac{1}{t} \left[ \D\theta_{\exp(-t\xi)}(\theta_{\exp(t\xi)}(p)) V_{\theta_{\exp(t\xi)}(p)} - V_p \right]
        = [\hat{\theta}(\xi), V]_p,
        }
    \end{equation}
    where the expression on the right is the Lie bracket of $\hat{\theta}(\xi)$ and $V$ evaluated at $p\in\mcal{M}$.

\subsection{Tensor fields}
\label{subsec:symmetry_of_tf}
    Symmetries of a tensor field can also be revealed using our framework.
    This will be useful for our study of Hamiltonian dynamics in Section~\ref{subsec:Hamiltonian_conservation_laws} and for our study of integral operators, whose kernels can be viewed as tensor fields, in Section~\ref{subsec:equivariant_integral_operators}.
    For simplicity, we consider a rank-$k$ covariant tensor field $A \in \tf^k(\mcal{M})$, although our results extend to contravariant and mixed tensor fields with minimal modification.
    We rely on the basic definitions and machinery found in \citet[Chapter~12]{Lee2013introduction}.
    Under a right $G$-action $\theta$ on $\mcal{M}$, the tensor field transforms via the pullback map $\theta_g^*:\tf^k(\mcal{M}) \to \tf^k(\mcal{M})$ defined by
    \begin{equation}
        (\theta_g^* A)_p(v_1, \ldots, v_k) 
        = (\D\theta_g(p)^* A_{p\cdot g})(v_1, \ldots, v_k) 
        = A_{p\cdot g}(\D\theta_g(p) v_1, \ldots, \D\theta_g(p) v_k)
    \end{equation}
    for every $v_1, \ldots, v_k \in T_p\mcal{M}$.
    \begin{definition}
        \label{def:invariance_of_tf}
        Given $g\in G$, a tensor field $A\in\tf^k(\mcal{M})$ is $g$-\textbf{invariant} if and only if $\theta_g^* A = A$, i.e,
        \begin{equation}
            A_{p\cdot g}(\D\theta_g(p) v_1, \ldots, \D\theta_g(p) v_k) = A_p(v_1, \ldots, v_k) \qquad \forall p\in\mcal{M}.
        \end{equation}
    \end{definition}
    
    To study the invariance of tensor fields in our framework, we recall that a tensor field is a section of the tensor bundle 
    $E = T^kT^*\mcal{M} = \coprod_{p\in\mcal{M}} (T_p^*\mcal{M})^{\otimes k}$, a vector bundle over $\mcal{M}$,
    where $T_p^*\mcal{M}$ is the dual space of $T_p\mcal{M}$.
    The right $G$-action $\theta$ on $\mcal{M}$ induces a
    right $G$-action $\Theta: T^k T\mcal{M} \times G \to T^k T\mcal{M}$ defined by
    \begin{equation}
        \Theta_g(A_p) = \D \theta_{g^{-1}}(\theta_g(p))^* A_p.
    \end{equation}
    It is clear that each $\Theta_g$ is a homomorphism of the vector bundle $T^kT^*\mcal{M}$ descending to $\theta_g$ under the bundle projection.
    The key observation that
    \begin{equation}
        \mcal{K}_g A 
        = \Theta_{g^{-1}}\circ A \circ \theta_g
        = \theta_g^* A
    \end{equation}
    means a tensor field $A\in\tf^k(\mcal{M})$ is $g$-invariant if and only if it is $g$-equivariant as a section of $T^kT^*\mcal{M}$ with respect to the action $\Theta$.
    Since $\theta_{\exp(t\xi)}$ gives the time-$t$ flow of the vector field $\hat{\theta}(\xi)\in\vf(\mcal{M})$, the Lie derivative in \eqref{eqn:Lie_derivative} for this action agrees with the standard Lie derivative of $A\in\tf^k(\mcal{M})$ along $\hat{\theta}(\xi)$ (see \citet[p.321]{Lee2013introduction}), that is
    \begin{equation}
    \boxed{
    (\mcal{L}_{\xi} A)_p 
    = \lim_{t\to 0} \frac{1}{t}\left[ \D\theta_{\exp(t\xi)}(p)^* A_{\theta_{\exp(t\xi)}(p)} - A_p \right]
    = (\mcal{L}_{\hat{\theta}(\xi)} A)_p.
    }
    \label{eqn:Lie_derivative_for_tf}
\end{equation}
The Lie derivative for arbitrary covariant tensor fields can be computed by applying Proposition~12.32 in \cite{Lee2013introduction} and its corollaries.
More generally, thanks to 6.16-18 in \cite{Kolar1993natural}, the Lie derivative for any tensor product of sections of natural vector bundles can be computed via the formula
\begin{equation}
    \mcal{L}_{\xi}(A_1\otimes A_2) = (\mcal{L}_{\xi} A_1) \otimes A_2 + A_1 \otimes (\mcal{L}_{\xi} A_2).
\end{equation}
For example, this holds when $A_1, A_2$ are arbitrary smooth tensor fields of mixed types.
The Lie derivative of a differential form $\omega$ on $\mcal{M}$ can be computed by Cartan's magic formula
\begin{equation}
    \mcal{L}_{\xi} \omega = \hat{\theta}(\xi) \intmul (\td \omega) + \td (\hat{\theta}(\xi) \intmul \omega),
    \label{eqn:Cartan_magic_formula}
\end{equation}
where $\td$ is the exterior derivative.

\subsection{Hamiltonian dynamics}
\label{subsec:Hamiltonian_conservation_laws}

The dynamics of frictionless mechanical systems can be described by Hamiltonian vector fields on symplectic manifolds.
Roughly speaking, these encompass systems that conserve energy, such as motion of rigid bodies and particles interacting via conservative forces.
The celebrated theorem of \cite{noether1918invariante} says that conserved quantities of Hamiltonian systems correspond with symmetries of the energy function (the system's Hamiltonian).
In this section, we briefly illustrate how to enforce Hamiltonicity constraints on learned dynamical systems and how to promote, discover, and enforce conservation laws.
A thorough treatment of classical mechanics, symplectic manifolds, and Hamiltonian systems can be found in \cite{Abraham2008foundations,Marsden:MS}.
This includes methods for reduction of systems with known symmetries and conservation laws.
The following brief introduction follows Chapter~22 of \cite{Lee2013introduction}.

Hamiltonian systems are defined on symplectic manifolds.
That is, a smooth even-dimensional manifold $\mcal{S}$ together with a smooth, nondegenerate, closed differential $2$-form $\omega$, called the symplectic form.
Nondegeneracy means that the map $\hat{\omega}_p: v \mapsto \omega_p(v,\cdot)$ is a bijective linear map of $T_p \mcal{S}$ onto its dual $T_p^* \mcal{S}$ for every $p \in \mcal{S}$. 
Closure means that $\td \omega = 0$, where $\td$ is the exterior derivative.
Thanks to nondegeneracy, any smooth function $H \in C^{\infty}(\mcal{S})$ gives rise to a smooth vector field
\begin{equation}
    V_H = \hat{\omega}^{-1}(\td H)
\end{equation}
called the \emph{Hamiltonian vector field} of $H$.
A vector field $V \in \vf(\mcal{S})$ is said to be Hamiltonian if $V = V_H$ for some function $H$, called the Hamiltonian of $V$.
A vector field is locally Hamiltonian if it is Hamiltonian in neighborhood of each point of $\mcal{S}$.

The symplectic manifolds considered in classical mechanics usually consist of the cotangent bundle $\mcal{S} = T^* \mcal{M}$ of an $m$-dimensional manifold $\mcal{M}$ describing the \emph{configuration} of the system, e.g., the positions of particles.
The cotangent bundle has a canonical symplectic form given by
\begin{equation}
    \omega = \sum_{i=1}^{m} \td x^i \wedge \td \xi_i,
    \label{eqn:canonical_symplectic_form}
\end{equation}
where $(x^i, \xi_i)$ are any choice of natural coordinates on a patch of $T^* \mcal{M}$ (see Proposition~22.11 in \cite{Lee2013introduction}).
Here, each $x^i$ is a generalized coordinate describing the configuration and $\xi_i$ is its \emph{conjugate} or \emph{generalized momentum}.
The Darboux theorem (Theorem~22.13 in \cite{Lee2013introduction}) says that any symplectic form on a manifold can be put into the form of \eqref{eqn:canonical_symplectic_form} by a choice of local coordinates.
In these \emph{Darboux coordinates}, the dynamics of a Hamiltonian system are governed by the equations
\begin{equation}
    \ddt x^i = V_H(x^i) = \frac{\partial H}{\partial \xi_i}, 
    \qquad
    \ddt \xi_i = V_H(\xi_i) = - \frac{\partial H}{\partial x^i},
\end{equation}
which should be familiar to anyone who has studied undergraduate mechanics.

Enforcing local Hamiltonicity on a vector field $V\in\vf(\mcal{S})$ is equivalent to the linear constraint
\begin{equation}
    \mcal{L}_V \omega 
    = 0
\end{equation}
thanks to Proposition~22.17 in \cite{Lee2013introduction}. 
Here $\mcal{L}_V$ is the Lie derivative of the tensor field $\omega \in \tf^2(\mcal{S})$, i.e., \eqref{eqn:Lie_derivative_for_tf} with $\theta$ being the flow of $V$ and its generator being the identity $\hat{\theta}(V) = V$.
Note that the Lie derivative still makes sense even when the orbits $t \mapsto \theta_t(p) = \theta_{\exp(t 1)}(p)$ are only defined for small $t\in (-\varepsilon, \varepsilon)$.
In Darboux coordinates, this constraint is equivalent to the set of equations
\begin{equation}
    \frac{\partial V(x^i)}{\partial x^j} + \frac{\partial V(\xi_j)}{\partial \xi_i} = 0, \qquad
    \frac{\partial V(\xi_i)}{\partial x^j} - \frac{\partial V(\xi_j)}{\partial x^i} = 0, \qquad
    \frac{\partial V(x^i)}{\partial \xi_j} - \frac{\partial V(x^j)}{\partial \xi_i} = 0
\end{equation}
for all $1\leq i,j \leq m$.
When the first de Rham cohomology group satisfies $H^1_{\text{dR}}(\mcal{S}) = 0$, for example when $\mcal{S}$ is contractible, local Hamilonicity implies the existence of a global Hamiltonian for $V$, unique on each component of $\mcal{S}$ up to addition of a constant by \citet[Proposition~22.17]{Lee2013introduction}.

Of course our approach also makes it possible to promote Hamiltonicity with respect to candidate symplectic structures when learning a vector field $V$.
To do this, we can penalize the nuclear norm of $\mcal{L}_V$ restricted to a subspace $\tilde{\Omega}$ of candidate closed $2$-forms using the regularization function
\begin{equation}
    R_{\tilde{\Omega},*}(V) = \big\Vert \left. \mcal{L}_V \right\vert_{\tilde{\Omega}} \big\Vert_*.
\end{equation}
The strength of this penalty can be increased when solving a regression problem for $V$ until there is a non-degenerate $2$-form in the nullspace $\Null(\mcal{L}_V) \cap \tilde{\Omega}$.
This gives a symplectic form with respect to which $V$ is locally Hamiltonian.

Another option is to learn a (globally-defined) Hamiltonian function $H$ directly by fitting $V_H$ to data.
In this case, we can regularize the learning problem by penalizing the breaking of conservation laws.
The time-derivative of a quantity, that is, a smooth function $f\in C^{\infty}(\mcal{S})$ under the flow of $V_H$ is given by the Poisson bracket
\begin{equation}
    \{ f, H \} := \omega(V_f, V_H) = \td f(V_H) = V_H(f) = -V_f(H).
\end{equation}
Hence, $f$ is a conserved quantity if and only if $H$ is invariant under the flow of $V_f$ --- this is Noether's theorem.
It is also evident that the Poisson bracket is linear with respect to both of its arguments.
In fact, the Poisson bracket turns $C^{\infty}(\mcal{S})$ into a Lie algebra with $f \mapsto V_f$ being a Lie algebra homomorphism, i.e.,
$
    V_{\{f, g \}} = [V_f, V_g]
$.

As a result of these basic properties of the Poisson bracket, the quantities conserved by a given Hamiltonian vector field $V_H$ form a Lie subalgebra given by the nullspace of a linear operator $L_H : C^{\infty}(\mcal{S}) \mapsto C^{\infty}(\mcal{S})$ defined by
\begin{equation}
    L_H : f \mapsto \{ f, H \}.
\end{equation}
To promote conservation of quantities in a given subalgebra $\mfrak{g} \subset C^{\infty}(\mcal{S})$ when learning a Hamiltonian $H$, we can penalize the nuclear norm of $L_H$ restricted to $\mfrak{g}$, that is
\begin{equation}
    R_{\mfrak{g},*}(H) = \big\Vert \left.L_H\right\vert_{\mfrak{g}} \big\Vert_*.
\end{equation}
For example, we might expect a mechanical system to conserve angular momentum about some axes, but not others due to applied torques.
In the absence of data to the contrary, it often makes sense to assume that various linear and angular momenta are conserved.



\subsection{Multilinear integral operators}
\label{subsec:equivariant_integral_operators}
In this section we provide machinery to study the symmetries of linear and nonlinear integral operators acting on sections of vector bundles, yielding far-reaching generalizations of our results in Section~\ref{subsec:NNs_acting_on_fields}.
Such operators can form the layers of neural networks acting on various vector and tensor fields supported on manifolds.

Let $\pi_0: E_0 \to \mcal{M}_0$ and $\pi_j: E_j \to \mcal{M}_j$ be vector bundles with $\mcal{M}_j$ being $d_j$-dimensional orientable Riemannian manifolds with volume forms $\volform_j \in \Omega^{d_j}(T^*\mcal{M}_j)$, $j=1,\ldots,r$.
Note that here, $\volform_j$ does not denote the exterior derivative of a $(d_j-1)$-form.
A continuous section $K$ of the bundle 
\begin{equation}
    E = E_0 \otimes E_1^* \otimes \cdots \otimes E_r^* 
    := \coprod_{(p,q_1, \ldots, q_r)\in\mcal{M}_0\times\cdots\times\mcal{M}_r} E_{0,p} \otimes E_{1,q_1}^* \otimes \cdots \otimes E_{r,q_r}^* 
\end{equation}
can be viewed as a continuous family of $r$-multilinear maps
\begin{equation}
    K(p,q_1,\ldots,q_r):\bigoplus_{j=1}^r E_{j,q_j} \to E_{0,p}.
\end{equation}
The section $K$ can serve as the kernel to define an $r$-multilinear integral operator $\mcal{T}_K : D(\mcal{T}_K) \subset \bigoplus_{j=1}^r\sect(E_j) \to \sect(E_0)$ with action on $(F_1,\ldots,F_r)\in D(\mcal{T}_K)$ given by
\begin{empheq}[box=\fbox]{equation}
    \mcal{T}_K[F_1,\ldots,F_r](p) = \int_{\mcal{M}_1\times\cdots\times\mcal{M}_r} K(p,q_1,\ldots,q_r)\big[ F_1(q_1), \ldots, F_r(q_r) \big] \volform_1(q_1) \wedge \cdots \wedge \volform_r(q_r). 
    \label{eqn:general_nonlinear_integral_operator}
\end{empheq}
This operator is linear when $r=1$.
When $r > 1$ and $E_1 = \cdots = E_r$, \eqref{eqn:general_nonlinear_integral_operator} can be used to define a nonlinear integral operator $\sect(E_1) \to \sect(E_0)$ with action $F \mapsto \mcal{T}_K[F,\ldots,F]$.

Given fiber-linear right $G$-actions $\Theta_j: E_j \times G \to E_j$, there is an induced fiber-linear right $G$-action $\Theta: E \times G \to E$ on $E$ defined by
\begin{equation}
    \Theta_g(K_{p,q_1,\ldots,q_r})\big[v_{1}, \ldots, v_{r}\big]
    = \Theta_{0,g} \left( K_{p,q_1,\ldots,q_r}\big[ \Theta_{1,g^{-1}}(v_{1}), \ldots, \Theta_{r,g^{-1}}(v_{r})\big]\right)
\end{equation}
for $K_{p,q_1,\ldots,q_r} \in E$ viewed as an $r$-multilinear map $E_{1,q_1}\oplus \cdots \oplus E_{r,q_r} \to E_{0,p}$ and $v_j \in E_{j,\theta_{j,g}(q_j)}$.
Sections $F_j \in \sect(E_j)$ transform according to
\begin{equation}
    \mcal{K}^{E_j}_{g} F_j = \Theta_{j,g^{-1}} \circ F_j \circ \theta_{j,g},
\end{equation}
with the section defining the integral kernel transforming according to
\begin{multline}
    \mcal{K}^{E}_g K(p,q_1,\ldots,q_r)[v_{q_1}, \ldots, v_{q_r}] = \\ 
    \Theta_{0,g^{-1}}\Big\{
    K\big(\theta_{0,g}(p),\theta_{1,g}(q_1),\ldots,\theta_{r,g}(q_r)\big)
    \big[ \Theta_{1,g}(v_{q_1}), \ldots, \Theta_{r,g}(v_{q_r}) \big] \Big\}.
    \label{eqn:multilinear_kernel_transformation}
\end{multline}
Using these transformation laws, we define equivariance for the integral operator as follows:
\begin{definition}\label{def:equivariant_multilinear_integral_operator}
    The integral operator in \eqref{eqn:general_nonlinear_integral_operator} is \textbf{equivariant} with respect to $g\in G$ if
    \begin{equation}
        \mcal{K}^{E_0}_{g} \mcal{T}_K\big[ \mcal{K}^{E_1}_{g^{-1}}F_1, \ldots, \mcal{K}^{E_r}_{g^{-1}}F_r \big]
        = \mcal{T}_K\big[ F_1, \ldots, F_r \big]
    \end{equation}
    for every $(F_1,\ldots,F_r) \in D(\mcal{T}_K)$.
\end{definition}
This definition is equivalent to the following condition on the integral kernel:
\begin{lemma}\label{lem:equivariant_multlinear_integral_kernel}
    The integral operator in \eqref{eqn:general_nonlinear_integral_operator} is equivariant with respect to $g \in G$ if and only if
    \begin{empheq}[box=\fbox]{equation}
    \mcal{K}_{g}\big( K \volform_1 \wedge \cdots \wedge V_r \big)
    := \mcal{K}^E_g(K) \ \theta_{1,g}^* \volform_1 \wedge \cdots \wedge \theta_{r,g}^*\volform_r
    = K \volform_1 \wedge \cdots \wedge \volform_r,
    \label{eqn:transformation_of_integral_kernels}
\end{empheq}
where $\mcal{K}_g^E$ is defined by \eqref{eqn:multilinear_kernel_transformation}.
A proof is available in Appendix~\ref{app:proofs_of_minor_results}.
\end{lemma}
We note that $\mcal{K}^{\Omega}_g (\volform_j) = \theta_{j,g}^* \volform_j$ is the natural transformation of the differential form $\volform_j \in \Omega^{d_j}(T^*\mcal{M}_j)$ (a covariant tensor field) described in Section~\ref{subsec:symmetry_of_tf}.
The Lie derivative of this action on volume forms is given by
\begin{equation}
    \mcal{L}^{\Omega}_{\xi} \volform 
    = \td \big( \hat{\theta}(\xi) \intmul \volform \big)
    = \diverg \hat{\theta}(\xi) \volform
\end{equation}
thanks to Cartan's magic formula and the definition of divergence (see \cite{Lee2013introduction}).
Therefore, differentiating \eqref{eqn:transformation_of_integral_kernels} along the curve $g(t) = \exp(t\xi)$ yields the Lie derivative
\begin{empheq}[box=\widefbox]{equation}
    \mcal{L}_{\xi} \big( K \volform_1 \wedge \cdots \wedge \volform_r \big)
    = \bigg[ \mcal{L}^E_{\xi}(K) + K \sum_{j=1}^r \diverg \hat{\theta}_j(\xi)\bigg] \volform_1 \wedge \cdots \wedge \volform_r.
    \label{eqn:Lie_derivative_for_integral_kernels}
\end{empheq}
Here, $\mcal{L}_{\xi}^E$ is the Lie derivative associated with \eqref{eqn:multilinear_kernel_transformation}.
For the integral operators discussed in Section~\ref{subsec:NNs_acting_on_fields}, the formulas derived here recover Eqs.~\ref{eqn:transformation_for_linear_integral_kernels_on_Rn}~and~\ref{eqn:Lie_derivative_for_linear_integral_kernels_on_Rn}.

\section{Invariant submanifolds and tangency}
\label{sec:submanifolds_and_tangency}
Studying the symmetries of smooth maps can be cast into a more general framework in which we study the symmetries of submanifolds.
Specifically, the symmetries of a smooth map $F: \mcal{M}_0 \to \mcal{M}_1$ between manifolds correspond to symmetries of its graph, $\graph(F)$, 
and the symmetries of a smooth section of a vector bundle $F\in C^{\infty}(\mcal{M},E)$ correspond to symmetries of its image, $\image(F)$ 
--- both of which are properly embedded submanifolds of $\mcal{M}_0 \times \mcal{M}_1$ and $E$, respectively.
We show that symmetries of 
a large class of submanifolds, including the above,
are revealed by checking whether the infinitesimal generators of the group action are tangent to the submanifold.
In this setting, the Lie derivative of $F\in C^{\infty}(\mcal{M},E)$ has a geometric interpretation as a projection of the infintesimal generator onto the tangent space of the image $\image(F)$, viewed as a submanifold of the bundle.

\subsection{Symmetry of submanifolds}
\label{subsec:detecting_symmetry_of_submanifolds}
In this section we study the infinitesimal conditions for a submanifold to be invariant under the action of a Lie group.
Suppose that $\mcal{N}$ is a manifold and $\theta: \mcal{N}\times G \to \mcal{N}$ is a right action of a Lie group $G$ on $\mcal{N}$.
Sometimes we denote this action by $p\cdot g = \theta(p,g)$ when there is no ambiguity.
Though our results also hold for left actions, as we discuss later in Remark~\ref{rem:left_actions_on_manifolds}, working with right actions is standard in this context and allows us to leverage results from \citet{Lee2013introduction} more naturally in our proofs.
Fixing $p\in\mcal{N}$, the orbit map of this action is denoted $\theta^{(p)}:G\to \mcal{N}$.
Fixing $g\in G$, the map $\theta_g:\mcal{N}\to \mcal{N}$ defined by $\theta_g: p \mapsto \theta(p,g)$ is a diffeomorphism with inverse $\theta_{g^{-1}}$.
\begin{definition}
    \label{def:group_invariance_of_subset}
    A subset $\mcal{M} \subset \mcal{N}$ is \textbf{$G$-invariant} if and only if $\theta(p, g) \in \mcal{M}$ for every $g\in G$ and $p\in\mcal{M}$.
\end{definition}
\noindent
Sometimes we will denote $\mcal{M} \cdot G = \{ \theta(p, g) \ : \ p \in \mcal{M}, \ g\in G \}$, in which case $G$-invariance of $\mcal{M}$ can be stated as $\mcal{M} \cdot G \subset \mcal{M}$.

We study the group invariance of submanifolds of the following type:
\begin{definition}
    \label{def:closed_slice}
    Let $\mcal{M}$ be a weakly embedded $m$-dimensional submanifold of an $n$-dimensional manifold $\mcal{N}$. 
    We say that $\mcal{M}$ is \textbf{arcwise-closed} if any smooth curve $\gamma:[a,b] \to \mcal{N}$ satisfying $\gamma((a,b)) \subset \mcal{M}$ must also satisfy $\gamma([a,b])\subset \mcal{M}$.
\end{definition}
\noindent
Submanifolds of this type include all properly embedded submanifolds of $\mcal{N}$ because properly embedded submanifolds are closed subsets (Proposition~5.5 in \cite{Lee2013introduction}).
More interestingly, we have the following:
\begin{proposition}
    \label{prop:foliations_have_arcwise_closed_leaves}
    The leaves of any (nonsingular) foliation of $\mcal{N}$ are arcwise-closed.
    We provide a proof in Appendix~\ref{app:proofs_of_minor_results}.
\end{proposition}
This means that the kinds of submanifolds we are considering include all possible Lie subgroups (\citet[Theorem~19.25]{Lee2013introduction}) as well as their orbits under free and proper group actions (\citet[Proposition~21.7]{Lee2013introduction}).
The leaves of singular foliations associated with integrable distributions of nonconstant rank (see \citet[Sections~3.18--25]{Kolar1993natural}) can fail to be arcwise-closed.
For example, the distribution spanned by the vector field $x\frac{\partial}{\partial x}$ on $\R$ has maximal integral manifolds $(-\infty, 0)$, $\{0\}$, and $(0,\infty)$ forming a singular foliation of $\R$.
Obviously, the leaves $(-\infty,0)$ and $(0,\infty)$ are not arcwise-closed.

Given a submanifold and a candidate group of transformations, the following theorem describes the largest connected Lie subgroup of symmetries of the submanifold.
Specifically, these symmetries can be identified by checking tangency conditions between infinitesimal generators and the submanifold.
\begin{theorem}
    \label{thm:symmetries_of_a_submanifold}
    Let $\mcal{M}$ be an immersed submanifold of $\mcal{N}$ and let $\theta: \mcal{N}\times G \to \mcal{N}$ be a right action of a Lie group $G$ on $\mcal{N}$ with infinitesimal generator $\hat{\theta}:\Lie(G) \to \vf(\mcal{N})$.
    Then
    \begin{equation} \label{eqn:submanifold_symmetry_Lie_subalgebra}
        \sym_G(\mcal{M}) = \big\{ \xi \in \Lie(G) \ : \ \hat{\theta}(\xi)_p \in T_p\mcal{M} \quad \forall p\in\mcal{M} \big\}
    \end{equation}
    is the Lie subalgebra of a unique connected Lie subgroup $H \subset G$.
    If $\mcal{M}$ is weakly-embedded and arcwise-closed in $\mcal{N}$, then this subgroup has the following properties:
    \begin{enumerate}[label=(\roman*)]
        \item $\mcal{M}\cdot H \subset \mcal{M}$
        \item If $\tilde{H}$ is a connected Lie subgroup of $G$ such that $\mcal{M} \cdot \tilde{H} \subset \mcal{M}$, then $\tilde{H} \subset H$.
    \end{enumerate}
    If $\mcal{M}$ is properly embedded in $\mcal{N}$ then $H$ is the identity component of the closed, properly embedded Lie subgroup
    \begin{equation}
        \Sym_G(\mcal{M}) = \{ g\in G \ : \ \mcal{M} \cdot g \subset \mcal{M} \}.
    \end{equation}
    A proof is provided in Appendix~\ref{app:symmetries_of_a_submanifold}.
\end{theorem}
\noindent Since the infinitesimal generator $\hat{\theta}$ is a linear map and $T_p\mcal{M}$ is a subspace of $T_p\mcal{N}$, the tangency conditions defining the Lie subalgebra \eqref{eqn:submanifold_symmetry_Lie_subalgebra} can be viewed as a set of linear equality constraints on the elements $\xi \in \Lie(G)$.
Hence, $\sym_G(\mcal{M})$ can be computed as the nullspace of a positive semidefinite operator on $\Lie(G)$, defined analogously to the case described earlier in Section~\ref{subsec:submanifolds_of_Rd}.

The following theorem provides necessary and sufficient conditions for arcwise-closed weakly-embedded submanifolds to be $G$-invariant.
These are generally nonlinear constraints on the submanifold, regarded as the zero section of its normal bundle under identification with a tubular neighborhood.
However, we will show in Section~\ref{subsec:Lie_derivative_as_projection} that these become linear constraints recovering the Lie derivative when the submanifold in question is the image of a section of a vector bundle and the group action is fiber-linear.
\begin{theorem}
    \label{thm:symmetry_conditions_for_submanifold}
    Let $\mcal{M}$ be an arcwise-closed weakly-embedded submanifold of $\mcal{N}$ and let $\theta: \mcal{N}\times G \to \mcal{N}$ be a right action of a Lie group $G$ on $\mcal{N}$ with infinitesimal generator $\hat{\theta}:\Lie(G) \to \vf(\mcal{N})$.
    Let $\xi_1, \ldots, \xi_q$ generate $\Lie(G)$ and let $g_1, \ldots, g_{n_G - 1}$ be elements from each non-identity component of $G$.
    Then $\mcal{M}$ is $G_0$-invariant if and only if
    \begin{equation}
        \hat{\theta}(\xi_i)_p \in T_p\mcal{M} \qquad \forall p\in\mcal{M}, \quad i=1, \ldots, q.
        \label{eqn:generator_manifold_tangency}
    \end{equation}
    If, in addition, we have $\mcal{M} \cdot g_j \subset \mcal{M}$ for every $j=1, \ldots, n_{G}-1$, then $\mcal{M}$ is $G$-invariant.
    A proof is provided in Appendix~\ref{app:symmetry_conditions_for_submanifold}.
\end{theorem}

\begin{remark}[Left actions]
    \label{rem:left_actions_on_manifolds}
    When the group $G$ acts on $\mcal{N}$ from the left according to $\theta^L:G\times\mcal{N} \to \mcal{N}$, we can always construct an equivalent right-action $\theta^R:\mcal{N}\times \mcal{N} \to \mcal{N}$ by setting
    $\theta^R(p, g) = \theta^L(g^{-1}, p)$.
    The corresponding infinitesimal generators are related by
        $\hat{\theta}^R = - \hat{\theta}^L$.
    Since $\hat{\theta}^L(\xi)_p \in T_p\mcal{M}$ if and only if $\hat{\theta}^R(\xi)_p \in T_p\mcal{M}$,  Theorems~\ref{thm:symmetry_conditions_for_submanifold}~and~\ref{thm:symmetries_of_a_submanifold} hold without modification for left $G$-actions.
\end{remark}

\subsection{The Lie derivative as a projection}
\label{subsec:Lie_derivative_as_projection}
We provide a geometric interpretation of the Lie derivative in \eqref{eqn:Lie_derivative} by expressing it in terms of a projection of the infinitesimal generator of the group action onto the tangent space of $\image(F)$ for smooth sections $F\in C^{\infty}(\mcal{M},E)$.
This allows us to connect the Lie derivative to the tangency conditions for symmetry of submanifolds presented in Section~\ref{subsec:detecting_symmetry_of_submanifolds}.

The Lie derivative $\mcal{L}_{\xi} F(p)$ lies in $E_p$, while $T_{F(p)}\image(F)$ is a subspace of $T_{F(p)} E$.
To relate quantities in these different spaces, the following lemma lifts each $E_p$ to a subspace of $T_{F(p)} E$.
\begin{lemma}
    \label{lem:injection_of_E_into_TE}
    For every smooth section $F\in C^{\infty}(\mcal{M},E)$ there is a well-defined injective vector bundle homomorphism $\imath_F : E \to T E$ that 
    is expressed in any local trivialization $\Phi : \pi^{-1}(\mcal{U}) \to \mcal{U} \times \R^k$ as
    \begin{equation}
    \begin{aligned}
        \D\Phi \circ \imath_F \circ \Phi^{-1} : \mcal{U} \times \R^k &\to T(\mcal{U} \times \R^k) \\
        (p, \vect{v}) &\mapsto (0, \vect{v})_{\Phi(F(p))}.
    \end{aligned}
    \label{eqn:injection_E_to_TE}
    \end{equation}
    We give a proof in Appendix~\ref{app:characterization_of_Lie_derivative}.
\end{lemma}
This is a special case of the \emph{vertical lift} of $E$ into the vertical bundle $V E = \{ v\in TE \ : \ \D \pi v = 0 \}$ described by \cite{Kolar1993natural} in Section~6.11.
The \emph{vertical projection} $\vpr_E: VE \to E$ provides a left-inverse satisfying $\vpr_E \circ \imath_F = \Id_E$.

The following result relates the Lie derivative to a projection via the vertical lifting.
\begin{theorem}
    \label{thm:characterization_of_Lie_derivative}
    Given a smooth section $F\in C^{\infty}(\mcal{M},E)$ and $p \in \mcal{M}$,
    the map $P_{F(p)}:=\D(F\circ \pi)(F(p)):T_{F(p)} E \to T_{F(p)} E$ is a linear projection onto $T_{F(p)}\image(F)$ and for every $\xi\in\Lie(G)$ we have
    \begin{equation}
        \imath_F \circ (\mcal{L}_{\xi} F) (p) = 
        -\big[ I - P_{F(p)} \big] \hat{\Theta}(\xi)_{F(p)}.
        \label{eqn:Lie_derivative_as_projection}
    \end{equation}
    We give a proof in Appendix~\ref{app:characterization_of_Lie_derivative}.
\end{theorem}
For the special case of smooth maps $F:\mcal{V} \to \mcal{W}$ between finite-dimensional vector spaces viewed as sections $x \mapsto (x,F(x))$ of the trivial bundle $\pi:\mcal{V}\times\mcal{W} \to \mcal{V}$, this theorem reproduces \eqref{eqn:Lie_derivative_as_projection_for_real_maps}.
The following corollary provides a link between our main results for sections of vector bundles and our main results for symmetries of submanifolds.
\begin{corollary}
\label{cor:properties_of_Lie_derivative}
For every smooth section $F\in C^{\infty}(\mcal{M},E)$, $\xi\in\Lie(G)$, and $p\in\mcal{M}$ we have
\begin{equation}
    (\mcal{L}_{\xi}F)(p) = 0 
    \qquad \Leftrightarrow \qquad 
    \hat{\Theta}(\xi)_{F(p)} \in T_{F(p)}\image(F).
\end{equation}
\end{corollary}
In particular, this means that for smooth sections,
Theorems~\ref{thm:equivariance_conditions_for_vb_section}~and~\ref{thm:symmetries_of_sections}
are special cases of
Theorems~\ref{thm:symmetry_conditions_for_submanifold}~and~\ref{thm:symmetries_of_a_submanifold}.

\section{Conclusion}
\label{sec:conclusion}

This paper provides a unified theoretical approach to enforce, discover, and promote symmetries in machine learning models. 
In particular, we provide theoretical foundations for Lie group symmetry in machine learning from a linear-algebraic viewpoint. 
This perspective unifies and generalizes several leading approaches in the literature, including approaches for incorporating and uncovering symmetries in neural networks and more general machine learning models. 
The central objects in this work are linear operators describing the finite and infinitesimal transformations of smooth sections of vector bundles with fiber-linear Lie group actions.
To make the paper accessible to a wide range of practitioners, Sections~\ref{sec:background_on_matrix_Lie_groups}--\ref{sec:numerical_experiments} deal with the special case where the machine learning models are built using smooth functions between vector spaces.
Our main results establish that the infinitesimal operators --- the Lie derivatives --- fully encode the connected subgroup of symmetries for sections of vector bundles (resp. functions between vector spaces).
In other words, the Lie derivatives encode symmetries that the machine learning models are equivariant with respect to. 

We illustrate that promoting and enforcing continuous symmetries in large classes of machine learning models are dual problems with respect to the bilinear structure of the Lie derivative.
Moreover, these ideas extend naturally to identify continuous symmetries of arbitrary submanifolds, recovering the Lie derivative when the submanifold is the image of a section of a vector bundle (resp., the graph of a function between vector spaces).
Using the fundamental operators, we also describe how symmetries can be promoted as inductive biases during training of machine learning models using convex penalties.
We also provide rigorous data-driven methods for discretizing and approximating the fundamental operators to accomplish the tasks of enforcing, promoting, and discovering symmetry.
We highlight that these tasks admit computational implementations using standard linear algebra and convex optimization.

We evaluate our new approach to promote symmetry using three numerical examples including symmetric polynomial recovery, extrapolatory recovery of a linear multi-spring-mass dynamical system, and training a new Equivariance-Promoting Neural Network (EPNN) to forecast the dynamics of a double-spring-pendulum system.
Our numerical results for polynomial recovery show that minimizing convex symmetry-promoting penalties can be used to recover highly symmetric polynomial functions using fewer samples than are required to determine the polynomial coefficients directly as the solution of a linear system.
This reduction in sample complexity becomes more pronounced in higher dimensions and with increasing symmetry of the underlying function being recovered.
Our numerical results on extrapolatory recovery of the linear spring-mass system illustrate how symmetry-promoting techniques can learn to extrapolate from restricted data, while common sparsity-promoting techniques can fail to do so.
The results obtained using EPNNs to model a double-spring pendulum system indicate that promoting symmetry can provide accuracy benefits in low-data regimes.
Smaller candidate groups containing correct symmetries were shown to provide the greatest benefits, with enough tolerance for incorrect symmetries in the candidate group to potentially allow the method to be useful in practice.
We note that very large candidate groups were not beneficial or harmful compared to EPNNs with no symmetry penalty.
Interestingly, EPNNs had significantly better performance than the baseline MLP in low-data regimes, even when no symmetry-promoting penalty was used.

The main limitation of our approach to promote symmetry is the need to choose a sufficiently small candidate group $G$ containing correct symmetries of the underlying system.
A second limitation is the significant computational cost to work with the discretized linear operators encoding the symmetries of a model. 
Another challenge and potentially significant limitation is to identify appropriate spaces of functions $\mcal{F}$ and model architectures capable of expressing non-trivial symmetric functions.
For example, it is possible that the only $G$-symmetric functions in $\mcal{F}$ are trivial, meaning that enforcing symmetry results in learning only trivial models.
One open question is whether our framework can be used in such cases to learn relaxed symmetries, as described by \cite{Wang2022approximately}.
In other words, it might be possible to find elements in $\mcal{F}$ that are nearly symmetric, and to bound the degree of asymmetry based on the quantities derived from the fundamental operators, such as their norms.
Finally, our reliance on the Lie algebra to study continuous symmetries also limits the ability of our proposed methods to account for partial symmetries, such as the invariance in classifying the characters ``Z'' and ``N'' to rotations by small angles, but not large angles.

In follow-up work, we aim to apply the proposed methods to a wider range of examples to provide detailed comparisons with state-of-the-art techniques and to assess performance trade-offs associated with problem dimension, candidate symmetry group, model size, amount of training data, and noise level.
Since our proposed methods remain to be tested in the more general setting of vector bundles, it will also be important to assess their practical usefulness by implementing, evaluating, and comparing them to existing techniques.
To enable these experiments, a major goal in future work must be to improve the computational efficiency and implementations of techniques to enforce, discover, and promote symmetry within our framework.
Work is currently underway to leverage sparse structure of the discretized operators in certain bases to enable use of efficient Krylov subspace algorithms and approximations based on matrix sketching \citep{halko2011finding, woodruff2014sketching}.
It will also be useful to identify efficient optimization algorithms for training symmetry-constrained or symmetry-regularized machine learning models.
Promising candidates include projected gradient descent, the Alternating Direction Method of Multipliers (ADMM) \citep{boyd2011distributed, fazel2013hankel}, variations of Bregman learning and mirror descent \citep{bungert2022bregman, beck2003mirror}, as well as Iteratively Reweighted Least Squares (IRLS) algorithms \citep{Mohan2012iterative}.
Using iterative reweighting could enable symmetry-promoting penalties to be based on non-convex Schatten $p$-``norms'' with $0 < p < 1$, potentially improving the recovery of underlying symmetry groups compared to the nuclear norm where $p=1$.

Additional avenues for future work include applying the proposed techniques to continuum problems and partial differential equations though principled implementation in neural operators and through jet bundle prolongation \citep{Olver1986applications}.
Combining analogues of our proposed methods with techniques using the weak formulation proposed by \cite{Messenger2021weak, Reinbold2020using} could also provide robust ways to identify symmetric PDEs in the presence of high noise and limited training data.
It will also be important to study the perturbative effects of noisy data in algorithms to discover and promote symmetry, with the goal of understanding the effects of problem dimension, noise level, and amount of data.
Another area of future work could be to
build on the work of \cite{peitz2023partial, steyert2022uncovering} by studying symmetry in the setting of Koopman operators for dynamical systems.
To do this, one might follow the program set forth by \citet{colbrook2023mpedmd}, where the measure preserving property of certain dynamical systems is exploited to enhance the Extended Dynamic Mode Decomposition (EDMD) algorithm of \cite{Williams2015jnls}.

\section*{Acknowledgements}
The authors acknowledge support from the National Science Foundation AI Institute in Dynamic Systems
(grant number 2112085).  SLB acknowledges support from the Army Research Office (ARO W911NF-19-1-0045) and the Boeing Company. 
The authors would also like to acknowledge valuable discussions with Tess Smidt and Matthew Colbrook.

\begin{spacing}{1}
\setlength{\bibsep}{6.pt}
\bibliography{references}
\end{spacing}

\appendix

\addtocontents{toc}{\protect\setcounter{tocdepth}{0}}

\section{Proofs of minor results}
\label{app:proofs_of_minor_results}

\begin{proof}[Proposition~\ref{prop:symmetries_of_linear_integral_operator}]
    \label{proof:symmetries_of_linear_integral_operator}
    Obviously, if $\mcal{K}_g K = K$ then $\mcal{K}_{g}^{(\mcal{W})} \circ \mcal{T}_K \circ \mcal{K}_{g^{-1}}^{(\mcal{V})} = \mcal{T}_K$.
    On the other hand, suppose that $\mcal{K}_g K(x_0,y_0) \neq K(x_0,y_0)$ for some $(x_0,y_0) \in \R^n \times \R^m$.
    Hence, there are vectors $v\in\mcal{V}$ and $w \in \mcal{W}^*$ such that $\langle w, \ \mcal{K}_g K(x_0,y_0) v - K(x_0,y_0) v \rangle > 0$.
    This remains true for all $y$ in a neighborhood $\mcal{U}$ of $y_0$ by continuity of $K$ and $\mcal{K}_g K$.
    Letting $F(x) = v \varphi(x)$ where $\varphi$ is a smooth, nonnegative, function with $\varphi(y_0) > 0$ and support in $\mcal{U}$, we obtain
    \begin{equation}
        \big\langle w, \ \mcal{K}_{g}^{(\mcal{W})} \circ \mcal{T}_K \circ \mcal{K}_{g^{-1}}^{(\mcal{V})} F(x) - \mcal{T}_K F(x)\big\rangle
        = \int_{\R^m} \big\langle w, \ \mcal{K}_g K(x,y) v - K(x,y) v \big\rangle \varphi(y) \td y > 0,
    \end{equation}
    meaning $\mcal{K}_{g}^{(\mcal{W})} \circ \mcal{T}_K \circ \mcal{K}_{g^{-1}}^{(\mcal{V})} \neq \mcal{T}_K$.
    Therefore, $\mcal{K}_g K = K$ if and only if $\mcal{K}_{g}^{(\mcal{W})} \circ \mcal{T}_K \circ \mcal{K}_{g^{-1}}^{(\mcal{V})} = \mcal{T}_K$.
\end{proof}

We use the following lemma in the proof of Proposition~\ref{prop:manifold_nullspace_convergence}.
\begin{lemma}\label{lem:matrix_nullspace_convergence}
    Suppose that $S_{m} \to S$ is a convergent sequence of matrices and $\Null(S) \subset \Null(S_n) \subset \Null(S_m)$ when $n \geq m$. 
    Then there is an integer $M_0$ such that for every $m \geq M_0$ we have $\Null(S_m) = \Null(S)$.
    We provide a proof in Appendix~\ref{app:proofs_of_minor_results}.
\end{lemma}
\begin{proof}
    Since the sequence of matrices acts on a finite-dimensional space, $\dim\Null(S_m)$ is a monotone bounded sequence of integers.
    Therefore, there exists an integer $M_0$ such that $\dim\Null(S_m) = \dim\Null(S_{M_0})$ for every $m \geq M_0$.
    Since $\Null(S_m) \subset \Null(S_{M_0})$, we must have $\Null(S_m) = \Null(S_{M_0})$ for every $m \geq M_0$.
    Since $\Null(S) \subset \Null(S_{M_0})$ by assumption, it remains to show the reverse containment. 
    Suppose $\xi \in \Null(S_{M_0})$, then
    \begin{equation}
        S \xi = \lim_{m \to \infty} S_m \xi = 0
    \end{equation}
    meaning that $\Null(S_{M_0}) \subset \Null(S)$.
\end{proof}

\begin{proof}[Proposition~\ref{prop:manifold_nullspace_convergence}]
    By the Cauchy-Schwarz inequality our assumption means that $\langle \eta, S_{\mcal{M}} \xi \rangle < \infty$ for every $\eta, \xi \in \Lie(G)$.
    Let $\xi_1, \ldots, \xi_{\dim G}$ be a basis for $\Lie(G)$.
    By the strong law of large numbers, specifically Theorem~7.7 in \cite{Koralov2012theory}, we have 
    \begin{equation}
        \big\langle \xi_j, \ S_{m} \xi_k \big\rangle_{\Lie(G)} \to \langle \xi_j, S_{\mcal{M}} \xi_k \rangle 
    \end{equation}
    for every $j,k$ almost surely.
    Consequently, $S_m \to S_{\mcal{M}}$ almost surely.
    By nonnegativity of each term in the sum defining $\big\langle \xi, \ S_{m} \xi \big\rangle_{\Lie(G)}$, it follows that $\Null(S_n) \subset \Null(S_m)$ when $n \geq m$.
    Moreover, if $\xi \in \Null(S_{\mcal{M}})$ then it follows from the continuity of $z\mapsto (I-P_z) \hat{\theta}(\xi)_{z}$ that $(I-P_z) \hat{\theta}(\xi)_{z} = 0$ for every $z\in\mcal{M}$.
    Hence, $\xi \in \Null(S_m)$ for every $m$.
    Therefore, $S_m$ and $S_{\mcal{M}}$ obey the hypotheses of Lemma~\ref{lem:matrix_nullspace_convergence} almost surely and the conclusion follows.
\end{proof}

\begin{proof}[Proposition~\ref{prop:symmetries_of_basic_functions}]
    Consider the function $\tilde{F}_{\text{rad}} : \R^n \to \R^r$ defined by 
    \begin{equation}
        \tilde{F}_{\text{rad}}(x) = (\| x - c_1 \|_2^2,\ \ldots,\  \| x - c_r \|_2^2)
    \end{equation}
    with standard action of $\SE(n)$ on its domain and the trivial action on its codomain.
    The symmetries of $\tilde{F}_{\text{rad}}$ are shared by $F_{\text{rad}}$.
    By Theorem~\ref{thm:symmetries_of_a_map}, the Lie algebra of $\tilde{F}_{\text{rad}}$'s symmetry group is characterized by
    \begin{equation}
        \xi = 
        \begin{bmatrix}
                S & v \\
                0 & 0
        \end{bmatrix} \in \sym_{\SE(n)}\tilde{F}_{\text{rad}} 
        \quad \Leftrightarrow \quad
        0 
        = \mcal{L}_{\xi} \tilde{F}_{\text{rad}}(x) 
        = \frac{\partial \tilde{F}_{\text{rad}}(x)}{\partial x} (S x + v) \quad \forall x\in\R^n.
    \end{equation}
    This means the generators $\xi$ are characterized by the equations
    \begin{equation}
        0 = (x - c_i)^T(S x + v) = x^T S x - c_i^T S x + x^T v - c_i^T v, \qquad \forall x \in \R^n \quad i = 1, \ldots, r.
    \end{equation}
    Since $\xi \in \se(n)$, we have $S^T = -S$, called \emph{skew symmetry}, giving $x^T S x = 0$.
    The above is satisfied if and only if
    \begin{equation}
        S c_i = - v,
    \end{equation}
    which automatically yields $c_i^T v = - c_i^T S c_i = 0$.
    Therefore, $\sym_{\SE(n)} \tilde{F}_{\text{rad}} = \mathfrak{g}_{\text{rad}} \subset \sym_{\SE(n)} F_{\text{rad}}$.
    To determine the dimension of the symmetry group, we observe that $S$ must satisfy 
    \begin{equation}
        S (c_2 - c_1) = \cdots = S (c_r - c_1) = 0,
        \label{eqn:generators_for_function_of_radii}
    \end{equation}
    and any such $S$ uniquely determines $v = -S c_1$.
    Therefore, the dimension of $\mathfrak{g}_{\text{rad}}$ equals the dimension of the space of skew-symmetric matrices satisfying \eqref{eqn:generators_for_function_of_radii}.
    Let the columns of $W = \begin{bmatrix} W_1 & W_2 \end{bmatrix}$ form an orthonormal basis for $\R^n$ with the $r-1$ columns of $W_1$ being a basis for $\vspan\{(c_2 - c_1), \ldots, (c_r - c_1) \}$.
    The above constraints, together with skew-symmetry, mean that $S$ takes the form
    \begin{equation}
        S = 
        \begin{bmatrix} W_1 & W_2 \end{bmatrix}
        \begin{bmatrix}
            0 & 0 \\
            0 & \tilde{S}
        \end{bmatrix}
        \begin{bmatrix} W_1^T \\ W_2^T \end{bmatrix},
    \end{equation}
    where $\tilde{S}$ is an $(n-r+1) \times (n-r+1)$ matrix skew-symmetric matrix.
    Therefore, the dimension of $\mathfrak{g}_{\text{rad}}$ equals the dimension of the space of $(n-r+1) \times (n-r+1)$ skew-symmetric matrices, which is $\frac{1}{2} (n-r) (n-r+1)$.

    The argument for $F_{\text{lin}}$ is similar, with the symmetries of
    \begin{equation}
        \tilde{F}_{\text{lin}}(x) = \big( u_1^T x,\ \ldots,\ u_r^T x \big) = U^T x
    \end{equation}
    also being symmetries of $F_{\text{lin}}$.
    The condition $0 = \mcal{L}_{\xi} \tilde{F}_{\text{lin}}$ is equivalent to
    \begin{equation}
        U^T S x + U^T v = 0 \qquad \forall x\in\R^n,
    \end{equation}
    which occurs if and only if $U^T S = 0$ and $U^T v = 0$.
    This immediately yields $\sym_{\SE(n)} \tilde{F}_{\text{lin}} = \mathfrak{g}_{\text{lin}} \subset \sym_{\SE(n)} F_{\text{lin}}$.
    Per our earlier argument, the skew-symmetric matrices $S$ satisfying
    \begin{equation}
        S u_1 = \cdots = S u_r = 0
    \end{equation}
    form a vector space with dimension $\frac{1}{2}(n-r)(n-r-1)$.
    The subspace of vectors $v \in\R^n$ satisfying $U^T v = 0$ is $(n-r)$-dimensional.
    Adding these gives the dimension of $\mathfrak{g}_{\text{lin}}$, which is $\frac{1}{2}(n-r)(n-r+1)$.
    
    Suppose there exists a polynomial $\varphi_0$ with degree $\leq d$ such that $\sym_{\SE(n)}(F_{\text{rad}}) = \mathfrak{g}_{\text{rad}}$ when $\varphi_{\text{rad}} = \varphi_0$.
    Let $\mfrak{g}_{\perp}$ be a complementary subspace to $\mfrak{g}_{\text{rad}}$ in $\se(n)$, that is, $\mfrak{g}_{\perp} \oplus \mfrak{g}_{\text{rad}} = \se(n)$.
    We observe that $\sym_G(F_{\text{rad}}) \neq \mfrak{g}_{\text{rad}}$ if and only if there is a nonzero $\xi_{\perp} \in \mfrak{g}_{\perp}$ satisfying $\mcal{L}_{\xi_{\perp}} F_{\text{rad}} = 0$.
    The ``if'' part of this claim is obvious.
    The ``only if'' part follows from the fact that $\sym_G(F_{\text{rad}}) \neq \mfrak{g}_{\text{rad}}$ means that $\mcal{L}_{\xi} F_{\text{rad}} = 0$ for some nonzero $\xi \notin \mfrak{g}_{\text{rad}}$.
    Using the direct-sum decomposition of $\se(n)$, there are unique $\xi_{\perp} \in \mfrak{g}_{\perp}$ and $\xi_{a} \in \mfrak{g}_{a}$ such that $\xi = \xi_{\perp} + \xi_{a}$, yielding
    \begin{equation}
        0 = \mcal{L}_{\xi} F_{\text{rad}} = \mcal{L}_{\xi_{\perp}} F_{\text{rad}} + \mcal{L}_{\xi_{a}} F_{\text{rad}} = \mcal{L}_{\xi_{\perp}} F_{\text{rad}}.
    \end{equation}
    Moreover, $\xi_{\perp} \neq 0$ because $\xi \notin \mfrak{g}_{a}$.
    Letting $\xi_1, \ldots, \xi_D$ form a basis for $\mfrak{g}_{\perp}$, we consider the $D \times D$ Gram matrix $\mat{G}(F_{\text{rad}})$ with entries
    \begin{equation}
        \big[\mat{G}(F_{\text{rad}})\big]_{i,j} = \int_{[0,1]^n} \mcal{L}_{\xi_i} F_{\text{rad}}(x) \mcal{L}_{\xi_j} F_{\text{rad}}(x) \ \td x.
    \end{equation}
    This matrix is singular if and only if there is a nonzero $\xi_{\perp} \in \mfrak{g}_{\perp}$ satisfying $\mcal{L}_{\xi_{\perp}} F_{\text{rad}}(x) = 0$ for every $x$ in the cube $[0,1]^n$.
    Since 
    \begin{equation}
        \mcal{L}_{\xi_{\perp}} F_{\text{rad}}(x) = \frac{\partial F_{\text{rad}}(x)}{\partial x}(S x + v), \qquad 
        \xi_{\perp} = 
        \begin{bmatrix}
            S & v \\
            0 & 0
        \end{bmatrix},
    \end{equation}
    is a polynomial function of $x$, it vanishes in the cube if and only if it vanishes everywhere.
    Hence $\mat{G}(F_{\text{rad}})$ is singular if and only if $\sym_G(F_{\text{rad}}) \neq \mfrak{g}_{\text{rad}}$.
    Letting $\vect{c}$ denote the vector of coefficients defining $\varphi_{\text{rad}}$ in a basis for the polynomials of degree $\leq d$ on $\R^r$, we observe that
    \begin{equation}
        f : \vect{c} \mapsto \det\big( \mat{G}(F_{\text{rad}}) \big)
    \end{equation}
    is a polynomial function of $\vect{c}$.
    The set of polynomials $\varphi_{\text{rad}}$ with degree $\leq d$ for which $\sym_G(F_{\text{rad}}) \neq \mfrak{g}_{\text{rad}}$ corresponds to the zero level set of $f$, i.e., those $\vect{c}$ such that $f(\vect{c}) = 0$.
    Obviously, $f(0) = 0$, and taking the coefficients $\vect{c}_0$ corresponding to $\varphi_0$ gives $f(\vect{c}_0) \neq 0$, meaning $f$ is a nonconstant polynomial.
    Since each level set of a nonconstant polynomial is a set of measure zero (\cite{Caron2005zero}), it follows that the zero level set of $f$ has measure zero.
    Precisely the same argument works for $\varphi_{\text{lin}}$, $F_{\text{lin}}$, and $\mfrak{g}_{\text{lin}}$.
\end{proof}

\begin{proof}[Corollary~\ref{cor:symmetries_of_polynomial_Fa}]
    By Proposition~\ref{prop:symmetries_of_basic_functions}, it suffices to find a degree-$r$ polynomial $\varphi_{\text{rad}}$ such that $\sym_G(F_{\text{rad}}) \subset \mfrak{g}_{\text{rad}}$.
    We choose $\varphi_{\text{rad}}(z_1, \ldots, z_r) = z_1 + z_2^2 + \cdots + z_r^r$, giving
    \begin{equation}
        F_{\text{rad}}(x) = \sum_{k=1}^r \| x - c_k \|_2^{2k}.
    \end{equation}
    If $\xi = \begin{bmatrix} S & v \\ 0 & 0 \end{bmatrix} \in \se(n)$ generates a symmetry of $F_{\text{lin}}$ then
    \begin{equation}
        \label{eqn:special_Fa_Lie_deriv}
        0 = \mcal{L}_{\xi} F_{\text{rad}}(x) = \sum_{k=1}^r k\| x - c_k \|_2^{2(k-1)} \underbrace{(x - c_k)^T (S x + v)}_{x^T(S c_k + v) - c_k^T v}
        \qquad \forall x \in \R^n.
    \end{equation}
    The terms in this expression with highest degree in $x$ must vanish, yielding
    \begin{equation}
        0 = r \| x \|_2^{2(r-1)}x^T(S c_r + v) \qquad \forall x \in \R^n.
    \end{equation}
    This implies that $S c_r + v = 0$.
    Proceeding inductively, suppose that $S c_l + v = 0$ for every $l > k$.
    Then, vanishing the highest-degree term in \eqref{eqn:special_Fa_Lie_deriv} gives
    \begin{equation}
        0 = k \| x \|_2^{2(k-1)}x^T(S c_k + v) \qquad \forall x \in \R^n,
    \end{equation}
    implying that $S c_k + v = 0$.
    It follows that 
    \begin{equation}
        S c_k + v = 0 \qquad \forall k = 1, \ldots, r,
    \end{equation}
    by induction, meaning that $\xi \in \mfrak{g}_{\text{rad}}$.
    Hence, $\sym_G(F_{\text{rad}}) \subset \mfrak{g}_{\text{rad}}$, which completes the proof.
\end{proof}

\begin{proof}[Corollary~\ref{cor:symmetries_of_polynomial_Fb}]
    By Proposition~\ref{prop:symmetries_of_basic_functions}, it suffices to find a quadratic polynomial $\varphi_{\text{lin}}$ such that $\sym_G(F_{\text{lin}}) \subset \mfrak{g}_{\text{lin}}$.
    Letting $D = \diag[ 1, 2, \ldots, r ]$ and $U = \begin{bmatrix} u_1 & \cdots & u_r\end{bmatrix}$, consider the quadratic function $\varphi_{\text{lin}}(z) = \frac{1}{2} z^T D z$, giving
    \begin{equation}
        F_{\text{lin}}(x) = \frac{1}{2} x^T U D U^T x.
    \end{equation}
    If $\xi = \begin{bmatrix} S & v \\ 0 & 0 \end{bmatrix} \in \se(n)$ generates a symmetry of $F_{\text{lin}}$ then
    \begin{equation}
        0 = \mcal{L}_{\xi} F_{\text{lin}}(x) = x^T U D U^T S x + x^T U D U^T v \qquad \forall x \in \R^n.
    \end{equation}
    Differentiating the above with respect to $x$ at $x = 0$ yields $U D U^T v = 0$, which, because $U D$ is injective, means that $U^T v = 0$.
    The fact that $x^T U D U^T S x = 0$ for every $x$ means that $U D U^T S + S^T U D U^T = 0$, i.e.,
    \begin{equation}
        U D U^T S = S U D U^T.
    \end{equation}
    Letting the columns of $U_{\perp}$ span the orthogonal complement to columns of $U$ and expressing 
    \begin{equation}
        S = U \tilde{S}_{1,1} U^T + U \tilde{S}_{1,2} U_{\perp}^T + U_{\perp} \tilde{S}_{2,1} U^T + U_{\perp} \tilde{S}_{2,2} U_{\perp}^T,
    \end{equation}
    the above commutation relation with $U D U$ gives
    \begin{equation}
        U D \tilde{S}_{1,1} U^T + U D \tilde{S}_{1,2} U_{\perp}^T 
        = U \tilde{S}_{1,1} D U^T + U_{\perp} \tilde{S}_{2,1} D U^T.
    \end{equation}
    Multiplying on left and right by combinations of $U^T$ or $U_{\perp}^T$ and $U$ or $U_{\perp}$ extracts the relations
    \begin{equation}
        D \tilde{S}_{1,1} = \tilde{S}_{1,1} D, 
        \qquad D \tilde{S}_{1,2} = 0, 
        \qquad \tilde{S}_{2,1} D = 0.
    \end{equation}
    Since $D$ is invertible, we must have $\tilde{S}_{1,2} = 0$ and $\tilde{S}_{2,1} = 0$.
    Since $S^T = - S$, we must also have $\tilde{S}_{1,1}^T = - \tilde{S}_{1,1}$, meaning that its diagonal entries are identically zero.
    Considering the $(j,k)$ element of $\tilde{S}_{1,1}$ with $j\neq k$, we have
    \begin{equation}
        j [\tilde{S}_{1,1}]_{j,k} 
        = [D \tilde{S}_{1,1}]_{j,k}
        = [\tilde{S}_{1,1} D]_{j,k}
        = k [\tilde{S}_{1,1}]_{j,k},
    \end{equation}
    meaning that $[\tilde{S}_{1,1}]_{j,k} = 0$.
    Therefore, only $\tilde{S}_{2,2}$ can be nonzero, which gives $S U = 0$.
    Combined with the fact that $U^T v = 0$, we conclude that $\sym_G(F_{\text{lin}}) \subset \mfrak{g}_{\text{lin}}$, completing the proof.
\end{proof}

\begin{proof}[Proposition~\ref{prop:intersection_subalgebra}]
    \label{proof:intersection_subalgebra}
    As an intersection of closed subgroups, $H := \bigcap_{l=1}^L \Sym_G\big(F^{(l)}\big)$ is a closed subgroup of $G$.
    By the closed subgroup theorem (see Theorem~20.12 in \citet{Lee2013introduction}), $H$ is an embedded Lie subgroup, whose Lie subalgebra we denote by $\mathfrak{h}$.
    If $\xi \in \mathfrak{h}$ then $\exp(t\xi)\in \Sym_G\big(F^{(l)}\big)$ for all $t\in\R$ and every $l=1,\ldots,L$.
    Differentiating $\mcal{K}_{\exp(t\xi)} F^{(l)} = F^{(l)}$ at $t = 0$ proves that $\mcal{L}_{\xi} F^{(l)} = 0$, i.e., $\xi\in\sym_G(F^{(l)})$ by Theorem~\ref{thm:symmetries_of_a_map}.
    Conversely, if $\mcal{L}_{\xi} F^{(l)} = 0$ for every $l=1,\ldots,L$, then by Theorem~\ref{thm:symmetries_of_a_map}, $\exp(t\xi) \in H$.
    Since $H$ is a Lie subgroup, differentiating $\exp(t\xi)$ at $t=0$ proves that $\xi \in \mathfrak{h}$.
\end{proof}

\begin{proof}[Proposition~\ref{prop:Monte_Carlo_eventually_gives_an_inner_product}]
    \label{proof:Monte_Carlo_eventually_gives_an_inner_product}
    Let $f_1, \ldots, f_N$ be a basis for $\mcal{F}'$.
    Consider the sequence of Gram matrices $\mat{G}_M$ with entries
    \begin{equation}
        [\mat{G}_M]_{i,j} = \left\langle f_i, \ f_j \right\rangle_{L^2(\mu_M)}.
    \end{equation}
    It suffices to show that $\mat{G}_M$ is positive-definite for sufficiently large $M$.
    Since the $L^2(\mu)$ inner product is positive-definite on $\mcal{F}'$, it follows that the Gram matrix $\mat{G}$ with entries
    \begin{equation}
        [\mat{G}]_{i,j} = \left\langle f_i, \ f_j \right\rangle_{L^2(\mu)}
    \end{equation}
    is positive-definite.
    Hence, its smallest eigenvalue $\lambda_{\text{min}}(\mat{G})$ is positive.
    Since the ordered eigenvalues of symmetric matrices are continuous with respect to their entries (see Corollary~4.3.15 in \cite{Horn2013matrix}) and $[\mat{G}_M]_{i,j} \to [\mat{G}]_{i,j}$ for all $1\leq i,j\leq N$ by assumption, we have $\lambda_{\text{min}}(\mat{G}_M) \to \lambda_{\text{min}}(\mat{G})$ as $M\to\infty$.
    Therefore, there is an $M_0$ so that for every $M \geq M_0$ we have $\lambda_{\text{min}}(\mat{G}_M) > 0$, i.e., $\mat{G}_M$ is positive-definite.
\end{proof}

\begin{proof}[Lemma~\ref{lem:equivariant_multlinear_integral_kernel}]
Using the fact that the integral is invariant under pullbacks by diffeomorphisms, we can express the left-hand-side of the equivariance condition in Definition~\ref{def:equivariant_multilinear_integral_operator} as
\begin{multline}
    \mcal{K}_{0,g} \mcal{T}_K\big[ \mcal{K}_{1,g^{-1}}F_1, \ldots, \mcal{K}_{r,g^{-1}}F_r \big](p)
    = \\
    \Theta_{0,g^{-1}}\Bigg\{
    \int_{\mcal{M}_1\times\cdots\times\mcal{M}_r} K(\theta_{0,g}(p),q_1,\ldots,q_r)
    \big[ \Theta_{1,g}\circ F_1\circ \theta_{1,g^{-1}}(q_1), \ldots, \Theta_{r,g}\circ F_r\circ \theta_{r,g^{-1}}(q_r) \big] \\
    \volform_1(q_1) \wedge \cdots \wedge \volform_r(q_r). 
    \Bigg\} = \\
    \Theta_{0,g^{-1}}\Bigg\{
    \int_{\mcal{M}_1\times\cdots\times\mcal{M}_r} K(\theta_{0,g}(p),\theta_{1,g}(q_1),\ldots,\theta_{r,g}(q_r))
    \big[ \Theta_{1,g}\circ F_1(q_1), \ldots, \Theta_{r,g}\circ F_r(q_r) \big] \\
    \theta_{1,g}^* \volform_1(q_1) \wedge \cdots \wedge \theta_{r,g}^*\volform_r(q_r). 
    \Bigg\} = \\
    \int_{\mcal{M}_1\times\cdots\times\mcal{M}_r} 
    \mcal{K}^E_{g} K (p, q_1, \ldots, q_r)\big[ F_1(q_1), \ldots, F_r(q_r) \big]
    \theta_{1,g}^* \volform_1(q_1) \wedge \cdots \wedge \theta_{r,g}^*\volform_r(q_r).
\end{multline}
Hence, by comparing the integrand to \eqref{eqn:general_nonlinear_integral_operator}, it is clear that \eqref{eqn:transformation_of_integral_kernels} implies that $\mcal{T}_K$ is equivariant in the sense of Definition~\ref{def:equivariant_multilinear_integral_operator}.
Conversely, if $\mcal{T}_K$ is $g$-equivariant, then
\begin{equation}
    \int_{\mcal{M}_1\times\cdots\times\mcal{M}_r} 
    \mcal{K}^E_{g} K \big[ F_1, \ldots, F_r \big]
    \theta_{1,g}^* \volform_1 \wedge \cdots \wedge \theta_{r,g}^*\volform_r
    = \int_{\mcal{M}_1\times\cdots\times\mcal{M}_r} 
    K \big[ F_1, \ldots, F_r \big]
    \volform_1 \wedge \cdots \wedge \volform_r
\end{equation}
holds for every $(F_1, \ldots, F_r)\in D(\mcal{T}_K)$.
Since the domain contains all smooth, compactly-supported fields $(F_1, \ldots, F_r)$, it follows that \eqref{eqn:transformation_of_integral_kernels} holds. 
\end{proof}

\begin{proof}[Proposition~\ref{prop:foliations_have_arcwise_closed_leaves}]
    Consider a leaf $\mcal{M}$ of an $m$-dimensional foliation on the $n$-dimensional manifold $\mcal{N}$ and let $\gamma:[a,b] \to \mcal{N}$ be a smooth curve satisfying $\gamma((a,b)) \subset \mcal{M}$.
    First, it is clear that $\mcal{M}$ is a weakly embedded submanifold of $\mcal{N}$ since $\mcal{M}$ is an integral manifold of an involutive distribution (\citet[Proposition~19.19]{Lee2013introduction}) and the local structure theorem for integral manifolds (\citet[Proposition~19.16]{Lee2013introduction}) shows that they are weakly embedded.
    
    By continuity of $\gamma$, any neighborhood of $\gamma(b)$ in $\mcal{N}$ must have nonempty intersection with $\mcal{M}$.
    By definition of a foliation (see \cite{Lee2013introduction}), there is a coordinate chart $(\mcal{U}, \vect{x})$ for $\mcal{N}$ with $\gamma(b) \in \mcal{U}$ such that $\vect{x}(\mcal{U})$ is a coordinate-aligned cube in $\R^n$ and $\mcal{M} \cap \mcal{U}$ consists of countably many slices of the form $x^{m+1} = c^{m+1}, \ldots, x^{n} = c^n$ for constants $c^{m+1}, \ldots, c^n$.
    Since $\gamma$ is continuous, there is a $\delta > 0$ so that $\gamma((b-\delta, b]) \subset \mcal{U}$, and in particular, $\gamma((b-\delta, b)) \subset \mcal{M} \cap \mcal{U}$.
    By continuity of $\gamma$, there are constants $c^{m+1}, \ldots, c^n$ such that $x^{i}(\gamma(t)) = c^{i}$ for every $i=m+1,\ldots, n$ and $t \in (b-\delta, b)$.
    Hence, we have
    \begin{equation}
        x^i(\gamma(b)) = \lim_{t\to b} x^i(\gamma(t)) = c^i, \qquad i=m+1, \ldots, n,
    \end{equation}
    meaning that $\gamma(b) \in \mcal{M}$.
    An analogous argument shows that $\gamma(a) \in \mcal{M}$, completing the proof that $\mcal{M}$ is arcwise-closed.
\end{proof}

\section{Proof of Proposition~\ref{prop:generic_polynomial_map_sampling}}
\label{app:generic_polynomial_map_sampling}

Our proof relies on the following lemma:
\begin{lemma}
    \label{lem:generic_polynomial_sampling}
    Let $\mcal{P}$ denote a finite-dimensional vector space of polynomials $\R^m \to \R$.
    If $M \geq \dim (\mcal{P})$ then the evaluation map $T:\mcal{P} \to \R^M$ defined by
    \begin{equation}
        T_{(x_1, \ldots, x_M)}: P \mapsto (P(x_1), \ldots, P(x_M))
    \end{equation}
    is injective for almost every $(x_1, \ldots, x_M) \in (\R^m)^M$ with respect to Lebesgue measure.
\end{lemma}
\begin{proof}
    Letting $M_0 = \dim(\mcal{P})$ and
    choosing a basis $P_1, \ldots, P_{M_0}$ for $\mcal{P}$, injectivity of $T_{(x_1, \ldots, x_M)}$ is equivalent to injectivity of the $M\times M_0$ matrix 
    \begin{equation}
        \mat{T}_{(x_1, \ldots, x_M)}
        = \begin{bmatrix}
        P_1(x_1) & \cdots & P_{M_0}(x_1) \\
        \vdots & \ddots & \vdots \\
        P_1(x_M) & \cdots & P_{M_0}(x_M)
        \end{bmatrix}.
    \end{equation}
    Finally, this is equivalent to
    \begin{equation}
        \varphi(x_1, \ldots, x_M) = \det\big( (\mat{T}_{(x_1, \ldots, x_M)})^T \mat{T}_{(x_1, \ldots, x_M)} \big)
    \end{equation}
    taking a nonzero value.
    We observe that $\varphi$ is a polynomial on the Euclidean space $(\R^{m})^M$.
    
    Suppose there exists a set of points $(\bar{x}_1, \ldots, \bar{x}_M)\in(\R^m)^M$ such that $T_{(\bar{x}_1, \ldots, \bar{x}_M)}$ is injective.
    Then for this set $\varphi(\bar{x}_1, \ldots, \bar{x}_M) \neq 0$.
    Obviously, $\varphi(0, \ldots, 0) = 0$, meaning that $\varphi$ cannot be constant.
    Thanks to the main result in \cite{Caron2005zero}, this means that each level set of $\varphi$ has zero Lebesgue measure in $(\R^{m})^M$. 
    In particular, the level set $\varphi^{-1}(0)$, consisting of those $x_1, \ldots, x_M$ for which $T_{(x_1, \ldots, x_M)}$ fails to be injective, has zero Lebesgue measure.
    Therefore, it suffices to prove that there exists $(\bar{x}_1, \ldots, \bar{x}_M)\in(\R^m)^M$ such that $T_{(\bar{x}_1, \ldots, \bar{x}_M)}$ is injective.
    We do this by induction.
    
    It is clear that there exists $\bar{x}_1$ so that the $1\times 1$ matrix
    \begin{equation}
        \mat{T}_1 = \begin{bmatrix}
            P_1(\bar{x}_1)
        \end{bmatrix}
    \end{equation}
    has full rank since $P_1$ cannot be the zero polynomial.
    Proceeding by induction, we assume that there exists $\bar{x}_1, \ldots, \bar{x}_s$ so that
    \begin{equation}
        \mat{T}_s = \begin{bmatrix}
            P_1(\bar{x}_1) & \cdots & P_{s}(\bar{x}_1) \\
            \vdots & \ddots & \vdots \\
            P_{1}(\bar{x}_s) & \cdots & P_s(\bar{x}_s)
        \end{bmatrix}
    \end{equation}
    has full rank.
    Suppose that the matrix
    \begin{equation}
        \mat{\tilde{T}}_{s+1}(x) = 
        \begin{bmatrix}
            P_1(\bar{x}_1) & \cdots & P_{s}(\bar{x}_1) & P_{s+1}(\bar{x}_1) \\
            \vdots & \ddots & \vdots & \vdots \\
            P_{1}(\bar{x}_s) & \cdots & P_s(\bar{x}_s) & P_{s+1}(\bar{x}_s) \\
            P_{1}(x) & \cdots & P_s(x) & P_{s+1}(x)
        \end{bmatrix}
    \end{equation}
    has rank $< s+1$ for every $x \in \R^m$.
    Since the upper left $s\times s$ block of $\mat{\tilde{T}}_{s+1}(x)$ is $\mat{T}_s$, we must always have $\rank(\mat{\tilde{T}}_{s+1}(x)) = s$.
    The nullspace of $\mat{\tilde{T}}_{s+1}(x)$ is contained in the nullspace of the upper $s\times(s+1)$ block of $\mat{\tilde{T}}_{s+1}(x)$.
    Since both nullspaces are one-dimensional, they are equal.
    The upper $s\times(s+1)$ block of $\mat{\tilde{T}}_{s+1}(x)$ does not depend on $x$, so there is a fixed nonzero vector $\vect{v}\in\R^{s+1}$ so that $\mat{\tilde{T}}_{s+1}(x) \vect{v} = \vect{0}$ for every $x\in\R^m$.
    The last row of this expression reads
    \begin{equation}
        v_1 P_1(x) + \cdots + v_{s+1} P_{s+1}(x) = 0 \qquad \forall x\in\R^m,
    \end{equation}
    contradicting the linear independence of $P_1, \ldots, P_{s+1}$.
    Therefore there exists $\bar{x}_{s+1}$ so that $\mat{T}_{s+1} = \mat{\tilde{T}}_{s+1}(\bar{x}_{s+1})$ has full rank.
    It follows by induction on $s$ that there exists $\bar{x}_1, \ldots \bar{x}_{M_0}\in \R^m$ so that $\mat{T}_{(\bar{x}_1, \ldots \bar{x}_{M_0})} = \mat{T}_{M_0}$ has full rank.
    Choosing any $M - M_0$ additional vectors yields an injective $\mat{T}_{(\bar{x}_1, \ldots \bar{x}_{M})}$, which completes the proof.
\end{proof}

\begin{proof}[Proposition~\ref{prop:generic_polynomial_map_sampling}]
    The sum in \eqref{eqn:empirical_L2_mu_inner_prod} clearly defines a symmetric, positive-semidefinite bilinear form on $\mcal{F}'$.
    It remains to show that this bilinear form is positive-definite.
    Suppose that there is a function $f \in \mcal{F}'$ such that $\langle f, f\rangle_{L^2(\mu_M)} = 0$.
    Thanks to Lemma~\ref{lem:generic_polynomial_sampling}, our assumption that $M \geq \dim(\pi_i(\mcal{F}'))$ means that the evaluation operator $T_{(x_1, \ldots, x_M)}$ is injective on $\pi_i(\mcal{F}')$ for almost every $(x_1, \ldots, x_M) \in (\R^m)^M$ with respect to Lebesgue measure.
    Since a countable (in this case finite) intersection of sets of measure zero has measure zero, it follows that for almost every $(x_1, \ldots, x_M) \in (\R^m)^M$ with respect to Lebesgue measure, $T_{(x_1, \ldots, x_M)}$ is injective on every $\pi_i(\mcal{F}')$, $i=1, \ldots, n$.
    Defining the positive diagonal matrix
    \begin{equation}
        \mat{D} = \frac{1}{\sqrt{N}} \begin{bmatrix}
            \sqrt{w_1} & & \\
            & \ddots & \\
            & & \sqrt{w_M}
        \end{bmatrix},
    \end{equation}
    and using \eqref{eqn:empirical_L2_mu_inner_prod} yields
    \begin{equation}
        0 = \langle f, f\rangle_{L^2(\mu_M)} = \sum_{j=1}^{n} \big(\mat{D} T_{(x_1, \ldots, x_M)} \pi_j f\big)^T \mat{D} T_{(x_1, \ldots, x_M)} \pi_j f .
    \end{equation}
    This implies that $T_{(x_1, \ldots, x_M)} \pi_j f = \vect{0}$ for $j=1, \ldots, n$.
    Since $T_{(x_1, \ldots, x_M)}$ is injective on each $\pi_j(\mcal{F}')$ it follows that each $\pi_j f = 0$, meaning that $f = 0$.
    This completes the proof.
\end{proof}

\section{Proof of Proposition~\ref{prop:properties_of_Lie_derivative}}
\label{app:properties_of_Lie_derivative}

    We begin by proving
    \begin{equation}
        \ddt \mcal{K}_{\exp(t\xi)}F
        = \mcal{K}_{\exp(t\xi)} \mcal{L}_{\xi} F
        = \mcal{L}_{\xi} \mcal{K}_{\exp(t\xi)} F,
        \label{eqn:Lie_derivative_as_generator_proof_version}
    \end{equation}
    for every $F \in D(\mcal{L}_{\xi)}$.
    To prove the first equality, we choose $p\in\mcal{M}$, let $p' = \theta_{\exp(t_0\xi)}(p)$, and compute
    \begin{equation}
        \begin{aligned}
            \frac{1}{t}\left[ \big(\mcal{K}_{\exp(t_0\xi)}\mcal{K}_{\exp(t\xi)}F\big)(p) - \big(\mcal{K}_{\exp(t_0\xi)}F\big)(p) \right]
            &= \frac{1}{t} \Theta_{\exp(-t_0\xi)} \circ \left( \mcal{K}_{\exp(t\xi)} F - F \right) \circ \theta_{\exp(t_0\xi)}(p) \\
            &= \Theta_{\exp(-t_0\xi)}\left( \frac{1}{t}\left[ \big(\mcal{K}_{\exp(t\xi)} F\big)(p') - F(p') \right] \right).
        \end{aligned}
    \end{equation}
    Here, we have used the composition law for the operators $\mcal{K}_{g h} = \mcal{K}_g \mcal{K}_h$ and the fact that $\Theta_{\exp(-t_0\xi)}$ is fiber-linear.
    Taking the limit at $t \to 0$ yields
    \begin{equation}
        \left.\ddt\right\vert_{t=t_0} \big(\mcal{K}_{\exp(t\xi)}F\big)(p)
        = \Theta_{\exp(-t_0\xi)}\left( \mcal{L}_{\xi}F (p') \right)
        = \big(\mcal{K}_{\exp(t_0\xi)} \mcal{L}_{\xi} F\big)(p),
    \end{equation}
    which is the first equality in \eqref{eqn:Lie_derivative_as_generator_proof_version}.
    
    The second equality in \eqref{eqn:Lie_derivative_as_generator_proof_version} follows from
    \begin{equation}
        \lim_{t\to 0} \frac{1}{t}\left[ \mcal{K}_{\exp(t\xi)}\mcal{K}_{\exp(t_0\xi)}F - \mcal{K}_{\exp(t_0\xi)}F \right]
        = \lim_{t\to 0} \frac{1}{t}\left[ \mcal{K}_{\exp(t_0\xi)} \mcal{K}_{\exp(t\xi)} F - \mcal{K}_{\exp(t_0\xi)}F \right]
        = \mcal{K}_{\exp(t_0\xi)} \mcal{L}_{\xi} F.
    \end{equation}
    This shows that $\mcal{K}_{\exp(t_0\xi)}F \in D(\mcal{L}_{\xi})$ and $\mcal{L}_{\xi} \mcal{K}_{\exp(t_0\xi)}F = \mcal{K}_{\exp(t_0\xi)} \mcal{L}_{\xi} F$.

    Next, we prove
    \begin{equation}
        \mcal{L}_{\alpha\xi + \beta \eta} F = \alpha \mcal{L}_{\xi} F + \beta \mcal{L}_{\eta} F,
        \label{eqn:linearity_of_Lie_derivative_wrt_Lie_algebra_proof_version}
    \end{equation}
    when $F \in C^1(\mcal{M}, E)$.
    To do this, we choose $p\in\mcal{M}$, and define the map $h: G \to E_p$ by
    \begin{equation}
        h: g \mapsto \mcal{K}_g F(p) = \Theta\big( F(\theta(p,g)), g^{-1} \big).
        \label{eqn:pointwise_orbit_transformation}
    \end{equation}
    As a composition of $C^1$ maps, $h$ is $C^1$, and its derivative at the identity is
    \begin{equation}
        \D h(e) \xi_e = \left.\ddt\right\vert_{t=0} h(\exp(t\xi)) = \mcal{L}_{\xi} F(p)
        \label{eqn:pointwise_orbit_transformation_tangent_map}
    \end{equation}
    for every $\xi_e \in T_e G  \cong \Lie(G)$.
    Since the derivative is linear, it follows that $\xi \mapsto \mcal{L}_{\xi} F(p)$ is linear.

    Finally, we prove that
    \begin{equation}
        \mcal{L}_{[\xi, \eta]} F 
        = \frac{1}{2}\left.\frac{\td^2}{\td t^2}\right\vert_{t=0} \mcal{K}_{\exp(t\xi)}\mcal{K}_{\exp(t\eta)}\mcal{K}_{\exp(-t\xi)}\mcal{K}_{\exp(-t\eta)} F
        = \mcal{L}_{\xi}\mcal{L}_{\eta} F - \mcal{L}_{\eta}\mcal{L}_{\xi} F,
        \label{eqn:Lie_derivative_commutator_and_second_derivative}
    \end{equation}
    when $F \in C^2(\mcal{M}, E)$.
    Recall that $\Flow_{\xi}^t: g \mapsto g \cdot \exp(t\xi)$ gives the flow of the left-invariant vector field $\xi\in\Lie(G)$ (see Theorem~4.18(3) in \cite{Kolar1993natural}).
    By Theorem~3.16 in \cite{Kolar1993natural} the curve $\gamma: \R \to G$ given by
    \begin{equation}
        \gamma(t) 
        = \Flow_{-\eta}^t \circ \Flow_{-\xi}^t \circ \Flow_{\eta}^t \circ \Flow_{\xi}^t(e)
        = \exp(t\xi)\exp(t\eta)\exp(-t\xi)\exp(-t\eta).
    \end{equation}
    satisfies $\gamma(0) = e$, $\gamma'(0) = 0$, and
    \begin{equation}
        \frac{1}{2}\gamma''(0) = [\xi, \eta]_e \in T_e G
        \label{eqn:first_vanishing_derivative_recovers_Lie_bracket}
    \end{equation}
    in the sense that $\gamma''(0): f \mapsto (f \circ \gamma)''(0)$ is a derivation on $C^{\infty}(G)$, hence an element of $T_e G$.
    Composing with the map in \eqref{eqn:pointwise_orbit_transformation} yields
    \begin{equation}
        0 = \D h(e) \gamma'(0) 
        = (h\circ \gamma)'(0) 
        = \left.\ddt \right\vert_{t=0} \mcal{K}_{\gamma(t)} F(p).
        \label{eqn:vanishing_first_derivative_of_K_gamma}
    \end{equation}
    Combining \eqref{eqn:first_vanishing_derivative_recovers_Lie_bracket} and \eqref{eqn:pointwise_orbit_transformation_tangent_map} (noting the definition of the tangent map $\D h(e)$ acting on derivations, as in \cite{Kolar1993natural}, \cite{Lee2013introduction}) gives
    \begin{equation}
        \mcal{L}_{[\xi, \eta]} F(p)
        = \frac{1}{2} \D h(e) \gamma''(0)
        = \frac{1}{2} (h\circ\gamma)''(0)
        = \frac{1}{2} \left.\frac{\td^2}{\td t^2}\right\vert_{t=0} \mcal{K}_{\gamma(t)} F(p).
    \end{equation}
    This proves the first equality in \eqref{eqn:Lie_derivative_commutator_and_second_derivative} thanks to the composition law
    \begin{equation}
        \mcal{K}_{\gamma(t)} = \mcal{K}_{\exp(t\xi)}\mcal{K}_{\exp(t\eta)}\mcal{K}_{\exp(-t\xi)}\mcal{K}_{\exp(-t\eta)}.
    \end{equation}

    To differentiate the above expression, we use the following observations.
    If $F_t \in \sect(E)$ is such that $(t,p) \mapsto F_t(p)$ is $C^2(\R\times\mcal{M}, E)$, then obviously $\ddt F_t\in C^1(\mcal{M}, E)$ with the usual identification $T E_p \cong E_p$.
    Moreover, we have
    \begin{equation}
    \begin{aligned}
        \ddt \mcal{K}_g F_t(p) 
        &= \ddt \Theta_{g^{-1}}\big( F_t(\theta_g(p)) \big) \\
        &= \Theta_{g^{-1}}\Big( \ddt F_t(\theta_g(p)) \Big)
        = \mcal{K}_g \Big( \ddt F_t \Big)(p)
        \label{eqn:passing_derivative_under_Kg}
    \end{aligned}
    \end{equation}
    because $F_t(\theta_g(p)) \in E_{\theta_g(p)}$ for all $t\in\R$ and $\Theta_{g^{-1}}$ is linear on $E_{\theta_g(p)}$.
    Using this, we obtain
    \begin{equation}
    \begin{aligned}
        \ddt \mcal{L}_{\xi} F_t(p) 
        &= \ddt \left.\frac{\td}{\td \tau}\right\vert_{\tau=0} \mcal{K}_{\exp(\tau\xi)} F_t(p) \\
        &= \left.\frac{\td}{\td \tau}\right\vert_{\tau=0} \ddt \mcal{K}_{\exp(\tau\xi)} F_t(p) \\
        &= \left.\frac{\td}{\td \tau}\right\vert_{\tau=0} \mcal{K}_{\exp(\tau\xi)} \Big( \ddt F_t \Big)(p)
        = \mcal{L}_{\xi} \Big( \ddt F_t \Big)(p)
        \label{eqn:passing_derivative_under_Lxi}
    \end{aligned}
    \end{equation}
    because $(t,\tau) \mapsto \mcal{K}_{\exp(\tau\xi)} F_t(p)$ lies in the vector space $E_{p}$, allowing us to exchanged the order of differentiation.
    Since $(t_1, t_2, t_3, t_4) \mapsto \mcal{K}_{\exp(t\xi)}\mcal{K}_{\exp(t\eta)}\mcal{K}_{\exp(-t\xi)}\mcal{K}_{\exp(-t\eta)} F(p)$ lies in the vector space $E_p$ for all $(t_1, t_2, t_3, t_4)\in\R^4$, we can apply the chain rule and \eqref{eqn:passing_derivative_under_Kg} to obtain
    \begin{multline}
        \ddt \mcal{K}_{\gamma(t)} F
        = \left.\frac{\partial}{\partial t_1}\right\vert_{t_1=t} \mcal{K}_{\exp(t_1\xi)}\mcal{K}_{\exp(t\eta)}\mcal{K}_{\exp(-t\xi)}\mcal{K}_{\exp(-t\eta)} F \\
        + \mcal{K}_{\exp(t\xi)} \left.\frac{\partial}{\partial t_2}\right\vert_{t_2=t} \mcal{K}_{\exp(t_2\eta)}\mcal{K}_{\exp(-t\xi)}\mcal{K}_{\exp(-t\eta)} F \\
        + \mcal{K}_{\exp(t\xi)}\mcal{K}_{\exp(t\eta)}\left.\frac{\partial}{\partial t_3}\right\vert_{t_3=t}\mcal{K}_{\exp(-t_3\xi)}\mcal{K}_{\exp(-t\eta)} F \\
        + \mcal{K}_{\exp(t\xi)}\mcal{K}_{\exp(t\eta)}\mcal{K}_{\exp(-t\xi)}\left.\frac{\partial}{\partial t_4}\right\vert_{t_4=t}\mcal{K}_{\exp(-t_4\eta)} F.
    \end{multline}
    Using \eqref{eqn:Lie_derivative_as_generator_proof_version} gives
    \begin{multline}
        \ddt \mcal{K}_{\gamma(t)} F
        = \mcal{L}_{\xi} \overbrace{\mcal{K}_{\exp(t\xi)}\mcal{K}_{\exp(t\eta)}\mcal{K}_{\exp(-t\xi)}\mcal{K}_{\exp(-t\eta)}}^{\mcal{K}_{\gamma(t)}} F \\
        + \mcal{K}_{\exp(t\xi)} \mcal{L}_{\eta} \mcal{K}_{\exp(t\eta)}\mcal{K}_{\exp(-t\xi)}\mcal{K}_{\exp(-t\eta)} F \\
        + \mcal{K}_{\exp(t\xi)}\mcal{K}_{\exp(t\eta)}\mcal{L}_{-\xi}\mcal{K}_{\exp(-t\xi)}\mcal{K}_{\exp(-t\eta)} F \\
        + \underbrace{\mcal{K}_{\exp(t\xi)}\mcal{K}_{\exp(t\eta)}\mcal{K}_{\exp(-t\xi)}\mcal{K}_{\exp(-t\eta)}}_{\mcal{K}_{\gamma(t)}}\mcal{L}_{-\eta} F.
    \end{multline}
    Applying the same technique to differentiate a second time and using the linearity in \eqref{eqn:linearity_of_Lie_derivative_wrt_Lie_algebra_proof_version} to cancel terms yields
    \begin{multline}
        \left.\frac{\td^2}{\td t^2}\right\vert_{t=0} \mcal{K}_{\gamma(t)} F
        = \mcal{L}_{\xi} \underbrace{\left.\ddt\right\vert_{t=0} \mcal{K}_{\gamma(t)} F}_{0}
        + \underbrace{
        \mcal{L}_{\xi}\mcal{L}_{\eta} F
        + \mcal{L}_{\eta}\mcal{L}_{-\xi} F
        + \mcal{L}_{\eta}\mcal{L}_{-\xi} F
        + \mcal{L}_{-\xi}\mcal{L}_{-\eta} F
        }_{2\big( \mcal{L}_{\xi}\mcal{L}_{\eta} F - \mcal{L}_{\eta} \mcal{L}_{\xi} F \big)} \\
        + \underbrace{\left.\ddt\right\vert_{t=0} \mcal{K}_{\gamma(t)} \mcal{L}_{-\eta} F}_{0},
    \end{multline}
    which completes the proof.
\hfill\qedsymbol

\section{Proof of Theorem~\ref{thm:symmetries_of_sections}}
\label{app:symmetries_of_sections}

    We begin by showing that $\Sym_G(F)$ is a closed subgroup of $G$.
    It is obviously a subgroup, for if $g_1, g_2 \in \Sym_G(F)$ then
    \begin{equation}
        \mcal{K}_{g_1 g_2} F = \mcal{K}_{g_1} \mcal{K}_{g_2} F = \mcal{K}_{g_1} F = F,
    \end{equation}
    meaning that $g_1 g_2 \in \Sym_G(F)$.
    To show that $\Sym_G(F)$ is closed, we observe that for each $p\in\mcal{M}$, the map $h_p : G \to E$ defined by
    \begin{equation}
        h_p: g \mapsto \mcal{K}_g F(p) = \Theta\big( F(\theta(p,g)) , g^{-1} \big)
    \end{equation}
    is continuous, as it is a composition of continuous maps.
    As $F(p)$ is a single point in $E$, the preimage set $h_p^{-1}\big( \{ F(p) \} \big)$ is closed in $G$.
    Since $\Sym_G(F)$ is an intersection,
    \begin{equation}
        \Sym_G(F) = \bigcap_{p\in\mcal{M}} h_p^{-1}\big( \{ F(p) \} \big),
    \end{equation}
    of closed sets, it follows that $\Sym_G(F)$ is closed in $G$.
    By the closed subgroup theorem (Theorem~20.12 in \citet{Lee2013introduction}) it follows that $\Sym_G(F)$ is an embedded Lie subgroup of $G$.

    Let $\mfrak{h} = \Lie(\Sym_G(F))$ be the Lie algebra of $\Sym_G(F)$.
    Choosing any $\xi \in \mfrak{h}$ we have $\exp(t\xi) \in \Sym_G(F)$ for every $t\in\R$, yielding
    \begin{equation}
        \lim_{t\to 0} \frac{1}{t} \left[ \mcal{K}_{\exp(t\xi)} F - F \right] = 0.
    \end{equation}
    Hence, $F \in D(\mcal{L}_{\xi})$ and $\mcal{L}_{\xi} F = 0$, 
    meaning that $\mfrak{h} \subset \sym_G(F)$, as defined by \eqref{eqn:sym_G_for_vb_section}.
    
    To show the reverse containment, choose $\xi \in \sym_G(F)$, meaning that $F \in D(\mcal{L}_{\xi})$ and $\mcal{L}_{\xi}F = 0$.
    We observe that \eqref{eqn:Lie_derivative_as_generator} in Proposition~\ref{prop:properties_of_Lie_derivative} yields
    \begin{equation}
        \ddt \mcal{K}_{\exp(t\xi)} F = \mcal{K}_{\exp(t\xi)} \mcal{L}_{\xi} F = 0 \qquad \forall t\in\R.
    \end{equation}
    It follows that $\mcal{K}_{\exp(t\xi)} F = F$, that is, $\exp(t\xi)\in\Sym_G(F)$ for all $t\in\R$.
    Differentiating at $t=0$ proves that $\xi\in\mfrak{h}$.
    Therefore, $\mfrak{h} = \sym_G(F)$, which completes the proof.
\hfill\qedsymbol

\section{Proof of Theorem~\ref{thm:equivariance_conditions_for_vb_section}}
\label{app:equivariance_conditions_for_vb_section}

    If $F\in C(\mcal{M}, E)$ is $G_0$-equivariant, then $\mcal{K}_{\exp(t\xi)} F = F$ for all $\xi\in\Lie(G)$ and $t\in\R$.
    Differentiating with respect to $t$ at $t=0$ gives $\mcal{L}_{\xi} F = 0$.
    
    Conversely, suppose that $\mcal{L}_{\xi_i} F = 0$ for a collection of generators $\xi_1, \ldots, \xi_q$ of $\Lie(G)$.
    By Theorem~\ref{thm:symmetries_of_sections}, $\Sym_G(F)$ is a closed Lie subgroup of $G$ whose Lie subalgebra $\sym_{G}(F)$ contains $\xi_1, \ldots, \xi_q$.
    Since $\xi_1, \ldots, \xi_q$ generate $\Lie(G)$, it follows that $\sym_{G}(F) = \Lie(G)$.
    This means that $G_0 \subset \Sym_G(F)$ due to the correspondence between connected Lie subgroups and their Lie subalgebras established by Theorem~19.26 in \cite{Lee2013introduction}.
    Specifically, the identity component of $\Sym_G(F)$ must correspond to $G_0$ since both are connected Lie subgroups of $G$ with identical Lie subalgebras.

    Now, let us suppose in addition that $\mcal{K}_{g_j} F = F$ for an element $g_j$ from each non-identity component $G_j$, $j=1, \ldots, n_G - 1$ of $G$.
    By Proposition~7.15 in \cite{Lee2013introduction}, $G_0$ is a normal subgroup of $G$ and every connected component $G_j$ of $G$ is diffeomorphic to $G_0$.
    In fact in the proof of this result it is shown that every connected component of $G_j$ is a coset of $G_0$, meaning that $G_j = G_0 \cdot g_j$.
    Choosing any $g \in G_j$ there is an element $g_0 \in G_0$ such that $g = g_0 \cdot g_j$ and we obtain
    \begin{equation}
        \mcal{K}_g F 
        = \mcal{K}_{g_0} \mcal{K}_{g_j} F
        = \mcal{K}_{g_0} F
        = F.
    \end{equation}
    This completes the proof because $G = \bigcup_{j=0}^{n_G - 1} G_j$.
    \hfill\qedsymbol

\section{Proof of Theorem~\ref{thm:symmetries_of_a_submanifold}}
\label{app:symmetries_of_a_submanifold}

Our proof of the theorem relies on the following technical lemma concerning the integral curves of vector fields tangent to weakly embedded, arcwise-closed submanifolds.

\begin{lemma}
    \label{lem:integral_curve_tangency}
    Let $\mcal{M}$ be an arcwise-closed weakly embedded submanifold of a manifold $\mcal{N}$. 
    Let $V\in\vf(\mcal{N})$ be a vector field tangent to $\mcal{M}$, that is
    \begin{equation}
        V_p \in T_p\mcal{M} \qquad \forall p\in\mcal{M}.
    \end{equation}
    If $\gamma:I \to \mcal{N}$ is a maximal integral curve of $V$ that intersects $\mcal{M}$, then $\gamma$ lies in $\mcal{M}$.
\end{lemma}
\begin{proof}
    By the translation lemma (Lemma~9.4 in \cite{Lee2013introduction}), we can assume without loss of generality that $0\in I$ and $p_0 = \gamma(0)\in\mcal{M}$.
    Let $\imath_{\mcal{M}}:\mcal{M} \hookrightarrow \mcal{N}$ denote the inclusion map.
    Since $\mcal{M}$ is an immersed submanifold of $\mcal{N}$ and $V$ is tangent to $\mcal{M}$, there is a unique smooth vector field $V\vert_{\mcal{M}}\in\vf(\mcal{M})$ that is $\imath_{\mcal{M}}$-related to $V$ thanks to Proposition~8.23 in \cite{Lee2013introduction}.
    Let $\tilde{\gamma}:\tilde{I} \to \mcal{M}$ be the maximal integral curve of $V\vert_{\mcal{M}}$ with $\tilde{\gamma}(0) = p_0$.
    By the naturality of integral curves (Proposition~9.6 in \cite{Lee2013introduction}) $\imath_{\mcal{M}} \circ \tilde{\gamma}$ is an integral curve of $V$ with $\imath_{\mcal{M}} \circ \tilde{\gamma}(0) = p_0$.
    Since integral curves of smooth vector fields starting at the same point are unique (Theorem~9.12, part~(a) in \cite{Lee2013introduction}) we have $\tilde{I} \subset I$ and 
    \begin{equation}
        \imath_{\mcal{M}} \circ \tilde{\gamma}(t) = \gamma(t) \qquad \forall t \in \tilde{I}.
        \label{eqn:equality_of_integral_curves}
    \end{equation}
    Therefore, it remains to show that $\tilde{I} = I$.
    
    By the local existence of integral curves (Proposition~9.2 in \cite{Lee2013introduction}), the domains $I$ and $\tilde{I}$ of the maximal integral curves $\gamma$ and $\tilde{\gamma}$ are open intervals in $\R$.
    Suppose, for the sake of producing a contradiction, that there exists $t\in I$ with $t > \tilde{I}$.
    Then it follows that the least upper bound $b = \sup \tilde{I}$ is an element of $I$. 
    By \eqref{eqn:equality_of_integral_curves} and continuity of $\gamma$ we have
    \begin{equation}
        q_0 = \gamma(b) = \lim_{t\to b} \imath_{\mcal{M}} \circ \tilde{\gamma}(t).
        \label{eqn:upper_limit_of_integral_curve}
    \end{equation}
    Since $\mcal{M}$ is arcwise-closed, it follows that $q_0 \in \mcal{M}$.
    
    To complete the proof, we use the local existence of an integral curve for $V\vert_{\mcal{M}}$ starting at $q_0$ to contradict the maximality of $\tilde{\gamma}$.
    By the local existence of integral curves (Proposition~9.2 in \cite{Lee2013introduction}) and the translation lemma (Lemma~9.4 in \cite{Lee2013introduction}), there is an $\varepsilon > 0$ and an integral curve $\hat{\gamma}: (b-\varepsilon, b+\varepsilon) \to \mcal{M}$ of $V\vert_{\mcal{M}}$ such that $\hat{\gamma}(b) = q_0 = \gamma(b)$.
    Shrinking the interval, we take $0<\varepsilon < b-a$.
    Again, by nauturality and uniqueness of integral curves we must have $\imath_{\mcal{M}} \circ \hat{\gamma}(t) = \gamma(t)$ for all $t\in (b-\varepsilon, b+\varepsilon)$.
    Hence, by \eqref{eqn:equality_of_integral_curves} and injectivity of $\imath_{\mcal{M}}$ it follows that $\hat{\gamma}(t) = \tilde\gamma(t)$ for all $t\in (b-\varepsilon, b)$.
    Applying the gluing lemma (Corollary~2.8 in \cite{Lee2013introduction}) to $\tilde{\gamma}$ and $\hat{\gamma}$ yields an extension of $\tilde{\gamma}$ to the larger open interval $\tilde{I} \cup (b-\varepsilon, b+\varepsilon)$.
    Since this contradicts the maximality of $\tilde{\gamma}$, there is no $t \in I$ with $t > \tilde{I}$.
    The same argument shows that there is no $t \in I$ with $t < \tilde{I}$, and so we must have $\tilde{I} = I$.
\end{proof}

We also use the following lemma describing the elements of a Lie group that can be constructed from products of exponentials.
\begin{lemma}
    \label{lem:generation_by_finite_products}
    Let $G_0$ be the identity component of a Lie group $G$.
    Then every element $g\in G_0$ can be expressed as a finite product $g = h_m \cdots h_1$ of elements $h_i=\exp(\xi_i)$ for $\xi_i \in\Lie(G)$.
    Let $G_i$ be a connected component of $G$ and let $g_i \in G_i$.
    Then every element $g\in G_i$ can be expressed as $g = g_0 g_i$ for some $g_0 \in G_0$.
\end{lemma}
\begin{proof}
    By the inverse function theorem (more specifically by Proposition~20.8(f) in \cite{Lee2013introduction}), the range of the exponential map contains an open, connected neighborhood $\mcal{U}$ of the identity element $e \in G$.
    The inverses of the elements in $\mcal{U}$ also belong to the range of the exponential map thanks to Proposition~20.8(c) in \cite{Lee2013introduction}.
    By Proposition~7.14(b) and Proposition~7.15 in \cite{Lee2013introduction}, it follows that $\mcal{U}$ generates the identity component $G_0$ of $G$.
    That is, any element $g\in G_0$ can be written as a finite product of elements in $\mcal{U}$ and their inverses, which proves the first claim.
    
    By Proposition~7.15 in \cite{Lee2013introduction}, $G_0$ is a normal subgroup of $G$ and every connected component of $G$ is diffeomorphic to $G_0$.
    In fact in the proof of this result it is shown that every connected component of $G$ is a coset of $G_0$.
    Therefore, if $G_i$ is a non-identity connected component of $G$ and $g_i \in G_i$ then $G_i = G_0 \cdot g_i$, which proves the second claim.
\end{proof}

\begin{proof}[Theorem~\ref{thm:symmetries_of_a_submanifold}]
    The set $\sym_G(\mcal{M})$ is a subspace of $\Lie(G)$, for if $\xi_1, \xi_2 \in \sym_G(\mcal{M})$ and $a_1, a_2 \in \R$ then
    \begin{equation}
        \hat{\theta}(a_1 \xi_1 + a_2 \xi_2)_p = a_1 \underbrace{\hat{\theta}(\xi_1)}_{\in T_p\mcal{M}} + a_2 \underbrace{\hat{\theta}(\xi_2)}_{\in T_p\mcal{M}} \in T_p\mcal{M}
    \end{equation}
    thanks to linearity of the infinitesimal generator $\hat{\theta}$.
    To show that $\sym_G(\mcal{M})$ is a Lie subalgrebra, we must show that it is also closed under the Lie bracket.
    Recall that $\hat{\theta}$ is a Lie algebra homomorphism (see Theorem~20.15 in \cite{Lee2013introduction}), and so $\hat{\theta}([\xi_1, \xi_2]) = [\hat{\theta}(\xi_1), \hat{\theta}(\xi_1)]$.
    Since the Lie bracket of two vector fields tangent to an immersed submanifold is also tangent to the submanifold (see Corollary~8.32 in \cite{Lee2013introduction}), it follows that $[\hat{\theta}(\xi_1), \hat{\theta}(\xi_1)]$ is tangent to $\mcal{M}$.
    Hence, $\sym_G(\mcal{M})$ is closed under the Lie bracket and is therefore a Lie subalgebra of $\Lie(G)$.
    By Theorem~19.26 in \cite{Lee2013introduction}, there is a unique connected Lie subgroup of $H \subset G$ whose Lie subalgebra is $\sym_G(\mcal{M})$.

    Now suppose that $\mcal{M}$ is weakly embedded and arcwise-closed in $\mcal{N}$.
    We aim to show that $\mcal{M} \cdot H \subset \mcal{M}$.
    Choosing any $\xi \in \sym_G(\mcal{M})$, Lemma~20.14 in \cite{Lee2013introduction} shows that $\xi$, regarded as a left-invariant vector field on $G$, and $\hat{\xi} = \hat{\theta}(\xi)$ are $\theta^{(p)}$-related for every $p\in\mcal{N}$.
    By the naturality of integral curves (Proposition~9.6 in \cite{Lee2013introduction}) it follows that $\gamma_{\xi}^{(p)}:\R \to \mcal{N}$ defined by
    \begin{equation}
        \gamma_{\xi}^{(p)}(t) = p \cdot \exp(t\xi)
    \end{equation}
    is the unique maximal integral curve of $\hat{\xi}$ passing through $p$ at $t=0$.
    When $p\in\mcal{M}$, this integral curve lies in $\mcal{M}$ thanks to Lemma~\ref{lem:integral_curve_tangency}.
    This means that $\mcal{M}$ is invariant under the action of any group element in the range of the exponential map restricted to $\sym_G(\mcal{M})$.
    Proceeding by induction, suppose that $\mcal{M}$ is invariant under the action of any product of $m$ such elements.
    If $g = h_1 \cdots h_m \cdot h_{m+1}$ is a product of $m+1$ elements $h_i \in \exp\big(\sym_G(\mcal{M})\big) \subset H$,
    then it follows from associativity and the induction hypothesis that
    \begin{equation}
        \mcal{M} \cdot (h_1 \cdots h_m \cdot h_{m+1}) = (\mcal{M} \cdot h_1 \cdots h_m) \cdot h_{m+1} \subset \mcal{M} \cdot h_{m+1} \subset \mcal{M}.
    \end{equation}
    Therefore, $\mcal{M}$ is invariant under the action of any finite product of group elements in $\exp\big(\sym_G(\mcal{M})\big)$ by induction on $m$.
    By Lemma~\ref{lem:generation_by_finite_products}, it follows that $\mcal{M}$ is $H$-invariant, proving claim (i).
    
    To prove claim (ii), suppose that $\tilde{H}$ is another connected Lie subgroup of $G$ such that $\mcal{M}\cdot \tilde{H} \subset \mcal{M}$.
    Choosing any $p\in\mcal{M}$ and $\xi \in \Lie(\tilde{H})$, we have
    \begin{equation}
        p \cdot \exp(t\xi) \in \mcal{M} \qquad \forall t\in \R.
    \end{equation}
    Since $\mcal{M}$ is weakly embedded in $\mcal{N}$, this defines a smooth curve $\gamma: \R \to \mcal{M}$ such that $\imath_{\mcal{M}} \circ \gamma(t) = p \cdot \exp(t\xi)$, where $\imath_{\mcal{M}} : \mcal{M} \hookrightarrow \mcal{N}$ is the inclusion map.
    Differentiating and using the definition of the infinitesimal generator gives
    \begin{equation}
        \hat{\theta}(\xi)_p = \left.\frac{\td}{\td t}\right\vert_{t=0} p \cdot \exp(t\xi) = \D \imath_{\mcal{M}}(p) \left.\frac{\td}{\td t}\right\vert_{t=0} \gamma(t) \in T_p \mcal{M}.
    \end{equation}
    Therefore, $\Lie(\tilde{H}) \subset \sym_G(\mcal{M})$ which implies that $\tilde{H} \subset H$ by Theorem~19.26 in \cite{Lee2013introduction}, establishing claim (ii).

    Now suppose that $\mcal{M}$ is properly embedded in $\mcal{N}$ and denote
    \begin{equation}
        \Sym_G(\mcal{M}) 
        = \{ g\in G \ : \mcal{M} \cdot g \subset \mcal{M} \}
        = \bigcap_{p\in\mcal{M}} \big(\theta^{(p)}\big)^{-1}(\mcal{M}).
    \end{equation}
    The equality of these expressions is a simple matter of unwinding their definitions.
    It is clear that $\Sym_G(\mcal{M})$ is a subgroup of $G$, for if $g_1, g_2 \in \Sym_G(\mcal{M})$ then the composition law for the group action gives $\mcal{M} \cdot (g_1\cdot g_2) = (\mcal{M} \cdot g_1) \cdot g_2 \subset \mcal{M} \cdot g_1 \subset \mcal{M}$.
    Since $\mcal{M}$ is properly embedded, it is closed in $\mcal{M}$ (see \citet[Proposition~5.5]{Lee2013introduction}), meaning that each preimge set $\big(\theta^{(p)}\big)^{-1}(\mcal{M})$ is closed in $G$ by continuity of $\theta^{(p)}$.
    As an intersection of closed subsets, it follows that $\Sym_G(\mcal{M})$ is closed in $G$.
    By the closed subgroup theorem (\citet[Theorem~20.12]{Lee2013introduction}), $\Sym_G(\mcal{M})$ is a properly embedded Lie subgroup of $G$.
    The same holds for the identity component $\Sym_G(\mcal{M})_0$ of $\Sym_G(\mcal{M})$ since $\Sym_G(\mcal{M})_0$ is closed in $\Sym_G(\mcal{M})$, which implies that $\Sym_G(\mcal{M})_0$ is closed in $G$.
    
    Finally, we show that $H = \Sym_G(\mcal{M})_0$ is the identity component of $\Sym_G(\mcal{M})$.
    First, we observe that $H \subset \Sym_G(\mcal{M})_0$ because $H$ is connected and contained in $\Sym_G(\mcal{M})$.
    The reverse containment follows from the fact that $\Sym_G(\mcal{M})_0$ is a connected Lie subgroup satisfying $\mcal{M}\cdot \Sym_G(\mcal{M})_0 \subset \mcal{M}$, which by our earlier result implies that $\Sym_G(\mcal{M})_0 \subset H$.
\end{proof}

\section{Proof of Theorem~\ref{thm:symmetry_conditions_for_submanifold}}
\label{app:symmetry_conditions_for_submanifold}

First, suppose that $\mcal{M}$ is $G_0$-invariant.
In particular, this means that for every $p\in\mcal{M}$ and $\xi\in\Lie(G)$, the smooth curve $\gamma_{\xi}^{(p)}:\R \to \mcal{N}$ defined by
\begin{equation}
    \gamma_{\xi}^{(p)}(t) = p \cdot \exp(t\xi)
\end{equation}
lies in $\mcal{M}$.
Since $\mcal{M}$ is weakly embedded in $\mcal{N}$, $\gamma_{\xi}^{(p)}$ is also smooth as a map into $\mcal{M}$.
Specifically, there is a smooth curve $\tilde{\gamma}_{\xi}^{(p)}:\R \to \mcal{M}$ so that $\gamma_{\xi}^{(p)} = \imath_{\mcal{M}} \circ \tilde{\gamma}_{\xi}^{(p)}$ where $\imath_{\mcal{M}} : \mcal{M} \hookrightarrow \mcal{N}$ is the inclusion map.
Differentiating at $t=0$ yields
\begin{equation}
    \hat{\theta}(\xi)_p
    = \left.\ddt \gamma_{\xi}^{(p)}(t) \right\vert_{t=0}
    = \D \imath_{\mcal{M}}(p) \left.\ddt \tilde{\gamma}_{\xi}^{(p)}(t) \right\vert_{t=0},
\end{equation}
which lies in $T_p \mcal{M} = \Range\left( \D \imath_{\mcal{M}}(p) \right)$.
In particular, $\hat{\theta}(\xi_i)_p \in T_p \mcal{M}$ for every $p \in \mcal{M}$ and $i=1, \ldots, q$.

Conversely, suppose that the tangency condition expressed in \eqref{eqn:generator_manifold_tangency} holds.
By Theorem~\ref{thm:symmetries_of_a_submanifold}, the elements $\xi_1, \ldots, \xi_q$ belong to the Lie subalgebra $\sym_G(\mcal{M})$ of the largest connected Lie subgroup $H \subset G$ of symmetries of $\mcal{M}$. 
Since $\xi_1, \ldots, \xi_q$ generate $\Lie(G)$, it follows that $\sym_G(\mcal{M}) = \Lie(G)$.
Therefore, by Theorem~19.26 in \cite{Lee2013introduction}, we obtain $H = G_0$ because both are connected Lie subgroups of $G$ with identical Lie subalgebras.

Finally, suppose, in addition, that $\mcal{M} \cdot g_j$ for an element $g_j$ from each non-identity component $G_j$ of $G$.
By Lemma~\ref{lem:generation_by_finite_products}, if $g \in G_j$ then there is an element $g_0 \in G_0$ such that $g = g_0 \cdot g_j$.
Therefore, we obtain
\begin{equation}
    \mcal{M} \cdot g 
    = \mcal{M} \cdot g_0 \cdot g_j
    \subset \mcal{M} \cdot g_j
    \subset \mcal{M},
\end{equation}
which completes the proof because $G = \bigcup_{j=0}^{n_G - 1} G_j$.
\hfill\qedsymbol

\section{Proof of Theorem~\ref{thm:characterization_of_Lie_derivative}}
\label{app:characterization_of_Lie_derivative}
\begin{proof}[Lemma~\ref{lem:injection_of_E_into_TE}]
    The map $\imath_F$ defined in a local trivialization $\Phi$ by \eqref{eqn:injection_E_to_TE} is injective.
    It is a vector bundle homomorphism because $\D\Phi \circ \imath_F \circ \Phi^{-1}$, $\Phi$, and $\D \Phi$ are vector bundle homomorphisms and $\Phi$ and $\D \Phi$ are invertible.
    It remains to show that the definition of $\imath_F$ does not depend on the choice of local trivialization.
    Given two local trivializations $\Phi$ and $\tilde{\Phi}$ defined on $\pi^{-1}(\mcal{U}) \subset E$ where $\mcal{U}$ is an open subset of $\mcal{M}$, it suffices to show that the following diagram commutes:
    \begin{equation}
        \begin{tikzcd}
	       {\mcal{U}\times\R^k} & {\pi^{-1}(\mcal{U})\subset E} & {\mcal{U}\times\R^k} \\
	       {T(\mcal{U}\times\R^k)} & {T(\pi^{-1}(\mcal{U}))\subset TE} & {T(\mcal{U}\times\R^k)}
	       \arrow["{\imath_F}", from=1-2, to=2-2]
	       \arrow["{\imath_{\Phi\circ F}}", from=1-1, to=2-1]
	       \arrow["{\imath_{\tilde{\Phi}\circ F}}", from=1-3, to=2-3]
	       \arrow["\Phi"', from=1-2, to=1-1]
	       \arrow["{\D\Phi}", from=2-2, to=2-1]
	       \arrow["{\tilde{\Phi}}", from=1-2, to=1-3]
	       \arrow["{\D\tilde{\Phi}}"', from=2-2, to=2-3]
        \end{tikzcd}
    \end{equation}
    Since $\tilde{\Phi} \circ \Phi^{-1}$ is a bundle homomorphism descending to the identity, it can be written as
    \begin{equation}
        \tilde{\Phi} \circ \Phi^{-1}: (p,\vect{v}) \mapsto (p, \mat{T}(p) \vect{v})
    \end{equation}
    for a matrix-valued function $\mat{T}:\mcal{U} \to \R^{k\times k}$.
    Moreover, the matrices are invertible because the local trivializations are bundle isomorphisms.
    Differentiating, we obtain
    \begin{equation}
        \D \tilde{\Phi} \circ \D \Phi^{-1} 
        : (w_p, \vect{w})_{(p, \vect{v})} \mapsto
        \big(w_p, \D \mat{T}(p) w_p + \mat{T}(p)\vect{w}\big)_{\tilde{\Phi} \circ \Phi^{-1}(p,\vect{v})},
    \end{equation}
    where $w_p \in T_p \mcal{U}$.
    Composing this with $\imath_{\Phi \circ F}:(p, \vect{v}) \mapsto (0, \vect{v})_{\Phi(F(p))}$, we obtain
    \begin{equation}
        \D \tilde{\Phi} \circ \D \Phi^{-1} \circ \imath_{\Phi \circ F}
        (p, \vect{v})
        = (0, \mat{T}(p)\vect{v})_{\tilde{\Phi}(F(p))}
        = \imath_{\tilde{\Phi}\circ F}(p, \mat{T}(p)\vect{v})
        = \imath_{\tilde{\Phi}\circ F} \circ \tilde{\Phi} \circ \Phi^{-1}(p, \vect{v}),
    \end{equation}
    proving that the diagram commutes.
\end{proof}

\begin{proof}[Theorem~\ref{thm:characterization_of_Lie_derivative}]
We observe that $F\circ\pi:E \to E$ is a smooth idempotent map whose image is $\image(F) \subset E$.
By differentiating the expression $(F\circ\pi) \circ (F \circ\pi) = F\circ \pi$ at a point $F(p) \in \image(F)$, we obtain
\begin{equation}
    \D (F\circ\pi)(F(p)) \D (F\circ\pi)(F(p)) = \D (F\circ\pi)(F(p)),
\end{equation}
meaning that $\D (F\circ\pi)(F(p)):T_{F(p)}E \to T_{F(p)}E$ is a linear projection.
Since 
\begin{equation}
    \D (F\circ\pi)(F(p)) = \D F(p) \D \pi(F(p)),
\end{equation}
we have $\Range\big(\D (F\circ\pi)(F(p))\big) \subset \Range(\D F(p)) = T_{F(p)}\image(F)$.
Differentiating $F = (F\circ\pi)\circ F$ yields 
\begin{equation}
    \D F(p) = \D (F\circ \pi)(F(p)) \D F(p),
\end{equation}
meaning that $T_{F(p)}\image(F) \subset \Range\big(\D (F\circ\pi)(F(p))\big)$.
Since $\Range\big(\D (F\circ\pi)(F(p))\big) = T_{F(p)}\image(F)$ it follows that $\D (F\circ\pi)(F(p))$ is a linear projection onto $T_{F(p)}\image(F)$. 

We observe that the generalized Lie derivative in \eqref{eqn:Lie_derivative} can be expressed as
\begin{equation}
    \begin{aligned}
        (\mcal{L}_{\xi} F)(p) &= \lim_{t\to 0} \frac{1}{t}\left[ \Theta_{\exp(-t\xi)}(F( \theta_{\exp(t\xi)}(p))) - F(p) \right] \\
        &= \lim_{t\to 0} \Theta\left(\exp(-t\xi), \ \frac{1}{t}\left[ F(\theta_{\exp(t\xi)}(p)) - \Theta_{\exp(t\xi)}(F(p)) \right] \right) \\
        &= \lim_{t\to 0} \frac{1}{t}\left[ F(\theta_{\exp(t\xi)}(p)) - \Theta_{\exp(t\xi)}(F(p)) \right].
    \end{aligned}
\end{equation}
The first equality follows because $\Theta_{g^{-1}}$ is a vector bundle homomorphism, meaning that the restricted map $\Theta_{g^{-1}}\vert_{E_{p\cdot g}}: E_{p\cdot g} \to E_p$ is linear; here $g = \exp(t\xi)$.
The second equality follows because $\Theta: E \times G \to E$ is continuous.
Note that in the first expression the limit is taken in the vector space $E_p$, whereas in the last expression the limit must be taken in $E$.

We proceed by expressing everything in a local trivialization $\Phi:\pi^{-1}(\mcal{U}) \to \mcal{U}\times \R^k$ of an open neighborhood $\mcal{U}\subset \mcal{M}$ of $p \in \mcal{M}$.
Since the maps $\Theta_g$, $\Phi$, and $\Phi^{-1}$ are vector bundle homomorphisms, there is a matrix-valued function $\mat{T}_g:\mcal{U} \to \R^{k\times k}$ such that
\begin{equation}
    \tilde{\Theta}_g = \Phi \circ \Theta_g \circ \Phi^{-1}: 
    (p,\vect{v}) \mapsto (\theta_g(p), \ \mat{T}_g(p)\vect{v}).
\end{equation}
Differentiating $\tilde{\Theta}_{\exp(t\xi)}(p,\vect{v})$ with respect to $t$ yields the generator
\begin{equation}
    \hat{\tilde{\Theta}}(\xi)_{(p,\vect{v})} 
    = \left.\ddt\right\vert_{t=0} \tilde{\Theta}_{\exp(t\xi)}(p,\vect{v})
    = \left(\hat{\theta}(\xi)_p, \ \mat{\hat{T}}(\xi)_p \vect{v}\right)_{(p,\vect{v})},
\end{equation}
where $\mat{\hat{T}}(\xi)_p = \left.\ddt\right\vert_{t=0} \mat{T}_{\exp(t\xi)}(p)$. 
We define the function $\vect{\tilde{F}}:\mcal{U} \to \R^k$ by
\begin{equation}
    (p, \vect{\tilde{F}}(p)) = \Phi \circ F(p) \qquad \forall p\in\mcal{U}.
\end{equation}

Using the above definitions, we can express the generalized Lie derivative in the local trivialization:
\begin{equation}
    \begin{aligned}
        \Phi \circ (\mcal{L}_{\xi} F) (p) 
        &= 
        \lim_{t\to 0} \left( \theta_{\exp(t\xi)}(p) ,\ 
        \frac{1}{t}\left[ \vect{\tilde{F}}(\theta_{\exp(t\xi)}(p)) - \mat{T}_{\exp(t\xi)} \vect{\tilde{F}}(p) \right] \right) \\
        &= \left( p, \ \D\vect{\tilde{F}}(p)\hat{\theta}(\xi)_p - \mat{\hat{T}}(\xi)_p\vect{\tilde{F}}(p) \right).
    \end{aligned}
    \label{eqn:Lie_derivative_in_local_trivialization}
\end{equation}
Applying Lemma~\ref{lem:injection_of_E_into_TE} allows us to express the left-hand-side of \eqref{eqn:Lie_derivative_as_projection} as
\begin{equation}
    \D\Phi \left[ \imath_F \circ (\mcal{L}_{\xi} F) (p) \right] =
    \left(0, \ \D\vect{\tilde{F}}(p)\hat{\theta}(\xi)_p - \mat{\hat{T}}(\xi)_p\vect{\tilde{F}}(p) \right)_{\Phi(F(p))}.
    \label{eqn:lifted_Lie_derivative_in_local_trivialization}
\end{equation}
We can also express the quantities on the right-hand-side of \eqref{eqn:Lie_derivative_as_projection} in the local trivialization.
To do this, we compute
\begin{equation}
    \begin{aligned}
        \D \Phi(F(p)) \D (F \circ \pi)(F(p)) \hat{{\Theta}}(\xi)_{F(p)} 
        &= \left.\ddt\right\vert_{t=0} \Phi \circ F \circ \pi \circ \Theta_{\exp(t\xi)}(F(p)) \\
        &= \left.\ddt\right\vert_{t=0} \Phi \circ F(\theta_{\exp(t\xi)}(p)) \\
        &= \left( \hat{\theta}(\xi)_p, \ \D\vect{\tilde{F}}(p)\hat{\theta}(\xi)_p \right)_{\Phi(F(p))}
    \end{aligned}
\end{equation}
and
\begin{equation}
    \begin{aligned}
        \D\Phi(F(p)) \hat{\Theta}(\xi)_{F(p)} 
        &= \left. \ddt \right\vert_{t=0} \Phi \circ \Theta_{\exp(t\xi)} \circ \Phi^{-1} \circ \Phi \circ F(p) \\
        &= \hat{\tilde{\Theta}}(\xi)_{\Phi(F(p))}
        = \left( \hat{\theta}(\xi)_p, \ \mat{\hat{T}}(\xi)_p \vect{\tilde{F}}(p) \right)_{\Phi(F(p))}.
    \end{aligned}
\end{equation}
Subtracting these yields
\begin{equation}
    \D\Phi \left[ \left( \D(F\circ\pi)(F(p)) - \Id_{T_{F(p)}E} \right) \hat{\Theta}(\xi)_{F(p)} \right]
    = \left( 0, \ \D\vect{\tilde{F}}(p)\hat{\theta}(\xi)_p - \mat{\hat{T}}(\xi)_p \vect{\tilde{F}}(p) \right)_{\Phi(F(p))},
\end{equation}
which, upon comparison with \eqref{eqn:lifted_Lie_derivative_in_local_trivialization} completes the proof.
\end{proof}
\end{document}